\documentclass[twoside,11pt]{article}

\usepackage{blindtext}

\usepackage{jmlr2e}
\usepackage{xcolor}

\usepackage{microtype}
\usepackage{graphicx}
\usepackage{grffile}
\usepackage{booktabs} %
\usepackage{amssymb,amsfonts,amsmath}

\usepackage{algorithmic}
\usepackage{algorithm}
\usepackage[algo2e]{algorithm2e} 
\usepackage{algorithm,algorithmic}
\usepackage{thmtools,thm-restate}
\usepackage{caption}

\usepackage{hyperref}
\usepackage{cleveref}
\usepackage{wrapfig}
\usepackage{subfigure}

\renewenvironment{proof}{\par\noindent{\bf Proof\ }}{\hfill\BlackBox\\[2mm]}
\newtheorem{assumption}[theorem]{Assumption}

\DeclareMathOperator*{\argmax}{arg\,max}

\renewcommand{\epsilon}{\varepsilon}
\renewcommand{\ln}{\log}

\usepackage{times}

\usepackage{lastpage}
\jmlrheading{26}{2025}{1-\pageref{LastPage}}{2/24; Revised 2/25}{8/25}{24-0267}{Aldo Pacchiano, Mohammad Ghavamzadeh and Peter Bartlett}

\ShortHeadings{Contextual Bandits with Stage-wise Constraints}{Pacchiano, Ghavamzadeh and Bartlett}
\firstpageno{1}

\begin{document}

\title{Contextual Bandits with Stage-wise Constraints}

\author{\name Aldo Pacchiano \email pacchian@bu.edu \\
       \addr Faculty of Computing \& Data Sciences\\
       Boston University\\
       Boston, MA, USA
       \AND
       \name Mohammad Ghavamzadeh \email ghavamza@amazon.com \\
       \addr Amazon AGI \\
       Sunnyvale, CA, USA
       \AND
       \name Peter Bartlett \email peter@berkeley.edu\\
       \addr  Department of Electrical Engineering and Computer Sciences \\
       University of California\\
       Berkeley, CA, USA
       }

\editor{Quanquan Gu}

\maketitle

\begin{abstract}%
We study contextual bandits in the presence of a stage-wise constraint when the constraint must be satisfied both with high probability and in expectation. We start with the linear case where both the reward function and the stage-wise constraint (cost function) are linear. In each of the high probability and in expectation settings, we propose an upper-confidence bound algorithm for the problem and prove a $T$-round regret bound for it. We also prove a lower-bound for this constrained problem, show how our algorithms and analyses can be extended to multiple constraints, and provide simulations to validate our theoretical results. In the high probability setting, we describe the minimum requirements for the action set for our algorithm to be tractable. In the setting that the constraint is in expectation, we specialize our results to multi-armed bandits and propose a computationally efficient algorithm for this setting with regret analysis. Finally, we extend our results to the case where the reward and cost functions are both non-linear. We propose an algorithm for this case and prove a regret bound for it that characterize the function class complexity by the eluder dimension.
\end{abstract}

\begin{keywords}
  multi-armed bandits, contextual bandits, constraints, safety, eluder dimension
\end{keywords}

\section{Introduction}
\label{sec:intro}

A {\em multi-armed bandit} (MAB)~\citep{lai85asymptotically,auer02finitetime,lattimore2018bandit} is an online learning problem in which the agent acts by pulling arms. After an arm is pulled, the agent receives its {\em stochastic reward} sampled from the distribution of the arm. The goal of the agent is to maximize its expected cumulative reward without knowledge of the arms' distributions. To achieve this goal, the agent has to balance its {\em exploration} and {\em exploitation}: to decide when to {\em explore} and learn about the arms, and when to {\em exploit} and pull the arm with the highest estimated reward thus far. A \emph{stochastic linear bandit}~\citep{dani08stochastic,rusmevichientong10linearly,abbasi2011improved} is a generalization of MAB to the setting where each of (possibly) infinitely many arms is associated with a feature vector. The mean reward of an arm is the dot product of its feature vector and an unknown parameter vector, which is shared by all the arms. This formulation contains time-varying action (arm) sets and feature vectors, and thus, includes the {\em linear contextual bandit} setting. These models capture many practical applications spanning clinical trials~\citep{Villar15MA}, recommendation systems~\citep{Li10CB,balakrishnan2018using}, wireless networks~\citep{maghsudi2016multi}, sensors~\citep{washburn2008application}, and strategy games \citep{ontanon2013combinatorial}. 
The most popular exploration strategies in stochastic bandits are {\em optimism in the face of uncertainty} (OFU) or {\em upper confidence bound} (UCB)~\citep{auer02finitetime} and {\em Thompson sampling} (TS)~\citep{thompson33likelihood,agrawal13further,abeille2017linear,russo18tutorial} that are relatively well-understood in both multi-armed and linear bandits~\citep{abbasi2011improved,agrawal13thompson}.

In many practical problems, the agent requires to satisfy certain operational constraints while maximizing its cumulative reward. Depending on the form of the constraints, several {\em constrained stochastic bandit} settings have been formulated and analyzed. One such setting is what is known as {\em knapsack bandits}. In this setting, pulling each arm, in addition to producing a reward signal, results in a random consumption of a global budget, and the goal is to maximize the cumulative reward before the budget is fully consumed (e.g.,~\citealt{badanidiyuru2013bandits,badanidiyuru2014resourceful,agrawal2014bandits,wu2015algorithms,agrawal2016linear}). Another such setting is referred to as {\em conservative bandits}. In this setting, there is a baseline arm or policy, and the agent, in addition to maximizing its cumulative reward, should ensure that at each round, its cumulative reward remains above a predefined fraction of the cumulative reward of the baseline~\citep{wu2016conservative,kazerouni2017conservative,Garcelon20IA}. In these two settings, the constraint is {\em history-dependent}, i.e.,~it applies to a cumulative quantity, such as budget consumption or reward, over the entire run of the algorithm. Thus, the set of feasible actions at each round is a function of the history of the algorithm. 

Another constrained bandit setting is where each arm is associated with two (unknown) distributions, generating reward and cost signals. The goal is to maximize the cumulative reward, while making sure that with {\em high probability}, the expected cost of the arm pulled at each round is below a certain threshold. Here the constraint is {\em stage-wise}, and unlike the last two settings, is independent of the history. This setting has many applications, for example, a recommendation system should not suggest an item to a customer that despite high probability of click (high reward) reduces their watch-time or their chance of coming back to the website (bounded cost), or a drug that may help with a certain symptom (high reward) should not have too many side-effects (bounded cost). Another example is a company whose goal is to optimize its app’s strategy for sending notification to its customers. Here the reward signal is often related to the customer’s engagement with the app, and the cost signal depends on the probability that the customer gets tired of the notifications and opt out. Thus, the goal is to derive a strategy that while maximizes customer’s engagement with the app, keeps the churn below a certain threshold. It is important to note that the reward and cost in this setting can be viewed as different objectives according to which a recommendation or a medical diagnosis system or an app's notification strategy are evaluated. 

\citet{amani2019linear} and~\citet{Moradipari19SL} studied this setting for linear bandits and derived and analyzed {\em explore-exploit}~\citep{amani2019linear} and {\em Thompson sampling}~\citep{Moradipari19SL} algorithms for it. We start the paper by studying the same setting for contextual linear bandits. After defining the setting in Section~\ref{sec:setting}, we propose a {\em UCB-style} algorithm for it, called Linear Constraint Linear UCB (LC-LUCB), in Section~\ref{sec:algo_high_prob}. We prove a $T$-round regret bound for LC-LUCB in Sections~\ref{sec:analysis}, which clearly identifies the main components that control the hardness of this problem. We also prove a lower-bound for this setting in Section~\ref{subsec:lower-bound}, show how this setting can be extended to multiple constraints (multiple cost distributions for each arm) in Section~\ref{sec:multi-constraints}, and report experimental results as a proof of concept for LC-LUCB in Section~\ref{section::experiments}. We provide a detailed comparison between our results and those in~\citet{amani2019linear} and~\citet{Moradipari19SL} in Section~\ref{sec:related-work}. 

We then switch to a slightly different setting in Section~\ref{section::expectation_constraints} in which we relax the high probability constraint and replace it with a constraint in expectation. High probability constraints are often quite restrictive and result in overly conservative strategies. This is why in many applications we may want to relax them to obtain strategies with higher expected cumulative reward. We describe this relaxed setting in Section~\ref{subsec:expectation-setting} and propose an algorithm for it, called Optimistic-Pessimistic Linear Bandit (OPLB) (Section~\ref{sec:algo_expectation}), with regret analysis (Section~\ref{section::lin_opt_pess_analysis}). We then specialize our results to multi-armed bandits (Section~\ref{sec:constrained-MAB}) and report experimental results as a proof of concept for OPLB (Section~\ref{sec:experiments}). Finally in Section~\ref{section::nonlinear_rewards_costs}, we extend our results to the case where the reward and cost functions are non-linear. We propose an algorithm, called Optimistic Pessimistic Nonlinear Bandit (OPNLB), and prove a regret bound for it. We use a characterization of function class complexity based on the eluder dimension~\citep{russo2013eluder} in our regret bound. This part of the paper is an extension of our earlier work~\citep{pacchiano2020stochastic}. Here we improve the regret bounds reported in~\citet{pacchiano2020stochastic} for both contextual linear and multi-armed bandit settings to better show their dependence on the components that contribute to the hardness of the problem. We also show how our results can be extended to non-linear reward and cost(s).

\section{Problem Formulation}
\label{sec:setting}

{\bf Notation.} We adopt the following notation throughout the paper. We denote by $\langle x,y \rangle=x^\top y$ and $\langle x,y \rangle_\mathbf{A}=x^\top \mathbf{A}y$, for a positive definite matrix $\mathbf{A}\in\mathbb R^{d\times d}$, the inner-product and weighted inner-product of vectors $x,y\in\mathbb R^d$. Similarly, we denote by $\|x\|=\sqrt{x^\top x}$ and $\|x\|_\mathbf{A}=\sqrt{x^\top \mathbf{A}x}$, the $\ell_2$ and weighted $\ell_2$ norms of vector $x\in\mathbb R^d$. For any square matrix $\mathbf{A}$, we denote by $\mathbf{A}^\dagger$, its Moore-Penrose pseudo-inverse. We represent the set of distributions with support over a compact set $\mathcal S$ by $\Delta_{\mathcal S}$. We use upper-case letters for random variables (e.g.,~$X$), and their corresponding lower-case letters for a particular instantiation of that random variable (e.g.,~$X=x$). The set $\{1,\ldots,T\}$ is denoted by $[T]$. Finally, we use $\widetilde{\mathcal O}$ for the big-$\mathcal O$ notation up to logarithmic factors. 

We study the following {\em constrained contextual linear bandit} setting in this paper. In each round $t\in[T]$, the agent (also referred to as learner) is given a decision set $\mathcal A_t\subset{\mathbb R}^d$ from which it has to choose an action $x_t$. Upon taking an action $X_t\in\mathcal A_t$, the agent observes a pair $(R_t,C_t)$, where $R_t=\langle X_t,\theta_*\rangle + \xi^r_t$ and $C_t=\langle X_t,\mu_*\rangle + \xi^c_t$ are the reward and cost signals, respectively. In the reward and cost definitions, $\theta_*\in{\mathbb R}^d$ and $\mu_*\in{\mathbb R}^d$ are the unknown {\em reward} and {\em cost parameters}, and $\xi^r_t$ and $\xi^c_t$ are reward and cost noise, satisfying conditions that will be specified in Assumption~\ref{ass:noise-sub-gaussian}. The agent aims to maximize its {\em expected $T$-round reward}, i.e.,~$\sum_{t=1}^T\langle X_t,\theta_*\rangle$, while is required to satisfy a {\bf\em stage-wise linear constraint}, i.e.,~$\langle X_t,\mu_*\rangle \leq \tau,\;\forall t\in[T]$, with high probability. The {\em constraint threshold} $\tau \geq 0$ is a positive constant that is known to the agent.

Because of the constraint, in each round $t$, the agent should pull an arm from the set of {\em feasible actions} in that round, i.e.,~$\mathcal A_t^f=\{x\in\mathcal A_t:\langle x,\mu_*\rangle\leq\tau\}$. Of course this set is unknown, because the agent does not know the cost parameter $\mu_*$. Maximizing the expected $T$-round reward is equivalent to minimizing the {\em expected $T$-round (constrained) (pseudo)-regret}, i.e.,
\begin{equation}
\label{equation::regret_definition}
\mathcal{R}_\mathcal{C}(T) = \sum_{t=1}^T\langle x^*_t,\theta_* \rangle - \langle X_t,\theta_* \rangle = \sum_{t=1}^T\langle x^*_t - X_t,\theta_* \rangle,
\end{equation}
where $x_t^*$ is the {\em optimal feasible action} in round $t$, i.e.,~$x_t^*\in\argmax_{x\in\mathcal A_t^f}\langle x,\theta_*\rangle$, and $X_t$ is the action taken by the agent in round $t$, which belongs to the set of feasible actions in that round, i.e.,~$X_t\in\mathcal A_t^f$, with high probability.

We make the following assumptions for our setting. The first four assumptions are standard in linear bandits and the fifth one is necessary for constraint satisfaction. 

\begin{assumption}[sub-Gaussian noise]
\label{ass:noise-sub-gaussian}
For all $t\in[T]$, the reward and cost noise random variables $\xi_t^r$ and $\xi_t^c$ are conditionally $R$-sub-Gaussian, i.e.,~for all $\alpha\in\mathbb R$,
\begin{align*}
&\mathbb{E}[\xi_t^r \mid \mathcal{H}_{t-1}] = 0, \quad \mathbb{E}[\exp(\alpha \xi_t^r) \mid \mathcal{H}_{t-1}] \leq \exp(\alpha^2 R^2/2), \\
&\mathbb{E}[\xi_t^c \mid \mathcal{H}_{t-1}] = 0, \quad \mathbb{E}[\exp(\alpha \xi_t^c) \mid \mathcal{H}_{t-1}] \leq \exp(\alpha^2 R^2/2),
\end{align*}
where $\mathcal H_t$ is the filtration that includes all the events $(R_{1:t}, C_{1:t},\xi^r_{1:t},\xi^c_{1:t})$ until the end of round $t$. %
\end{assumption}

\begin{assumption}[bounded parameters]
\label{ass:bounded-reward-cost-param}
There is a known constant $S > 0$, such that $\|\theta_*\| \leq S$ and $\|\mu_*\| \leq S$.\footnote{The choice of the same upper-bound $S$ for both $\theta_*$ and $\mu_*$ is just for simplicity and convenience.}
\end{assumption}
\begin{assumption}[bounded actions]
\label{ass:bounded-action}
The $\ell_2$-norm of all actions are bounded by $L > 0$, i.e.,
\begin{equation*}
\max_{t\in[T]}\;\max_{x \in \mathcal{A}_t}\;\|x\| \leq L.   
\end{equation*}
\end{assumption}
\begin{assumption}[bounded rewards and costs]
\label{ass:bounded-mean-reward-cost}
For all $t\in[T]$ and $x \in \mathcal{A}_t$, the mean rewards and costs are bounded, i.e.,~$\langle x, \theta_* \rangle \in [0,1]$ and $\langle x, \mu_* \rangle \in [0,1]$.
\end{assumption}
\begin{assumption}[safe action]
\label{ass:safe-action}
There is a known safe action $x_0\in\mathcal A_t,\;\forall t\in[T]$, with known %
cost $c_0$, i.e.,~$\langle x_0, \mu_* \rangle = c_0 < \tau$, and known reward $r_0$.
\end{assumption}

Knowing a safe action $x_0$ is absolutely necessary for solving the constrained contextual linear bandit problem studied in this paper, because it requires the constraint to be satisfied from the very first round. However, the assumption of knowing its expected reward $r_0$ and cost $c_0$ can be relaxed. We can think of the safe action as a baseline policy, the current strategy (e.g.,~resource allocation) of a company that is safe (i.e.,~its cost $c_0<\tau$) and has a reasonable performance (i.e.,~its reward $r_0$ is not low). In this case, it makes sense to assume that the reward $r_0$ and cost $c_0$ of this action (policy) are both known. We will discuss how our proposed algorithms will change if $r_0$ and $c_0$ are unknown in Sections~\ref{sec:algo_high_prob} and~\ref{sec:algo_expectation}, and Appendix~\ref{section::unknown_c0}. 

We will show later that the difficulty of solving the above constrained bandit problem is directly related to the quality of the safe action $x_0$, more specifically to its {\em safety gap} and {\em sub-optimality}.

\begin{definition}[safety gap \& sub-optimality]
\label{def:safety-suboptimality}
The safety gap and sub-optimality of a safe action $x_0$ quantify how close its cost $c_0$ and reward $r_0$ are to the constraint threshold $\tau$ and the maximum achievable reward $1$, and are defined as $(\tau - c_0)$ and $(1-r_0)$, respectively.
\end{definition}

\noindent
{\bf Notation.} We conclude this section with introducing another set of notations that will be used in describing our algorithms and their analyses. We define the normalized safe action as $e_0:=x_0/\|x_0\|$ and the span of the safe action as $\mathcal{V}_o := \mathrm{span}(x_0) =\{\eta x_0 : \eta\in\mathbb R\}$. We denote by $\mathcal{V}_o^{\perp}$, the orthogonal complement of $\mathcal{V}_o$, i.e.,~$\mathcal{V}_o^{\perp}=\{x\in\mathbb R^d : \langle x,y\rangle=0,\;\forall y\in\mathcal{V}_o\}$.\footnote{In the case of $x_0 = \mathbf{0}\in\mathbb R^d$, we define $\mathcal{V}_o$ as the empty subspace and $\mathcal{V}_o^\perp$ as the entire $\mathbb{R}^d$.} We define the projection of a vector $x\in\mathbb R^d$ into the subspace $\mathcal{V}_o$, as $x^o:=\langle x,e_0\rangle e_0$, and into the subspace $\mathcal{V}_o^{\perp}$, as $x^{o,\perp} := x - x^o$.  %

\section{Related Work}
\label{sec:related-work}

As described in Section~\ref{sec:intro}, the high probability constraint satisfaction setting that we study in Section~\ref{section::high_probability_constraints} is similar to the one in~\cite{Moradipari19SL} and~\cite{amani2019linear}.~\cite{Moradipari19SL} propose a Thompson sampling (TS) algorithm for this setting and prove an $\widetilde{O}(d^{3/2}\sqrt{T}/\tau)$ regret bound for it. Our algorithm is UCB-style and our regret bound is $\widetilde{\mathcal{O}}((1+\frac{1-r_0}{\tau-c_0})d\sqrt{T})$, which not only has a better dependence on $d$, but also clearly identifies the ratio between the sub-optimality $(1-r_0)$ of the safe action $x_0$ and its safety gap $(\tau - c_0)$ as the measure of hardness for the problem. The $\widetilde{\mathcal O}(\sqrt{d})$ advantage to their bound is similar to the best regret results for UCB vs.~TS. Moreover, they restrict themselves to linear bandits, i.e.,~$\mathcal A_t=\mathcal A,\forall t\in[T]$, and define their action set to be any convex compact subset of $\mathbb R^d$ that contains the origin. Therefore, they restrict their {\em ``known"} safe action to be the origin, $x_0=\mathbf{0}$, with the {\em ``known"} cost $c_0=0$. This is why $c_0$ does not appear in their bounds. Although later in their proofs, to guarantee that their algorithm does not violate the constraint in the first round, they require the action set to contain the ball with radius $\tau/S$ around the origin. Hence, our setting and action set are more general than theirs. We also prove a lower-bound for the problem and show how our algorithm and analysis can be extended to multiple constraints and to the case when the reward and cost of the safe action are unknown. Finally, unlike us, their action set does not allow their results to be immediately applicable to MAB. However, their algorithm is TS, and thus, is less complex than ours. Although it can still be intractable, even when the action set $\mathcal A$ is convex, as we can see they require several approximations in their experiments. Unlike them, we describe the minimum requirements on the action set in order for our algorithm to be tractable.

\citet{amani2019linear} propose an explore-exploit algorithm for a slightly different setting than ours, in which reward and cost have the same unknown parameter $\theta_*$, and the constraint is defined as $c_t=x_t^\top B\theta_*\leq\tau$, for a {\em ``known"} matrix $B$. They prove a regret bound of $\widetilde{\mathcal O}(T^{2/3})$ for their algorithm. Although our algorithm has a better regret rate $\widetilde{\mathcal O}(\sqrt{T})$, it cannot immediately give the same rate for the setting studied in~\citet{amani2019linear}, except in special cases, such as when all $\mathcal{A}_t$ are convex and $B = I$.

Several authors have extended the constrained problem studied in this paper to other constrained bandit settings.~\citet{Chen22DO} modified the constraint to cumulative and obtained an $o(T)$ bound for cumulative constraint violation while obtaining an $\mathcal O(\log(T)^2)$ instance-dependent bound for the cumulative regret. Other extensions include to anytime cumulative constraints~\citep{Liu21EP}, kernel setting~\citep{Zhou22KM}, best arm identification~\citep{Wang22BA}, and online convex optimization~\citep{Chaudhary22SO}. 

Our stage-wise constrained bandit problem has also been extended to reinforcement learning (RL) where the goal is to find a policy with maximum expected cumulative reward while the learner is required to keep the expected cumulative cost below a threshold at every single round. Here the learner has access to a safe policy that can be deployed while the learner does not have sufficient knowledge of the safety constraint. This RL setting has been studied in tabular~\citep{efroni2020exploration,liu2021learning,Wei21PE,Bura22DD} and linear~\citep{Ding21PE,Ghosh22PE} MDPs. It is notable that~\citet{liu2021learning} make use of the optimism-pessimism principle that we developed in our earlier work~\citep{pacchiano2020stochastic} and used in the analysis of this paper. Their result is a direct extension of ours to constrained RL. 

\citet{amani2021safe} extended this constrained RL setting to per-step (from per-round) constraints, i.e.,~the expected cost of the action taken at every visited state should be below a threshold. The key idea is that some actions are unsafe and need to be avoided at every step. Here the assumption is the access to a safe action whose expected cost is below the threshold. They use the same geometric conditions that we impose in the analysis of our LC-LUCB algorithm (see Definition~\ref{definition::star_convex}) to ensure that knowledge of a safe action is sufficient for safe exploration. Moreover, their algorithmic techniques rely heavily on our optimism-pessimism principle. \citet{Shi23NO} later extended the per-step constrained RL work of~\citet{amani2021safe} to the case where some state/action combinations are unsafe.

\subsection{A Summary of our Results}

In this paper, we introduce several algorithms for constrained linear bandits in high-probability and in-expectation settings. The learner's objective is to achieve low regret while playing actions that satisfy a cost constraint. An action (or policy) is safe if its expected cost is upper-bounded by a known cost threshold $\tau$. In order to achieve this, the learner has access to a safe arm $x_0$ that belongs to all contexts, and has an expected reward $\langle x_0, \theta_*\rangle  = r_0$ and an expected cost $\langle x_0, \mu_*\rangle = c_0$ satisfying the constraint $c_0 < \tau$ (see Assumption~\ref{ass:safe-action}). All our algorithms satisfy a regret bound of order $\mathcal{O}\left( \frac{1-r_0}{\tau-c_0} d\sqrt{T}\right)$ (ignoring logarithmic factors). In contrast, previous approaches such as Safe-TS satisfy a regret bound of order $\mathcal{O}\left( \frac{1}{\tau} d^{3/2}\sqrt{T}    \right)$ for problems where $c_0 = 0$. In the following table, we compare and contrast our algorithms with Safe-TS. We also highlight the requirements that our approaches require for computational tractability.

\begin{table}[ht]
    \centering
    \begin{tabular}{|c|c|c|c|}
        \hline
        \textbf{Algorithm} & \textbf{Contextual} & \textbf{Action Space} & \textbf{$x_0 = \mathbf{0}$}  \\
        \hline
         Safe-LTS~\citep{Moradipari19SL} & {\color{red} X}  & Convex and Compact & {\color{green} \checkmark} \\             LC-LUCB (Algorithm~\ref{alg::linear_optimism_pessimism})  & {\color{red} \checkmark}  & Star Convex & {\color{red} X}  \\
      OPLB (Algorithm~\ref{alg:optimistic-pessimistic-LB}) &  {\color{green} \checkmark}   & Arbitrary & {\color{red} X}\\
    OPB (Algorithm~\ref{alg::optimism_pessimism}) & {\color{green} \checkmark}   & Multi-Armed Bandits & {\color{red} X} \\
        \hline
    \end{tabular}
    \label{tab:results}
\end{table}

Safe-LTS is not adapted to contextual scenarios and requires the action space to be convex and compact, and to contain the safe action $x_0 = \mathbf{0}$. In contrast, our algorithms LC-LUCB and OPLB are adapted to the contextual scenario. LC-LUCB achieves high probability guarantees and is tractable when the contexts are finite star-convex centered around the safe action $x_0$. In contrast, the OPLB algorithm achieves in-expectation guarantees and is not tractable for general context spaces. Finally, the OPB algorithm attains in-expectation guarantees in multi-armed bandit problems. Note that the OPB policies can be computed by solving a linear program.    

\begin{table}[ht]
    \centering
    \begin{tabular}{|c|c|c|c|}
        \hline
        \textbf{Algorithm} &   \textbf{Regret Bound} & \textbf{Tractability} \\
        \hline
         Safe-LTS~\citep{Moradipari19SL} & $\mathcal{O}\left( \frac{1}{\tau} d^{3/2}\sqrt{T}    \right)$ & {\color{green} \checkmark} \\             LC-LUCB (Algorithm~\ref{alg::linear_optimism_pessimism})  &  $\mathcal{O}\left( \frac{1-r_0}{\tau-c_0} d\sqrt{T}    \right)$& Finite Star-Convex  \\
      OPLB (Algorithm~\ref{alg:optimistic-pessimistic-LB}) &  
$\mathcal{O}\left( \frac{1-r_0}{\tau-c_0} d\sqrt{T}    \right)$ &  {\color{red} X} \\
    OPB (Algorithm~\ref{alg::optimism_pessimism}) & 
$\mathcal{O}\left( \frac{1-r_0}{\tau-c_0} \sqrt{KT}    \right)$ & Linear Program  \\
        \hline
    \end{tabular}
    \label{tab:results_one}
\end{table}

\section{High Probability Constraint Satisfaction} 
\label{section::high_probability_constraints}

As described in Section~\ref{sec:setting}, we study a {\em contextual linear bandit} setting in which each action (arm) is associated with two distributions, generating reward $R_t=\langle X_t,\theta_*\rangle + \xi_t^r$ and cost $C_t=\langle X_t,\mu_*\rangle + \xi_t^c$ signals. The agent aims to maximize its {\em expected cumulative reward} in $T$ rounds, i.e.,~$\sum_{t=1}^T\langle X_t,\theta_*\rangle$, while is required to satisfy the {\bf\em stage-wise linear constraint}
\begin{equation}
\label{eq:high-prob-constraint}
\langle X_t,\mu_* \rangle \leq \tau, \quad \forall t\in[T],
\end{equation}
with probability at least $1-\delta$. The agent knows the constraint threshold $\tau\geq 0$ and has access to a safe action $x_0\in\mathcal A_t$ with known cost $c_0 = \langle x_0, \mu_* \rangle < \tau$ and reward $r_0$ (Assumption~\ref{ass:safe-action}). 

\begin{remark}
\label{remark:extension-MAB}
It is important to note that the high probability constrained setting described above cannot be solved for multi-armed bandits (MABs). This is because there is no generalization among the arms/actions in MABs, and thus, we cannot have an estimate of the cost of an arm without pulling it, which may itself violate the constraint~\eqref{eq:high-prob-constraint}. In other words, pulling the safe action/arm, $x_0$, does not give us any information about the cost of the other arms in MABs. Thus, only interaction with decision sets $\mathcal A_t$ that allow for the safe exploration of progressively better actions may yield provable guarantees. We capture this intuition via a geometric condition on the decision sets $\mathcal A_t$ known as star-convexity. This is in contrast with the in-expectation constrained setting that we study in Section~\ref{section::expectation_constraints}, where it is possible to guarantee safety by playing a distribution over the arms. Extensions to reinforcement learning, such as in~\cite{amani2021safe}, follow the same in-expectation structure that we study in Section~\ref{section::expectation_constraints} and cannot be achieved in the high probability setting studied in this section.  
\end{remark}

\begin{definition}[star-convex set]
\label{definition::star_convex}
We call a set $\mathcal{A}$ star-convex around a point $x \in \mathcal{A}$ if for all other points $a \in \mathcal{A}$, the ray $[x,a ]$ (the line between $x$ and $a$) is in $\mathcal{A}$. When all action sets are star convex centered around $x_0$ the family of star-convex sets is rich enough to contain all convex sets (i.e.,~any convex set is star-convex). %
\end{definition}

Definition~\ref{definition::star_convex} subsumes the case where the action sets $\mathcal A_t$ are convex, and thus, assuming $\mathcal A_t$'s are star-convex is weaker than assuming that they are convex. In this section, we make the following assumption:
\begin{assumption}\label{assumption::star_convex_all_sets}
    All action sets $\mathcal{A}_t$ are star convex centered around the safe action $x_0$.
\end{assumption}

Here we first propose an algorithm for the high probability {\em contextual linear bandit} setting described above. We provide its regret analysis under Assumption~\ref{assumption::star_convex_all_sets}, prove a lower-bound for it, discuss how this setting can be extended to multiple constraints, and finally conclude with a set of experimental results as a proof of concept.

\subsection{Algorithm}
\label{sec:algo_high_prob}

Let $\{X_s\}_{s=1}^t$ be the sequence of actions played by the agent up to time $t$, and $\{ R_s = \langle X_s, \theta_* \rangle + \xi_s^r \}_{s=1}^t$ and $\{ C_s =  \langle X_s, \mu_* \rangle + \xi_s^c\}_{s=1}^t$ be the rewards and costs it observes in the same duration. Since the agent knows the cost of the safe action, i.e.,~$c_0=\langle x_0, \mu_* \rangle$, it can compute the (random) cost incurred by $X_t$ along the subspace $\mathcal{V}_o^\perp$, i.e.,~$C^{\perp}_t = C_t - \frac{\langle X_{t}, e_0\rangle c_0 }{\|x_0\|}$. The knowledge of $c_0$ allows us to build a (regularized) least-squares estimator for $\mu_*$ without estimating it along the $e_0$ direction (recall $x_0 = \| x_0 \|e_0$). For any regularization parameter $\lambda > 0$, we define the regularized covariance matrix in round $t$ as %
\begin{equation}
\label{eq:Sigmas}
\Sigma_t = \lambda \mathrm{I} + \sum_{s=1}^{t-1} X_s X_s^\top, \qquad\qquad \Sigma_t^{o, \perp} = \lambda \mathrm{I}^{o, \perp}+ \sum_{s=1}^{t-1} X_s^{o, \perp}(X_s^{o, \perp})^\top\!,
\end{equation}
where $\mathrm{I}^{o, \perp} = \mathrm{I}-e_0e^\top_0$, and $\Sigma_t$ and $\Sigma_t^{o,\perp}$ are the Gram matrices of the actions and projection of actions into the sub-space $\mathcal V_o^\perp$, respectively. Using~\eqref{eq:Sigmas}, we define the regularized least-squares estimates $\widehat{\theta}_t$ and $\widehat{\mu}^{o, \perp}_t$ of the reward $\theta_*$ and cost $\mu_*^{o,\perp}$ parameters as
\begin{equation}
\label{eq:theta-mu}
\widehat{\theta}_t = \Sigma_t^{-1} \sum_{s=1}^{t-1} R_s X_s, \qquad\qquad \widehat{\mu}^{o, \perp}_t = (\Sigma_t^{o, \perp})^\dagger \sum_{s=1}^{t-1} C^\perp_s X_s^{o,\perp}.  
\end{equation}
To define high probability confidence sets %
around estimators $\widehat{\theta}_t$ and $\widehat{\mu}^{o, \perp}_t$, and to capture how far they are from $\theta_*$ and $\mu_*^{o ,\perp}$, we make use of Theorem~2 in~\citet{abbasi2011improved}. These confidence sets, and in particular their radii, will play an important role in our algorithm. 
\begin{theorem}[Thm.~2 in~\citealp{abbasi2011improved}]
\label{theorem::yasin_theorem}
For a fixed $\delta \in (0,1)$ and 
\begin{equation*}
\beta_t(\delta, d) = R \sqrt{ d \log\left(\frac{ 1 + (t-1)L^2/\lambda}{\delta} \right)} + \sqrt{\lambda}S, \qquad \forall t\in[T],
\end{equation*}
it holds with probability (w.p.) at least $1-\delta$ that %
\begin{equation*}
\| \widehat{\theta}_t - \theta_* \|_{\Sigma_t} \leq \beta_t(\delta,d), \qquad\qquad \| \widehat{\mu}^{o, \perp}_t - \mu_*^{o, \perp} \|_{\Sigma^{o, \perp}_t} \leq \beta_t(\delta, d-1).
\end{equation*} 
\end{theorem}
Using Theorem~\ref{theorem::yasin_theorem}, we now define the following confidence sets (ellipsoids):
\begin{equation}
\label{eq:assymetric-confidence-sets}
\begin{split}
\mathcal{C}_t^r(\alpha_r ) &= \{ \theta \in \mathbb{R}^d: \| \theta -\widehat{\theta}_t \|_{\Sigma_t} \leq \alpha_r \beta_t(\delta,d) \}, \\
\mathcal{C}_t^c(\alpha_c) &= \{ \mu \in \mathcal{V}_0^{ \perp} : \| \mu - \widehat{\mu}^{o, \perp}_t\|_{\Sigma^{o,\perp}_t} \leq \alpha_c \beta_t(\delta, d-1) \},
\end{split}
\end{equation}
around the estimates $\widehat{\theta}_t$ and $\widehat{\mu}^{o, \perp}_t$ with {\em scaling parameters} $\alpha_r,\alpha_c \geq 1$. It is important to note that these confidence sets are {\em asymmetrically scaled}, i.e.,~their radii have been scaled with different scaling parameters. %
Theorem~\ref{theorem::yasin_theorem} suggests that $\theta_*\in \mathcal{C}_t^r(\alpha_r)$ and $\mu_*^{o,\perp}\in\mathcal{C}_t^c(\alpha_c)$, each with probability at least $1-\delta$. 

\begin{algorithm}[t]
\caption{Linear Constraint Linear UCB (LC-LUCB)}
\label{alg::linear_optimism_pessimism}
\begin{algorithmic}[1]
\STATE {\bfseries Input:} Safe action $x_0$ with reward $r_0$ and cost $c_0$; $\;$ Constraint threshold $\tau\geq 0$; $\;$ Scaling parameters $\alpha_r, \alpha_c \geq 1$ 
\FOR{$t=1, \ldots , T$}
\STATE Observe star-convex $\mathcal{A}_t$ and build the estimated feasible action set $\widetilde{\mathcal{A}}^f_t$ using~\eqref{eq:pessimistic-cost} and~\eqref{equation::safe_set}
\STATE Compute action $X_t = \arg\max_{x \in \widetilde{\mathcal{A}}^f_t} \widetilde{V}_t^r(x)\;\;\;$ {\em (see~\eqref{equation::selecting_reward_maximizing_action} and~\eqref{equation::v_t_r} for the definition of $\widetilde{V}_t^r$)}
\STATE Take action $X_t$ and observe reward and cost signals $(R_t,C_t)$
\ENDFOR
\end{algorithmic}
\end{algorithm}

Algorithm~\ref{alg::linear_optimism_pessimism} contains the pseudo-code of our upper confidence bound (UCB) algorithm, which we call Linear Constraint Linear UCB (LC-LUCB). Our algorithm leverages the asymmetrically scaled confidence sets in~\eqref{eq:assymetric-confidence-sets} to appropriately balance its optimism about rewards and pessimism about costs. LC-LUCB starts by constructing a feasible (safe) action set $\widetilde{\mathcal A}_t^f$ from the original action set $\mathcal{A}_t$. In each round $t$, this is done by first computing a {\bf\em pessimistic cost value} for an action $x$ as
\begin{equation}
\label{eq:pessimistic-cost}
\widetilde{V}_t^c(x)  =\underbrace{ \frac{ \langle x^o, e_0\rangle c_0}{\| x_0 \|} }_{\text{known cost along $e_0$}} + \underbrace{\max_{\mu^{o, \perp} \in \mathcal{C}_t^c(\alpha_c)} \langle x^{o, \perp}, \mu^{o, \perp} \rangle }_{\text{max possible cost in $\mathcal{V}_o^{\perp}$}}.
\end{equation}
Note that the known cost of $x$ along $e_0$ equals $\frac{\langle x^o, e_0 \rangle c_0}{\| x_0 \|} $, since $\frac{c_0}{\|x_0 \|}$ is the unit cost in direction $e_0$. Whenever the confidence interval $\mathcal{C}_t^c(\alpha_c)$ holds, $\widetilde{V}_t^c(x)$ overestimates the cost of action $x$ (pessimistic). The feasible action set constructed by LC-LUCB in round $t$, i.e.,~$\widetilde{\mathcal{A}}^f_t$, contains all actions whose pessimistic cost value $\widetilde{V}_t^c(\cdot)$ is at most $\tau$, i.e.,
\begin{equation}
\label{equation::safe_set}
\widetilde{\mathcal{A}}^f_t = \big\{x\in\mathcal{A}_t : \widetilde{V}_t^c(x) \leq \tau\big\}.
\end{equation}
We construct $\widetilde{\mathcal{A}}^f_t$ pessimistically in order to ensure that all its actions are indeed feasible. It is important to note that  $\widetilde{\mathcal{A}}^f_t$ is always non-empty, since as a consequence of Assumption~\ref{ass:safe-action}, the safe action $x_0$ is always in $\widetilde{\mathcal{A}}^f_t$.

LC-LUCB then proceeds by playing optimistically w.r.t.~the reward signal, but only makes use of the feasible actions $x \in \widetilde{\mathcal{A}}^f_t$. In each round $t$, this is done by first computing an {\bf\em optimistic reward value} for every action $x\in\mathcal A_t$ as
\begin{equation}
\label{equation::selecting_reward_maximizing_action}
\widetilde{V}_t^r(x) = \max_{\theta \in \mathcal{C}_t^r(\alpha_r)} \langle x, \theta \rangle, 
\end{equation}
and then playing the arm $X_t$ that maximizes it over the feasible action set $\widetilde{\mathcal{A}}^f_t$ (see Lines~2 and~3 of Algorithm~\ref{alg::linear_optimism_pessimism}). The following proposition contains the closed-form expressions for the {\em pessimistic} cost and {\em optimistic} reward values defined by~\eqref{eq:pessimistic-cost} and~\eqref{equation::selecting_reward_maximizing_action}. 

\begin{proposition}
\label{PROP:OPTIMISTIC-REWARD-PESSIMISTIC-COST}
The optimistic reward and pessimistic cost values in~\eqref{equation::selecting_reward_maximizing_action} and~\eqref{eq:pessimistic-cost} can be written in closed-form as
\begin{align}
\label{equation::v_t_r}
\widetilde{V}_t^r(x) &=\langle x, \widehat{\theta}_t \rangle + \alpha_r \beta_t(\delta,d) \| x\|_{\Sigma_t^{-1}}, \\
\label{equation::v_t_c}
\widetilde{V}_t^c(x) &= \frac{\langle x^o, e_0\rangle c_0 }{\|x_0\|} + \langle x^{o, \perp}, \widehat{\mu}^{o, \perp}_t \rangle  + \alpha_c \beta_t(\delta,d-1) \| x^{o, \perp}\|_{(\Sigma_t^{o, \perp})^{-1}}. 
\end{align}
\end{proposition}
\begin{proof}
See Appendix~\ref{appendix:section-high-prob}.
\end{proof}
The leading term $\frac{\langle x^o, e_0\rangle c_0}{\| x_0 \|}$ in~\eqref{equation::v_t_c} accounts for the knowledge of $\mu_*^{o,\perp}$ derived from the information we possess about the safe action $x_0$ and its cost $c_0$. %
Later we use~\eqref{equation::v_t_r} and~\eqref{equation::v_t_c} to derive a computationally efficient implementation of Algorithm~\ref{alg::linear_optimism_pessimism} for a specific form of the action sets $\{\mathcal{A}_t\}_{t=1}^T$. %

\begin{remark}[unknown $r_0$ and $c_0$] As discussed in Section~\ref{sec:setting}, knowing a safe action $x_0$ is absolutely necessary for solving the constrained contextual linear bandit setting studied in this paper, otherwise, it would be impossible to satisfy the constraint from the very first round. However, we can relax the assumption of knowing the reward $r_0$ and cost $c_0$ of the safe arm. In this case, we start by playing $x_0$ for $T_0$ rounds in order to construct conservative estimates $\widehat{\Delta}_r$ and $\widehat{\Delta}_c$ of the quantities $1-r_0$ and $\tau - c_0$ that satisfy $\widehat{\Delta}_r \geq \frac{1-r_0}{2}$ and $\widehat{\Delta}_c \geq \frac{\tau - c_0}{2}$. We then warm-start our estimators for $\theta_*$ and $\mu_*$ using the data collected by playing $x_0$ and instead of only estimating $\mu_*^{o, \perp}$, we build an estimator for $\mu_*$ over all its directions, including $e_0$, just as LC-LUCB already does for $\theta_*$. Finally, we set $\frac{\alpha_r}{\alpha_c} = \frac{\widehat{\Delta}_r}{\widehat{\Delta}_c}$ and run Algorithm~\ref{alg::linear_optimism_pessimism} for rounds $t > T_0$. The regret incurred during these first $T_0$ rounds can be upper bounded by $\mathcal{O}\left( \log(T/\delta)\max\left(\frac{1-r_0}{(\tau - c_0)^2},\frac{1}{1 - r_0} \right)\right)$. We report the details of this modification of LC-LUCB in Appendix~\ref{section::unknown_c0}. 
\end{remark}

\subsubsection{Computational Tractability of LC-LUCB}
\label{subsubsection:LC-LUCB-Comp-Complexity}

As described above, each round of LC-LUCB involves computing a feasible action set followed by selecting an action that maximizes a linear function over this set. Unfortunately, even if the action set $\mathcal A_t$ is convex, the feasible set $\widetilde{\mathcal A}_t^f$ can have a form for which maximizing the linear function is intractable.\footnote{Note that even in unconstrained linear bandits, the optimization problem that needs to be solved in each round of OFU-style algorithms (e.g.,~\citealt{abbasi2011improved}) can be intractable even when the set is convex. This is because the problem of maximizing a quadratic form over a convex set can be hard in general.} Here we show (see Lemma~\ref{lemma:star-convex}) that whenever the action set $\mathcal{A}_t$ is {\em star-convex} and {\em finite}, (see Definition~\ref{definition::star_convex}), the optimization in Line~2 of LC-LUCB can be solved efficiently. 

\begin{definition}[finite star-convex set]
\label{definition::finite_star_convex}
 We say a star-convex set (see Definition~\ref{definition::finite_star_convex}) is {\bf\em finite}, if there exist finitely many points $\{ x_i \}_{i=1}^M$ such that $\mathcal{A} = \cup_{i=1}^M \{ [x, x_i] \}$.  
\end{definition}

It is important to emphasize that according to Definition~\ref{definition::finite_star_convex}, a finite star-convex set is not necessarily a finite set and can have infinitely many members. We now report the main result of this section that shows when the action sets $\{\mathcal{A}_t\}_{t=1}^T$ are all star-convex and finite, the LC-LUCB algorithm is tractable. We also empirically evaluate LC-LUCB %
in Section~\ref{section::experiments}. %

\begin{lemma}
\label{lemma:star-convex}
If all action sets $\{\mathcal{A}_t\}_{t=1}^T$ are star-convex around the safe action $x_0$ and finite, then LC-LUCB can be implemented in polynomial time.
\end{lemma}

\begin{proof}
We may write each action set $\mathcal A_t$ as $\mathcal{A}_t = \cup_{i=1}^M \{ [x_0, x_i] \}$, because it is star-convex around $x_0$ and finite. Since $x_0\in\widetilde{\mathcal{A}}^f_t$, the feasible action set constructed by LC-LUCB, $\widetilde{\mathcal{A}}^f_t = \mathcal{A}_t \cap \{x : \widetilde{V}_t^c(x)\leq\tau\}$, is also a finite star-convex set around $x_0$ and can be written as $\widetilde{\mathcal{A}}^f_t = \cup_{i=1}^M \{[x_0, \widetilde{x}_i]\}$, where $\widetilde{x}_i = \alpha_i^* x_i$ and $\alpha_i^* = \arg\max_{\alpha \in [0,1], \; \alpha x_i \in \widetilde{\mathcal{A}}^f_t} \alpha$. Solving for $\alpha_i^*$ can be done by a simple line search, hence, Line~2 in Algorithm~\ref{alg::linear_optimism_pessimism} can be executed by optimizing over each ray $[x_0, \widetilde{x}_i],\;\forall i\in [M]$. This optimization is easy because $\widetilde{V}_t^r(x)$ is a convex function of $x$ (see Eq.~\ref{equation::v_t_r}), and thus, its maximum over the one dimensional set $[x_0, \widetilde{x}_i]$ is achieved at either $x_0$ or $\widetilde{x}_i$.
\end{proof}

\subsection{Regret Analysis}
\label{sec:analysis}

In this section, we prove a regret bound for Algorithm~\ref{alg::linear_optimism_pessimism}. Although LC-LUCB can be used in the presence of arbitrary action sets $\mathcal A_t$, we require $\mathcal A_t$ to be star convex around $x_0$ for our regret analysis. Let $\{X_t\}_{t=1}^T$ be the sequence of actions selected by Algorithm~\ref{alg::linear_optimism_pessimism} and $\{\widetilde{V}_t^r(X_t)\}_{t=1}^T$ be their corresponding {\em optimistic reward values} defined by~\eqref{equation::selecting_reward_maximizing_action} and~\eqref{equation::v_t_r}. We start by adding $\{\widetilde{V}_t^r(X_t)\}_{t=1}^T$ to and subtracting them from the regret defined by~\eqref{equation::regret_definition}, and rewriting it as
\begin{equation}
\label{eq:regret-rewritten}
\mathcal{R}_{\mathcal{C}}(T) = \underbrace{\sum_{t=1}^T  V_t^r(x_t^*) - \widetilde{V}_t^r(X_t)}_{\mathrm{(I)}} \;+\; \underbrace{\sum_{t=1}^T \widetilde{V}_t^r(X_t) - V_t^r(X_t)}_{\mathrm{(II)}},
\end{equation}
where for any action $x\in\mathcal A_t$, we denote its {\bf\em true reward value} by $V_t^r(x) = \langle x, \theta_* \rangle$. 

\paragraph{Optimism via Asymmetric Scaling.} In the unconstrained bandit algorithms that are based on the OFU principle (e.g.,~\citealt{abbasi2011improved}), term $\mathrm{(I)}$ in~\eqref{eq:regret-rewritten} is upper-bounded by $0$. This is because most of such algorithms select action $X_t$ that maximizes an optimistic reward value $\widetilde{V}_t^r: \mathcal{A}_t \rightarrow \mathbb{R}$, and thus, satisfies $\widetilde{V}_t^r(x_t) \geq V_t^r(x),\;\forall x \in \mathcal{A}_t$. Unfortunately, this property does not hold for LC-LUCB, because it selects $X_t$ as the maximizer of $\widetilde{V}_t^r(x)$ over the pessimistic set $\widetilde{A}_t^f$ (see Eq.~\ref{equation::selecting_reward_maximizing_action} and Line~2 in Algorithm~\ref{alg::linear_optimism_pessimism}), hence it is possible that $x_t^* \not\in \widetilde{\mathcal{A}}_t^f$. Therefore, it does not immediately follow that $\widetilde{V}_t^r(X_t) \geq V_t^r(x_t^*)$ in LC-LUCB. We get around this limitation using the asymmetrically scaled confidence sets $\mathcal{C}_t^r(\alpha_r)$ and $\mathcal{C}_t^c(\alpha_c)$ defined in~\eqref{eq:assymetric-confidence-sets}. By selecting $\alpha_r$ to be much larger than $\alpha_c$, we ensure that the scaling of $\widetilde{V}_t^r(x)$ is enough to overcome the potential absence of $x_t^*$ in $\widetilde{\mathcal{A}}_t^f$. This imbalanced scaling allows us to enjoy the benefits of optimism without requiring the optimal action $x_t^*$ to be in the estimated set of feasible actions $\widetilde{\mathcal{A}}_t^f$. Although stretching the optimistic reward value $\widetilde{V}_t^r(x)$ allows us to control $\mathrm{(I)}$, the extra scaling causes challenges in bounding $\mathrm{(II)}$. As we will show in Lemma~\ref{LEMMA::LINEAR_BANDITS_OPTIMISM}, the amount of stretching needed for the argument to work for $\mathrm{(II)}$ depends on the ratio between the {\em sub-optimality}, $1-r_0$, and {\em safety gap}, $\tau-c_0$, of the safe action $x_0$. Our results indicate that the smaller the value of $\frac{1-r_0}{\tau - c_0}$, the harder learning becomes. 

Before bounding the two terms in~\eqref{eq:regret-rewritten}, we define the following event that according to Theorem~\ref{theorem::yasin_theorem} holds with probability at least $1-\delta$:
\begin{align}
\label{eq:high-prob-event}
\mathcal{E} := \left\{ \| \widehat{\theta}_t - \theta_* \|_{\Sigma_t} \leq \beta_t(\delta, d) \; \wedge \; \| \widehat{\mu}_t^{o, \perp} - \mu_*^{o, \perp} \|_{\Sigma_t^{o, \perp}} \leq  \beta_t(\delta, d-1), \;\forall t\in[T] \right\}.
\end{align}
\paragraph{Bounding $\mathrm{(II)}$:} Let $\;\widetilde{\theta}_t = \argmax_{\theta \in \mathcal{C}_t^r(\alpha_r)}\max_{x\in\widetilde{\mathcal{A}}_t^f} \; \langle x,\theta \rangle$ be the parameter attaining the optimistic maximum. Since \begin{small}$\widetilde{V}^r_t(X_t) = \langle X_t, \widetilde{\theta}_t \rangle$\end{small}, we may write \begin{small}$\mathrm{(II)} = \sum_{t=1}^T \langle X_t, \widetilde{\theta}_t - \theta_* \rangle$\end{small}. We now state the following proposition that is used in bounding $\mathrm{(II)}$. This proposition is a direct consequence of Eq.~20.9 and Lemma~19.4 in~\citet{lattimore2018bandit}. Similar result has also been reported in the appendix of~\citet{amani2019linear}. 

\begin{proposition}
\label{proposition::det_lemma} 
For any given (possibly random) sequence of actions $\{x_s\}_{s=1}^t$, let $\Sigma_t$ be its corresponding Gram matrix defined by~\eqref{eq:Sigmas} with $\lambda \geq 1$. Then, for all $t\in[T]$, we have
\begin{equation*}
\sum_{s=1}^T \| x_s \|_{\Sigma^{-1}_{s}} \leq \sqrt{2Td\log\left(1+\frac{TL^2}{\lambda}\right)}.
\end{equation*}

\end{proposition}
Armed with Proposition~\ref{proposition::det_lemma}, we now prove an upper-bound for $\mathrm{(II)}$ in the following lemma.

\begin{lemma}
\label{lemma::bounding_B}
On event $\mathcal E$ defined by~\eqref{eq:high-prob-event} (that holds with probability at least $1-\delta$), we have
\begin{equation*}
\mathrm{(II)} \leq \alpha_r \beta_T(\delta,d) \sqrt{2Td\log\left( 1+\frac{TL^2}{\lambda}\right)}.
\end{equation*}
\end{lemma}

\begin{proof}
The following inequalities hold on event $\mathcal E$:
\begin{align*}
\sum_{t=1}^T \langle X_t,\widetilde{\theta}_t \rangle - \langle X_t,\theta_* \rangle &\stackrel{\text{(a)}}{\leq} \; \sum_{t=1}^T \| x_t \|_{\Sigma_t^{-1}}\| \widetilde{\theta}_t - \theta_* \|_{\Sigma_t} \\ &\stackrel{\text{(b)}}{\leq} \; \sum_{t=1}^T  \alpha_r \beta_t(\delta, d)\| X_t \|_{\Sigma_t^{-1}} \stackrel{\text{(c)}}{\leq} \; \alpha_r \beta_T(\delta, d) \sum_{t=1}^T  \| X_t \|_{\Sigma_t^{-1}} \\ &\stackrel{\text{(d)}}{\leq} \; \alpha_r \beta_T(\delta, d)  \sqrt{2Td\log\left(1+\frac{TL^2}{\lambda}\right)}.
\end{align*}
$\textbf{(a)}$ follows from Cauchy Schwartz. $\textbf{(b)}$ is a direct consequence of conditioning on $\mathcal{E}$ that implies $\|\widetilde{\theta}_t - \theta_* \| \leq \alpha_r \beta_t(\delta, d)$. $\textbf{(c)}$ holds because $\beta_t(\delta, d)$ is an increasing function of $t$. $\textbf{(d)}$ follows from Proposition~\ref{proposition::det_lemma}.
\end{proof}

\paragraph{Bounding $\mathrm{(I)}$:} Here we show that by appropriately selecting the reward and cost scaling parameters $\alpha_r$ and $\alpha_c$, we can guarantee optimism for our constrained linear bandit formulation, i.e.,~in each round $t\in[T]$, the optimistic reward value of the action selected by Algorithm~\ref{alg::linear_optimism_pessimism}, $\widetilde{V}_t^r(X_t)$, overestimates the true reward value of the optimal action, $V_t^r(x_t^*)$. This result implies that $\mathrm{(I)}$ can be upper-bounded by $0$. Before proving the main result of this section (Lemma~\ref{LEMMA::LINEAR_BANDITS_OPTIMISM}), we state the following supporting lemma, whose proof is reported in Appendix~\ref{appendix:section-high-prob}.

\begin{restatable}{lemma}{inversenormdominationlemma}\label{LEMMA::INVERSE_NORM_DOMINATION}    
For any $x \in \mathbb{R}^d$, the following inequality holds:
\begin{equation}
\label{equation::comparison_conf_bounds}
\| x^{o, \perp}\|_{(\Sigma_t^{o, \perp})^{\dagger}} \leq \| x \|_{\Sigma_t^{-1}}.
\end{equation}
\end{restatable}

We now find the appropriate conditions on $\alpha_r$ and $\alpha_c$ in order to ensure optimism for Algorithm~\ref{alg::linear_optimism_pessimism}. %

\begin{lemma}
\label{LEMMA::LINEAR_BANDITS_OPTIMISM}
If the scaling parameters $\alpha_r$ and $\alpha_c$ are set such that $\alpha_r,\alpha_c\geq 1$ and $(1+\alpha_c)(1-r_0) \leq (\tau-c_0) (\alpha_r-1)$, then for all $t\in[T]$, with probability at least $1-\delta$, we have $\widetilde{V}_t^r(X_t) \geq V_t^r(x_t^*)$. 
\end{lemma}

\begin{proof}
On event $\mathcal{E}$, for any action $x\in\mathcal A_t$, we have
\begin{equation}
\label{equation::domination_max_main}
\widetilde{V}_t^r(x) = \max_{\theta \in \mathcal{C}_t^r(\alpha_r)} \langle x,\theta \rangle \geq \langle x,\theta_* \rangle = V_t^r(x).
\end{equation}
We divide the proof into two cases depending on whether in each round $t\in[T]$, the optimal action $x^*_t$ belongs to the set of feasible actions $\widetilde{\mathcal{A}}_t^f$, or not.

\paragraph{Case~1.} When $x_t^* \in \widetilde{\mathcal{A}}_t^f$, the result follows immediately, since by definition $X_t$ is a maximizer of $\widetilde{V}_t^r(x)$ over $\widetilde{\mathcal{A}}_t^f$, and thus, we have
\begin{equation}
\label{equation::domination_assumption_main}
\widetilde{V}_t^r(X_t) \geq \widetilde{V}_t^r(x_t^*). 
\end{equation}
Combining~\eqref{equation::domination_max_main} and~\eqref{equation::domination_assumption_main}, we can conclude that $\widetilde{V}_t(X_t) \geq V_t^r(x_t^*)$ as desired.

\paragraph{Case~2.} When $x_t^* \not\in \widetilde{\mathcal{A}}_t^f$, we know that the {\em pessimistic cost value} of the optimal action violates the constraint, i.e.,~$\widetilde{V}_t^c(x_t^*) > \tau$, while its {\bf\em true cost value} satisfies the constraint, i.e.,~$V_t^c(x_t^*) := \langle x_t,\mu_* \rangle \leq \tau$.  Since $\mathcal{A}_t$ is assumed to be star-convex around $x_0$, action $ \gamma x_t^* + (1-\gamma)x_0 \in \mathcal{A}_t$ for all $\gamma \in [0,1]$. Now consider the following mixture action $\widetilde{x}_t = \gamma_t x_t^* + (1-\gamma_t)x_0$, where $\gamma_t \in [0,1]$ is the maximum value of $\gamma$ for which the mixture action belongs to the estimated set of feasible actions, i.e.,~$\widetilde{x}_t \in \widetilde{\mathcal{A}}_t^f$. Since $\mathcal A_t$ is star-convex, all actions $ \gamma x_t^* + (1-\gamma)x_0 $ for $\gamma \leq \gamma_t$ are in $\widetilde{\mathcal{A}}_t^f$. From the definition of $\widetilde{x}_t$, we have 
\begin{equation}
\label{equation::x_tilde_projection_main}
\widetilde{x}_t^{o, \perp}= \gamma_t  x_t^{*,o,\perp},
\end{equation}
which allows us to write
\begin{align}
\label{eq:temp0}
\widetilde{V}_t^c(x_t^*) &\stackrel{\text{(a)}}{=} \frac{\langle x_t^{*,o}, e_0\rangle c_0}{\|x_0\| } + \langle x_t^{*,o,\perp}, \widehat{\mu}^{o, \perp}_t \rangle + \alpha_c \beta_t(\delta, d-1)\| x_t^{*,o,\perp}\|_{(\Sigma^{o, \perp}_t)^{\dagger}}, \\
\widetilde{V}_t^c(\widetilde{x}_t) &\stackrel{\text{(b)}}{=} \frac{\left( \gamma_t \langle x_t^{*,o}, e_0 \rangle + (1-\gamma_t ) \langle x_0 , e_0 \rangle\right)c_0 }{\| x_0\| } + \gamma_t \langle x_t^{*,o,\perp}, \widehat{\mu}^{o, \perp}_t\rangle + \gamma_t \alpha_c \beta_t(\delta, d-1) \| x_t^{*,o,\perp} \|_{(\Sigma_t^{o, \perp})^{\dagger}} \nonumber \\ 
&\stackrel{\text{(c)}}{=} (1-\gamma_t ) c_0  + \gamma_t \widetilde{V}_t^c(x_t^*).
\label{eq:temp00}    
\end{align}
{\bf (a)} is from the definition of pessimistic cost value in~\eqref{equation::v_t_c}, {\bf (b)} is obtained from the definition of $\widetilde{x}_t$, together with~\eqref{equation::v_t_c} and~\eqref{equation::x_tilde_projection_main}, and finally, {\bf (c)} comes directly from~\eqref{eq:temp0}.

Since $x_t^* \not\in \widetilde{\mathcal{A}}_t^f$, from the definition of $\gamma_t$, it is easy to see that $\widetilde{V}_t^c(\widetilde{x}_t) = \tau$. Using this fact and~\eqref{eq:temp00}, we first write $\gamma_t$ in terms of $\widetilde{V}_t^c(x^*_t)$ and then with the following chain of inequalities obtain a lower-bound on $\gamma_t$ as 
\begin{align}
\gamma_t &= \frac{\tau- c_0}{\widetilde{V}_t^c(x_t^*) - c_0} \notag \\ 
&= \frac{\tau - c_0}{ 
\frac{\langle x_t^{*,o}, e_0\rangle c_0}{\|x_0\| } + \langle x_t^{*,o, \perp}, \widehat{\mu}^{o, \perp}_t \rangle + \alpha_c \beta_t(\delta, d-1)\| x^{*,o, \perp}_{t}\|_{(\Sigma^{0, \perp}_t)^{\dagger}} - c_0} \notag \\
&= \frac{\tau- c_0}{ \frac{\langle x_t^{*,o}, e_0\rangle c_0}{\|x_0\| } +
\langle x_t^{*,o, \perp}, \mu^{o, \perp}_*\rangle + \langle x_t^{*,o, \perp}, \widehat{\mu}^{o, \perp}_t - \mu^{o, \perp}_* \rangle + \alpha_c \beta_t(\delta, d-1)\| x_t^{*,o, \perp}\|_{(\Sigma^{o, \perp}_t)^{\dagger}} - c_0} \notag \\
&\stackrel{\text{(a)}}{\geq} \frac{\tau- c_0}{\frac{\langle x_t^{*,o}, e_0\rangle c_0}{\|x_0\| } +
\langle x_t^{*,o, \perp}, \mu^{o, \perp}_*\rangle +   (1+\alpha_c) \beta_t(\delta, d-1)\| x^{*,o, \perp}_{t}\|_{(\Sigma_t^{o, \perp})^{\dagger}} - c_0} \notag \\
&\stackrel{\text{(b)}}{\geq} \frac{\tau- c_0}{ \tau - c_0 + (1+\alpha_c) \beta_t(\delta, d-1) \| x^{*,o, \perp}_{t}\|_{(\Sigma_t^{o, \perp})^{\dagger}}}. 
\label{equation::gamma_t_lower_bound}
\end{align}
{\bf (a)} holds because 
\begin{equation*}
\langle x^{*,o, \perp}_{t}, \widehat{\mu}^{o, \perp}_t - \mu^{o, \perp}_* \rangle \leq \|\widehat{\mu}^{o, \perp}_t - \mu^{o, \perp}_*\|_{\Sigma^{o, \perp}_t} \|x^{*,o, \perp}_{t} \|_{(\Sigma^{o, \perp}_t)^{\dagger}} \leq  \beta_t(\delta, d-1) \| x^{*,o, \perp}_{t} \|_{(\Sigma^{o, \perp}_t)^{\dagger}}.
\end{equation*}
{\bf (b)} holds because $x_t^*$ is the optimal action in round $t$, and thus, $\frac{\langle x_t^{*,o}, e_0\rangle c_0}{\|x_0\| } + \langle x^{*,o, \perp}_{t}, \mu^{o, \perp}_* \rangle \leq \tau$.

Now let's assume that $\langle x_0, \theta_* \rangle = \langle x_t^*, \theta_*\rangle - \Delta_t$ for all $t\in [T]$. Since both $X_t$ and $\widetilde{x}_t$ are in the feasible set $\widetilde{\mathcal{A}}^f_t$, and given the definition of $X_t$, we may write
\begin{align}
\widetilde{V}_t^r(X_t) \geq \widetilde{V}_t^r(\widetilde{x}_t) &= \langle \widetilde{x}_t, \widehat{\theta}_t \rangle + \alpha_r \beta_t(\delta,d) \| \widetilde{x}_t \|_{\Sigma_t^{-1}} \notag \\ 
&= \langle \widetilde{x}_t, \theta_* \rangle + \langle \widetilde{x}_t, \widehat{\theta}_t - \theta_* \rangle + \alpha_r \beta_t(\delta,d) \| \widetilde{x}_t \|_{\Sigma_t^{-1}} \notag\\
&\stackrel{\text{(a)}}{\geq} \langle \widetilde{x}_t, \theta_* \rangle +  (\alpha_r - 1) \beta_t(\delta, d) \| \widetilde{x}_t \|_{\Sigma_t^{-1}} \notag \\ &\stackrel{\text{(b)}}{\geq} \langle \widetilde{x}_t, \theta_* \rangle + (\alpha_r - 1) \beta_t(\delta, d-1) \| \widetilde{x}_t^{o, \perp}\|_{(\Sigma^{o, \perp}_t)^\dagger}\notag \\
&\stackrel{\text{(c)}}{=} \gamma_t \langle x^*_{t}, \theta_* \rangle + (1-\gamma_t)\langle x_0, \theta_* \rangle + \gamma_t (\alpha_r - 1) \beta_t(\delta, d-1) \| x_t^{*,o,\perp} \|_{(\Sigma^{o, \perp}_t)^\dagger} \notag \\
&\stackrel{\text{(d)}}{=}  \langle x^*_{t}, \theta_* \rangle  -(1-\gamma_t)\Delta_t + \gamma_t (\alpha_r - 1) \beta_t(\delta, d-1) \| x_t^{*,o,\perp} \|_{(\Sigma^{o, \perp}_t)^\dagger} \notag \\
&=  \langle x^*_{t}, \theta_* \rangle + \underbrace{\gamma_t\left((\alpha_r - 1) \beta_t(\delta, d-1) \| x_t^{*,o,\perp} \|_{(\Sigma^{o, \perp}_t)^\dagger}  +\Delta_t \right)  -\Delta_t }_{\mathrm{(V)}}. 
\label{eq:I-def}
\end{align}

\noindent
{\bf (a)} follows from the definition of event $\mathcal{E}$ in~\eqref{eq:high-prob-event} and Cauchy Schwartz, i.e.,
\begin{equation*}
|\langle \widetilde{x}_t, \widehat{\theta}_t - \theta_* \rangle| \leq \| \widehat{\theta}_t - \theta_*   \|_{\Sigma_t} \| \widetilde{x}_t \|_{\Sigma_t^{-1}} \leq  \beta_t(\delta, d) \| \widetilde{x}_t   \|_{\Sigma_t^{-1}}. 
\end{equation*}
{\bf (b)} is a consequence of Lemma~\ref{LEMMA::INVERSE_NORM_DOMINATION}. {\bf (c)} is from the definition of $\widetilde{x}_t$ and~\eqref{equation::x_tilde_projection_main}. {\bf (d)} follows from the assumption that $\langle x_0, \theta_* \rangle = \langle x_t^*, \theta_*\rangle - \Delta_t$. 

Now we derive conditions under which term $\mathrm{(V)}$ in~\eqref{eq:I-def} is non-negative. To reduce notation clutter let $C_1 := \beta_t(\delta, d-1) \| x_t^{*,o,\perp} \|_{(\Sigma^{o, \perp}_t)^\dagger} $. Then, the following inequality holds for $\mathrm{(V)}$:
\begin{equation*}
    \mathrm{I} \geq \frac{\tau-c_0}{\tau - c_0 + (1+\alpha_c) C_1} \big((\alpha_r - 1) C_1 + \Delta_t \big) - \Delta_t,
\end{equation*}
where the inequality follows by lower-bounding $\gamma_t$ using~\eqref{equation::gamma_t_lower_bound}. 

Consequently if $\frac{\tau-c_0}{\tau - c_0 + (1+\alpha_c) C_1} \left((\alpha_r - 1) C_1 + \Delta_t \right) - \Delta_t \geq 0$, then $\mathrm{(V)}$ will be non-negative, which holds whenever
\begin{equation}
\label{equation::condition_optimism_high_prob}
    (\tau - c_0) (\alpha_r- 1) \geq (1+\alpha_c) \Delta_t.
\end{equation}
By the definition of $\Delta_t$ and the fact that rewards are bounded in $[0,1]$ (Assumption~\ref{ass:bounded-mean-reward-cost}), we have $\Delta_t\leq 1-r_0$. Thus, inequality~\eqref{equation::condition_optimism_high_prob} holds if $(\tau - c_0) (\alpha_r- 1) \geq (1+\alpha_c) (1-r_0)$. This concludes the proof, since we proved that $\widetilde{V}_t^r(X_t)\geq V_t^r(x_t^*)$ in both cases where $x_t^*\in\widetilde{\mathcal A}_t^f$ and $x_t^*\notin\widetilde{\mathcal A}_t^f$.
\end{proof}

After bounding the two terms in~\eqref{eq:regret-rewritten} using Lemmas~\ref{lemma::bounding_B} to~\ref{LEMMA::LINEAR_BANDITS_OPTIMISM}, we are now ready to state the main theorem of this section, which is a regret bound for Algorithm~\ref{alg::linear_optimism_pessimism}. %

\begin{theorem}[regret bound for LC-LUCB]
\label{thm:regret-bound-HPLOP}
Let $\alpha_c = 1$ and $\alpha_r = 1+ \frac{ 2(1-r_0)}{\tau-c_0}$. Then, with probability at least $1-\delta$, the regret of Algorithm~\ref{alg::linear_optimism_pessimism} can be upper-bounded as
\begin{equation}
\mathcal{R}_{\mathcal{C}}(T) \leq   \alpha_r \beta_T(\delta, d) \sqrt{2Td\log\left( 1+\frac{TL^2}{\lambda}\right)}.
\label{eq:regret-bound-LC-LUCB}
\end{equation}
\end{theorem}

\begin{proof}
The proof follows directly from bounding $\mathrm{(I)}$ and $\mathrm{(II)}$ in the regret decomposition~\eqref{eq:regret-rewritten} using Lemmas~\ref{lemma::bounding_B} to~\ref{LEMMA::LINEAR_BANDITS_OPTIMISM}.
\end{proof}

\begin{remark}
When $\lambda = 1$, ignoring logarithmic dependencies on $T$ and $1/\delta$, the term $\beta_T(\delta, d)$ in~\eqref{eq:regret-bound-LC-LUCB} is of order $\sqrt{d}$. Thus, this parameter setting yields a regret bound of order $\mathcal{R}_\mathcal{C}(T) = \widetilde{\mathcal{O}}\big((1+\frac{1-r_0}{\tau-c_0})d\sqrt{T}\big)$, which shows that LC-LUCB recovers the same dependence on $d$ and $T$ as unconstrained OFU-style linear bandit algorithms (e.g.,~\citealt{abbasi2011improved}). The extra term of $\widetilde{\mathcal{O}}(\frac{1-r_0}{\tau-c_0}d\sqrt{T})$ is the cost of satisfying the constraint and the multiplier $\frac{1-r_0}{\tau-c_0}$ represents the hardness of the constrained problem.
\end{remark}

\subsection{Lower Bound}
\label{subsec:lower-bound}

We also prove a min-max lower-bound for the constrained contextual linear bandit setting described in Section~\ref{sec:setting}. We prove in Theorem~\ref{theorem::lower_bound} that no algorithm can obtain a regret better than $\mathcal{O}\big(\max( d\sqrt{T}, \frac{1-r_0}{(\tau - c_0)^2})\big)$ on all such constrained contextual linear bandit instances. This result substantiates our intuition that learning while satisfying linear constraints is statistically harder than the unconstrained case, particularly when the safety gap $\tau - c_0$ is small w.r.t.~the horizon $T$ and the reward suboptimality $1-r_0$. %
\begin{restatable}{theorem}{theoremlowerboundlclucb}\label{theorem::lower_bound}
Let $\tau,c_0, r_0\in (0,1)$, $B = \max\big(\frac{d\sqrt{T}}{8e^2},\frac{1-r_0}{21(\tau - c_0)^2}\big)$, and assume $T \geq \max(d-1,\frac{168eB}{1-r_0})$. Then, for any algorithm $\mathfrak{A}$, there is a pair of reward and cost parameters $(\theta_*, \mu_*)$, such that $\mathcal{R}_\mathcal{C}(T) \geq B$. %
\end{restatable}

\begin{proof}
See Appendix~\ref{appendix::High_probability_Lower_Bound}.    
\end{proof}
Theorem~\ref{theorem::lower_bound} shows that if $\frac{1}{\tau - c_0} = \Omega(\sqrt{T})$ and $r_0 \leq 1/2$, learning while satisfying linear constraints is impossible in a min-max sense since in this case our lower bound  indicates the regret must grow at least linearly. As an additional example, if $\frac{1}{\tau - c_0 }= \sqrt{d}\;T^{3/8}$ and $r_0 \leq 1/2$, Theorem~\ref{theorem::lower_bound} implies that a constrained learner must incur $\Omega(d \; T^{3/4})$ regret, while unconstrained learning can achieve a regret rate of order $d\sqrt{T}$. This shows the existence of a fundamental statistical separation between constrained and unconstrained learning as a function of the ratio between the safety gap $\tau - c_0$ and the reward suboptimality $1-r_0$. The question of whether the quadratic dependence on $\tau-c_0$ is optimal in this lower-bound remains open. %

\subsection{Extension to Multiple Constraints}
\label{sec:multi-constraints}

The formulation, algorithm, and analysis of %
the previous sections can be extended to multiple constraints. 
In this setting, when the agent takes an action $X_t$ in each round $t\in [T]$, in addition to the reward signal $R_t$, it observes a vector of $m$ cost signals $C_t^{(i)} = \langle X_t, \mu_*^{(i)}\rangle + \xi_t^{c(i)},\;\forall i\in[m]$, where the reward and costs satisfy the assumptions %
listed in Section~\ref{sec:setting}. The agent is required to satisfy $m$ {\bf\em stage-wise linear constraints} $\langle X_t, \mu_*^{(i)} \rangle \leq \tau_i,\;\forall i\in[m]$. Here we also need the following assumption for the safe action, which is a generalization of Assumption~\ref{ass:safe-action} to multiple constraints. %

\begin{assumption}[safe action]
\label{assumption::known_safe_arm_multiple}
There is a known {\em safe action} $x_0\in\mathcal A_t,\;\forall t\in[T]$, with known reward $r_0$ and costs $c_0^{(i)}=\langle x_0, \mu^{(i)}_* \rangle \leq \tau_i,\;\forall i\in[m]$. 
\end{assumption}

In extending LC-LUCB to multiple constraints, we maintain estimators $\big\{\widehat{\mu}_t^{o, \perp (i)}\big\}_{i=1}^m$ for all cost parameters $\big\{\mu_*^{(i)}\big\}_{i=1}^m$, and construct the feasible action set as
\begin{equation}
\label{equation::safe_set1}
\widetilde{\mathcal{A}}^f_t = \big\{x\in\mathcal{A}_t : \widetilde{V}_t^{c(i)}(x) \leq \tau_i,\;\forall i\in [m]\big\},
\end{equation}
where $\widetilde{V}_t^{c(i)}(\cdot)$ is the pessimistic cost value for the $i$'th cost signal. The rest of the algorithm remains unchanged. 

To derive a regret bound for the extension of LC-LUCB to multiple constraints, it suffices to prove the following extension of Lemma~\ref{LEMMA::LINEAR_BANDITS_OPTIMISM}.

\begin{restatable}{lemma}{linearbanditsoptimismmultiple}\label{LEMMA::LINEAR_BANDITS_OPTIMISM_MULTIPLE}
If we set the scaling parameters $\alpha_r$ and $\alpha_c$ such that $\alpha_r,\alpha_c \geq 1$ and $(1+\alpha_c)(1-r_0) \leq \Delta_c^* (\alpha_r-1)$, where $\Delta_c^* = \min_{i\in[m]}(\tau - c_0^{(i)})$, then for all $t\in[T]$, with probability at least $1-\delta$, we have $\widetilde{V}_t^r(X_t) \geq V^r(x_t^*)$.
\end{restatable}

\begin{proof}
A simple modification to the proof of Lemma~\ref{LEMMA::LINEAR_BANDITS_OPTIMISM} yields the desired result. Note that substituting $\tau - c_0$ with $\Delta_c^*$ and following the same argument as in the derivation of inequality~\eqref{equation::gamma_t_lower_bound} yields
\begin{equation}
\label{equation::multi_constraint_lower_bound}
\gamma_t \geq \frac{\Delta_c^*}{\Delta_c^* + (1+\alpha_c ) \beta_t(\delta, d-1) \| x_t^{*, o, \perp}\|_{(\Sigma_t^{o, \perp})^{\dagger}}}.
\end{equation}
Plugging this result into~\eqref{eq:I-def} and continuing with the proof logic of Lemma~\ref{LEMMA::LINEAR_BANDITS_OPTIMISM} concludes the proof.
\end{proof}

\vspace{-0.05in}
Lemma~\ref{LEMMA::LINEAR_BANDITS_OPTIMISM_MULTIPLE} allows us to derive a regret bound for the extension of LC-LUCB to multiple constraints, identical to the one we proved for the single-constraint case in Theorem~\ref{thm:regret-bound-HPLOP}.

\begin{theorem}
Let $\alpha_c = 1$ and $\alpha_r = 1+ 2(1-r_0)/\Delta_c^*$. Then, with probability at least $1-\delta$, the regret of the extension of LC-LUCB to multiple constraints can be upper-bounded as
\begin{equation}
\mathcal{R}_{\mathcal{C}}(T) \leq \alpha_r \beta_T(\delta, d) \sqrt{2Td\log\left(1+\frac{TL^2}{\lambda}\right)}.
\end{equation}
\end{theorem}

We show in Appendix~\ref{section::unknown_c0} how the multi-constraint algorithm (similar to the single-constraint case) can be changed to handle the scenario where the reward $r_0$ and costs $\{c^{(i)}_0\}_{i=1}^m$ of the safe action are unknown.

\subsection{Experiments}
\label{section::experiments}

In this section we compare the performance of LC-LUCB and the Safe-LTS algorithm of~\citet{Moradipari19SL} in two simulation-based experiments. In each of these scenarios we show that LC-LUCB performs better than Safe-LTS. In all our experiments, we run a regularized least-squares regression by setting $\lambda = 1$.

In our first experiment, presented in Figure~\ref{fig::figure_linear_constrained_1}, we consider a linear bandit problem in which the safe action is the zero-vector $x_0 = 0$ and the arm sets, $\mathcal{A}_t$, are $10$ dimensional star-convex sets generated by the $10$ cyclic shifted versions of the vector $v/\| v\|$, where $v = (0, 1, \ldots, 9)$. For all $t$, the action set $\mathcal{A}_t$ is the star-convex set defined by this set of actions and the lines emanating from the zero vector. We set $\theta_* = v/\|v\|$ where $v = (0, 1, \ldots, 9)$ and\footnote{The vector $(9, 8, \ldots, 0)$ is the flipped version of $v$. } $\mu_* = (9, 8, \ldots, 0)/\| v\|$ to be the $\ell_2$ normalized version of $v$ and $(9, \ldots, 0)$. In Figure~\ref{fig::figure_linear_constrained_1}, we plot the regret and cost evolution of LC-LUCB for different threshold values $\tau$, and compare them with those for the Safe-LTS algorithm of~\citet{Moradipari19SL}. The safe action is the zero vector and each plot is an average over $10$ runs. We show that as the threshold $\tau$ is driven to $0$, the problem gets progressively harder. The results show that for all threshold values and dimensions, LC-LUCB has a better regret profile than Safe-LTS, while satisfying the constraint. We report the results for dimensions $d=3$ and $d=5$ of this problem, and also show the reward evolution (in addition to regret and cost) for LC-LUCB and Safe-LTS in Appendix~\ref{sec:experiments}.

\begin{figure*}
\centering
\includegraphics[width=0.29\linewidth]{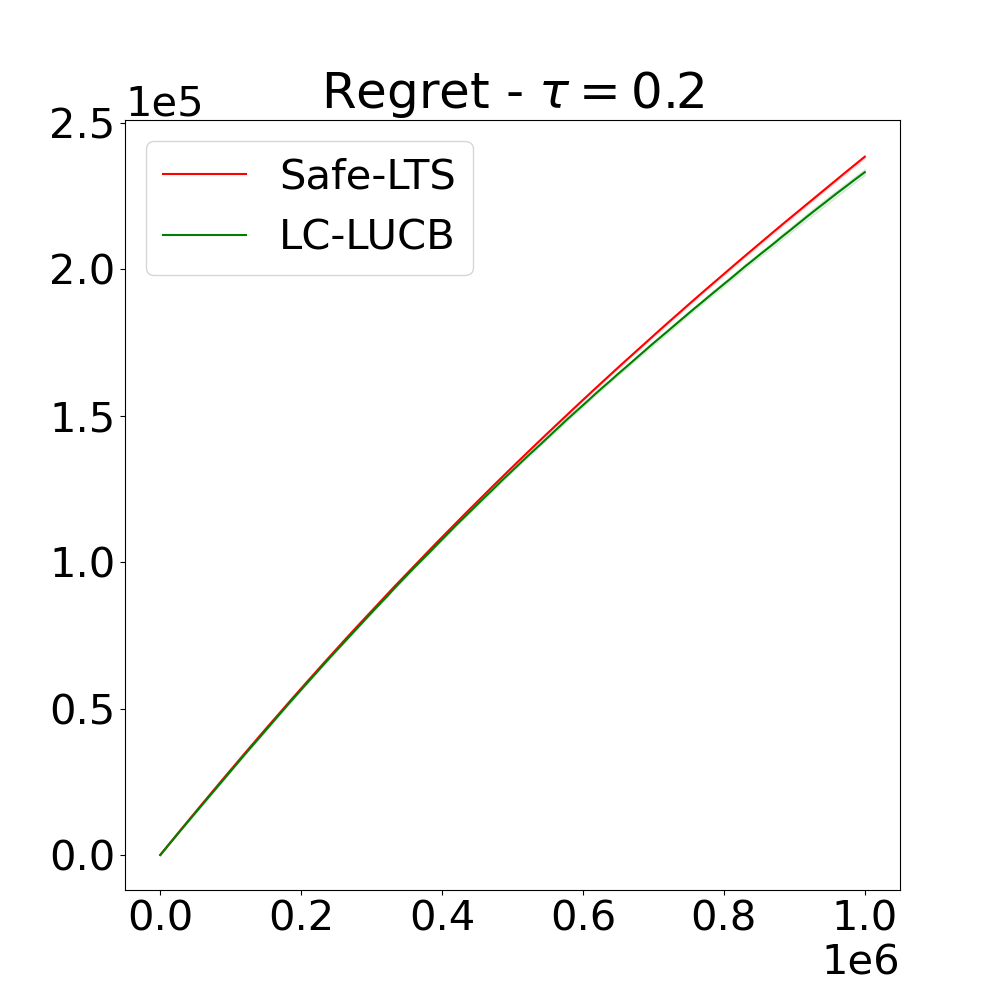}
\centering
\includegraphics[width=0.29\linewidth]{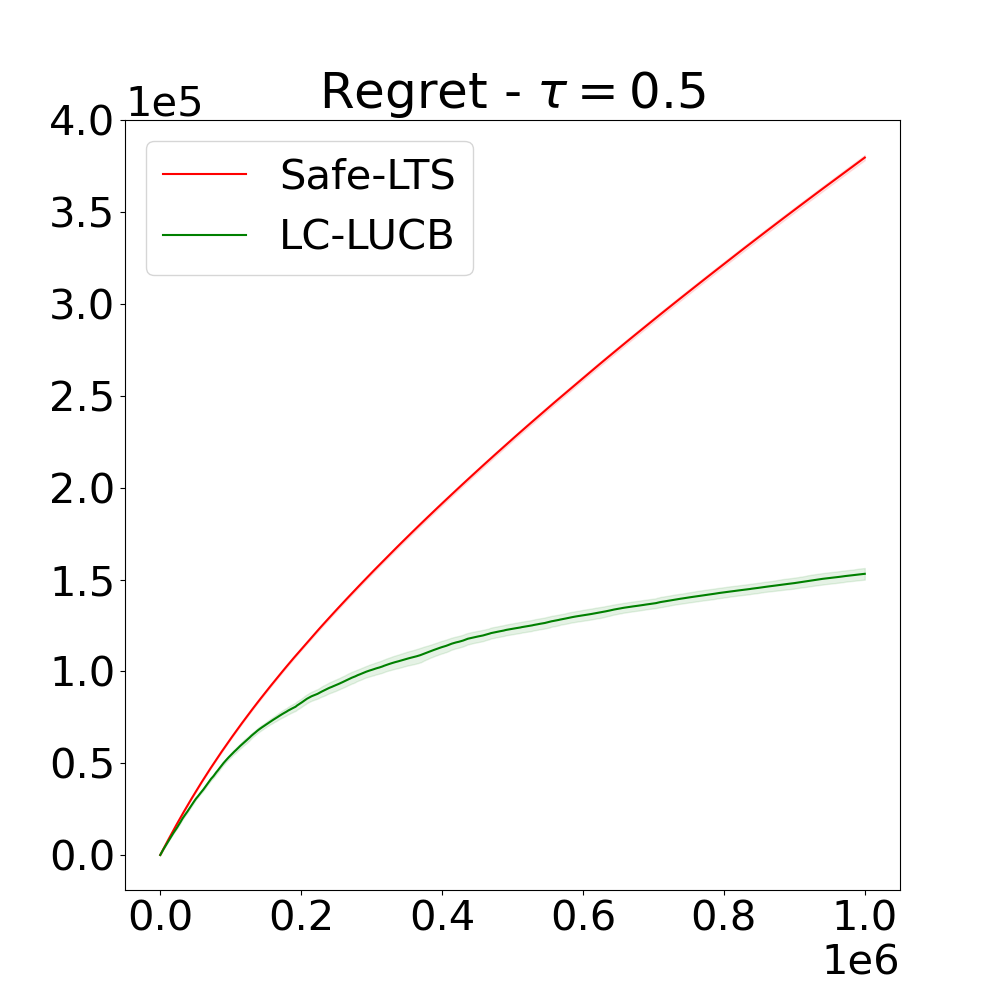} 
\centering
\includegraphics[width=0.29\linewidth]{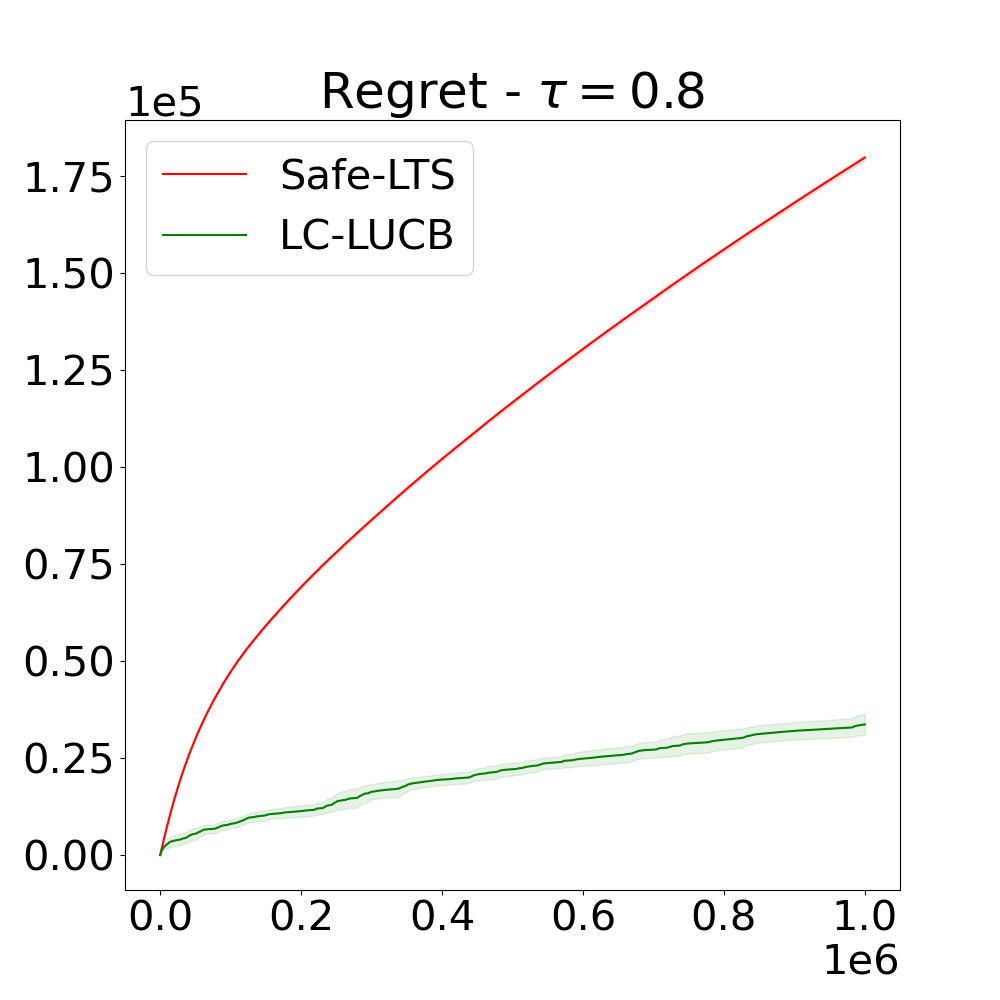}
\centering
\includegraphics[width=0.29\linewidth]{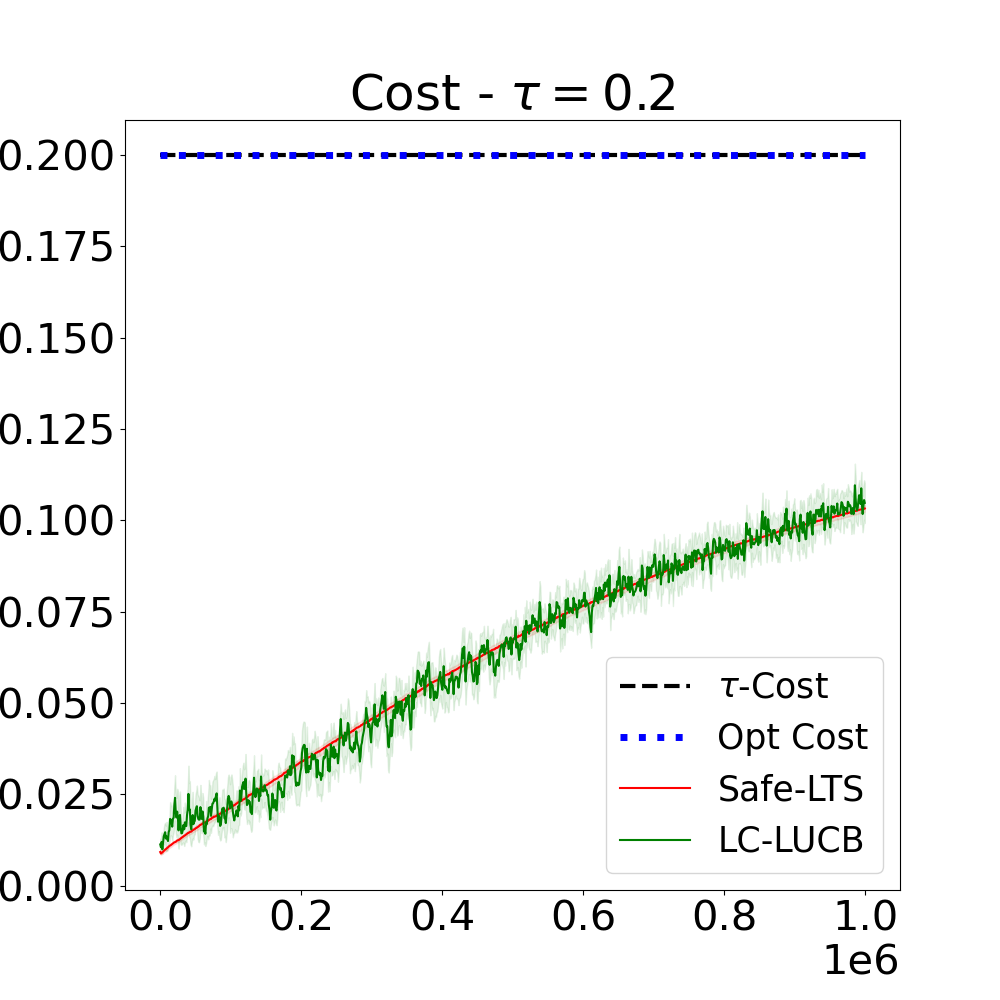}
\centering
\includegraphics[width=0.29\linewidth]{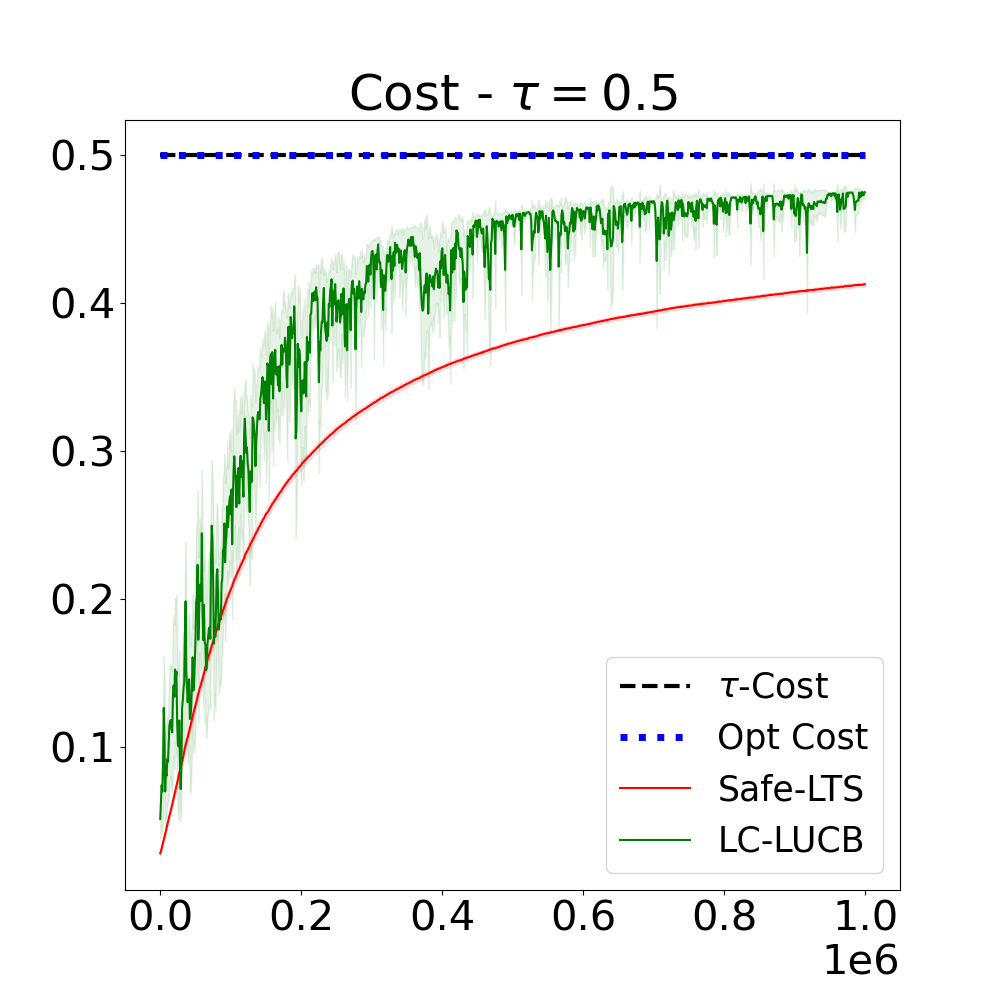}
\centering
\includegraphics[width=0.29\linewidth]{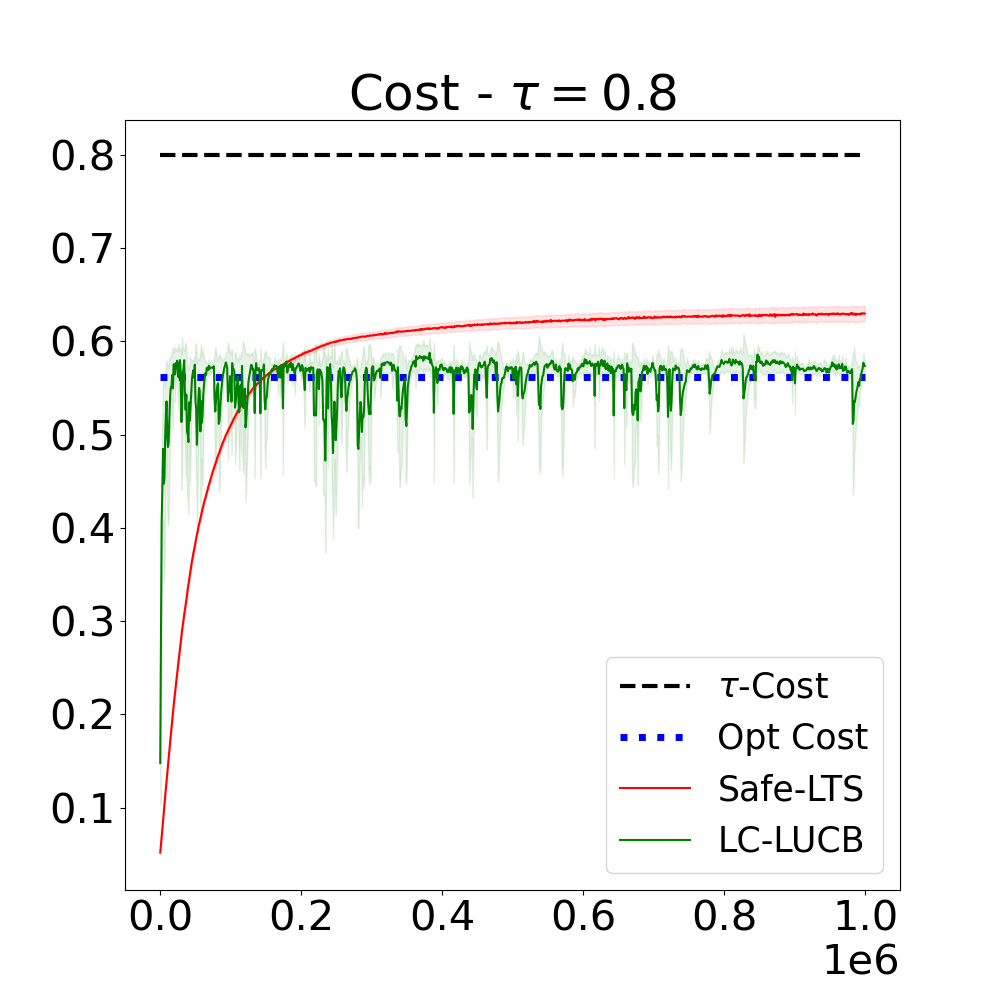}
\caption{Regret and cost evolution of \textbf{LC-LUCB} and \textbf{Safe-LTS}. \textbf{Left:} Constraint Threshold $\tau=0.2$. \textbf{Center:} Constraint Threshold $\tau = 0.5$. \textbf{Right:} Constraint Threshold $\tau = 0.8$.  Learning is harder for smaller thresholds $\tau$. The arm sets $\mathcal{A}_t$ are $10$ dimensional star-convex sets generated by the $10$ cyclic shifted versions of the vector $v/\| v\|$, where $v = (0, 1, \ldots, 9)$. There exist optimal unconstrained solutions with cost less than 0.561. This is less than the cost threshold $\tau = 0.8$. This is why the lower right cost evolution plot shows convergence to a level below the $0.8$ threshold.} \label{fig::figure_linear_constrained_1}
\end{figure*}

In our second experiment, presented in Figure~\ref{fig::figure_unit_sphere}, we consider the setting where the action sets $\mathcal{A}_t$ are the unit ball (infinite) and the safe action is the zero vector. Our plots compare the regret across time of Safe-LTS and LC-LUCB when averaged across problem instances. We generate different problem instances by sampling $\theta_*$ and $\mu_*$ vectors uniformly from the unit sphere and also generate thresholds $\tau$ by sampling uniformly from the interval $[0,1]$. Each sample run of this experiment corresponds to a sample problem instance. In order to make the optimization problem at each round of LC-LUCB tractable for this infinite size action set, we approximate $\mathcal{A}_t$ by sampling $100$ vectors $\{v_i\}^{1000}_{i=1}$ uniformly from the unit sphere and defining an approximate (still infinite) action set $\tilde{\mathcal{A}}_t$, consisting of the rays from zero to each of the $v_i$. Figure~\ref{fig::figure_unit_sphere} compares the regret of LC-LUCB with Safe-LTS for dimensions $d=5$ and $d=10$ of this problem. Each plot is averaged over $10$ sample runs and the shaded regions around the curves correspond to $1$ standard deviation. Similar to the previous experiment, LC-LUCB shows better performance than Safe-LTS.

\begin{figure}%
        \centering{\includegraphics[width=0.3\linewidth]{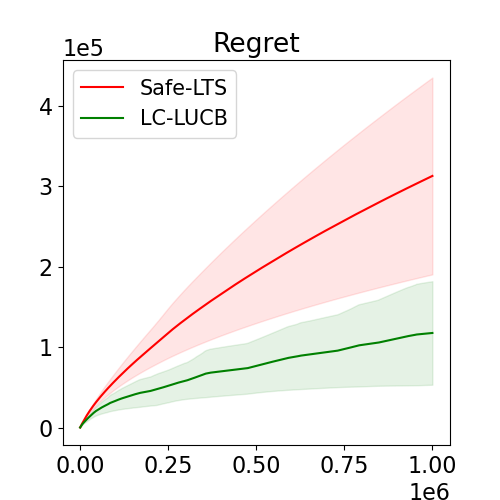}}
        \hspace{0.5in}
        \centering{\includegraphics[width=0.3\linewidth]{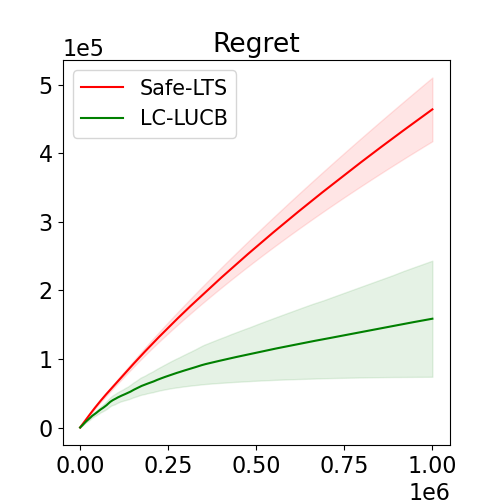}}
        \caption{ Unit sphere, \textbf{left } dimension $d = 5$,  \textbf{right} dimension $d=10$.}
        \label{fig::figure_unit_sphere}
\end{figure}

\section{Constraint Relaxation: From High Probability to Expectation}
\label{section::expectation_constraints}

In Section~\ref{section::high_probability_constraints}, we studied a constrained contextual linear bandit setting in which the agent maximizes its expected $T$-round cumulative reward (minimizes its expected $T$-round constrained pseudo-regret) while satisfying a stage-wise linear high probability constraint defined in~\eqref{eq:high-prob-constraint}. In many constrained or multi-objective problems, making sure that the constraints are not violated or certain objectives are within certain thresholds with high probability would result in overly conservative strategies. A common solution to balance performance and constraint satisfaction is to replace conservative {\em high probability} constraints with more relaxed {\em in-expectation} constraints. In this section, we study such a relaxation in which the high probability constraint~\eqref{eq:high-prob-constraint} is replaced with a constraint in expectation. We describe this relaxed setting in Section~\ref{subsec:expectation-setting}. We then propose an algorithm for this setting, provide its regret analysis, specialize our results to multi-armed bandits (MABs),\footnote{Unlike the high probability constrained setting of Section~\ref{section::high_probability_constraints} that is impossible to solve it in MABs (see Remark~\ref{remark:extension-MAB}), the relaxed setting we study in this section can be solved for MABs.} and report experimental results as a proof of concept. Additionally, in Section~\ref{section::nonlinear_rewards_costs} we extend these results to the scenario where the reward and cost functions are non-linear. We propose an algorithm for this setting with regret analysis. We use a characterization of function class complexity based on the eluder dimension~\citep{russo2013eluder} in the derivation of this regret bound.

\subsection{Linear Contextual Bandits with In-expectation Stage-Wise Linear Constraints}
\label{subsec:expectation-setting}
 
In the relaxed setting we study in this section, in each round $t\in[T]$, the agent selects its action $X_t\in\mathcal A_t$ according to its policy $\pi_t\in\Pi_t=\Delta_{\mathcal A_t}$, i.e.,~$X_t\sim\pi_t$. The goal of the agent is to produce a sequence of policies $\{\pi_t\}_{t=1}^T$ with maximum expected cumulative reward over the course of $T$ rounds, while satisfying the {\bf\em stage-wise linear constraint} %
\begin{equation}
\label{eq:constraint}
\mathbb E_{X\sim\pi_t}[\langle X,\mu_* \rangle]\leq \tau,\quad\forall t\in[T], 
\end{equation}
where $\tau\geq 0$ is the {\em constraint threshold}. Note that unlike the constraint~\eqref{eq:high-prob-constraint} studied in Section~\ref{section::high_probability_constraints} which is in high probability, this constraint is in expectation. 

The policy $\pi_t$ that the agent selects in each round $t\in[T]$ should belong to the set of {\em feasible policies} over the action set $\mathcal A_t$, i.e.,~$\Pi_t^f=\{\pi\in\Pi_t:\mathbb E_{X\sim\pi}[\langle X,\mu_* \rangle]\leq \tau\}$. Maximizing the expected cumulative reward in $T$ rounds is equivalent to minimizing the $T$-round {\em constrained (pseudo)-regret}%
\begin{equation}
\label{eq:regret}
\mathcal R_\Pi(T) = \sum_{t=1}^T \mathbb E_{X\sim\pi_t^*}[\langle X,\theta_* \rangle] - \mathbb E_{X\sim\pi_t}[\langle X,\theta_* \rangle], 
\end{equation}
where $\pi_t,\pi^*_t\in\Pi_t$ for all $t\in[T]$, and $\pi^*_t\in\max_{\pi\in\Pi_t^f}\mathbb E_{X\sim\pi}[\langle X,\theta_* \rangle]$ is the {\em optimal feasible} policy in round $t$. The terms $\mathbb E_{X\sim\pi}[\langle X,\theta_* \rangle]$ in~\eqref{eq:regret} and $\mathbb E_{X\sim\pi}[\langle X,\mu_* \rangle]$ in~\eqref{eq:constraint} are the expected reward and cost of policy $\pi$, respectively. Thus, a feasible policy is the one whose expected cost is below the constraint threshold $\tau$, and the optimal feasible policy is a feasible policy with maximum expected reward. We use the shorthand notations $x_\pi := \mathbb{E}_{X \sim \pi}[X]$, %
$R_\pi:=\mathbb E_{X\sim\pi}[\langle X,\theta_* \rangle]$, and $C_\pi:=\mathbb E_{X\sim\pi}[\langle X,\mu_* \rangle]$ for the expected action, reward, and cost of a policy $\pi$, respectively. With these notations, we may write the $T$-round regret in~\eqref{eq:regret} as 
\begin{equation}
\label{eq:regret2}
\mathcal R_\Pi(T)=\sum_{t=1}^T R_{\pi_t^*} - R_{\pi_t}.
\end{equation}
{\bf Notation.} Here we use some extra notations in addition to those defined in Section~\ref{sec:setting}. We define the projection of a policy $\pi$ into $\mathcal{V}_o$ and $\mathcal{V}_o^{\perp}$, as $x_{\pi}^o := \mathbb{E}_{X \sim \pi}[X^o]$ and $x_{\pi}^{o, \perp} := \mathbb{E}_{X \sim \pi}[X^{o, \perp}]$, respectively.

\subsubsection{Algorithm}
\label{sec:algo_expectation}

We propose a UCB-style algorithm for this setting, called {\em optimistic-pessimistic linear bandit} (OPLB), because it maintains a pessimistic assessment of the set of available policies, while acting optimistically within this set. Algorithm~\ref{alg:optimistic-pessimistic-LB} contains the pseudo-code of OPLB. Similar to LC-LUCB (Algorithm~\ref{alg::linear_optimism_pessimism}), the main idea behind Algorithm~\ref{alg:optimistic-pessimistic-LB} is to balance exploration and constraint satisfaction by {\em asymmetrically} scaling the radii of the reward and cost confidence sets with different scaling factors $\alpha_r$ and $\alpha_c$. This will prove crucial in the regret analysis of OPLB. We now describe in detail the steps of OPLB that differ from the LC-LUCB algorithm.  
 
\begin{algorithm}[t]
\caption{Optimistic-Pessimistic Linear Bandit (OPLB)}
\label{alg:optimistic-pessimistic-LB}    
\begin{algorithmic}[1]
\STATE {\bfseries Input:} Safe action $x_0$ with reward $r_0$ and cost $c_0$; $\;$ Constraint threshold $\tau\geq 0$; $\;$ Scaling parameters $\alpha_r, \alpha_c \geq 1$ 
\FOR{$t=1,\ldots,T$}
\STATE Compute regularized least-squares estimates $\;\widehat{\theta}_t\;$ and $\;\widehat{\mu}_t^{o,\perp}$ using~\eqref{eq:theta-mu}
\STATE Construct confidence sets $\;\mathcal{C}_t^r(\alpha_r)\;$ and $\;\mathcal{C}_t^c(\alpha_c)$ using~\eqref{eq:assymetric-confidence-sets} 
\STATE Observe $\mathcal{A}_t$ and construct the estimated feasible policy set $\;\widetilde{\Pi}^f_t$ using~\eqref{eq:safe-policy-set}
\STATE Compute policy $\;(\pi_t,\widetilde{\theta}_t) = \argmax_{\pi \in \widetilde{\Pi}^f_t,\;\theta \in \mathcal{C}_t^r(\alpha_r)} \mathbb{E}_{x \sim \pi}[ \langle x, \theta \rangle]$
\STATE Take action $X_t \sim \pi_t$ and observe reward and cost signals $(R_t,C_t)$ 
\ENDFOR
\end{algorithmic}
\end{algorithm}

Just as in the analysis of LC-LUCB, the choice of $\alpha_r,\alpha_c\geq 1$ and Theorem~\ref{theorem::yasin_theorem} suggest that $\theta_* \in \mathcal{C}_t^r(\alpha_r)$ and $\mu_*^{o,\perp}\in\mathcal{C}_t^c(\alpha_c)$, each w.p.~at least $1-\delta$. Replacing actions with policies in~\eqref{eq:pessimistic-cost} and~\eqref{equation::selecting_reward_maximizing_action}, we define the {\em optimistic reward} and {\em pessimistic cost} values for any policy $\pi$ in round $t$ as 
\begin{align}
\label{eq:max-reward}
\widetilde{V}_t^r(\pi) &:= \max_{\theta\in\mathcal{C}_t^r(\alpha_r)} \mathbb{E}_{X \sim \pi}[\langle X,\theta\rangle], \\
\widetilde{V}_t^c(\pi) &:= \frac{\langle x_{\pi}^o,e_0 \rangle c_0}{\| x_0 \|} + \max_{\mu\in\mathcal{C}_t^c(\alpha_c)} \mathbb{E}_{X\sim\pi}[\langle X,\mu \rangle].
\label{eq:min-cost}
\end{align}
Similar to Proposition~\ref{PROP:OPTIMISTIC-REWARD-PESSIMISTIC-COST}, we can derive the following closed-form expressions for $\widetilde{V}_t^r(\pi)$ and $\widetilde{V}_t^c(\pi)$. 

\begin{restatable}{proposition}{expectationoptimisticrewardpessimisticcost}\label{prop:optimistic-reward-pessimistic-cost}
We may write $\widetilde{V}_t^r(\pi)$ and $\widetilde{V}_t^c(\pi)$, defined in~\eqref{eq:max-reward} and~\eqref{eq:min-cost}, in closed-form as %
\begin{align}
\label{eq:optimistic-reward-closed-form}
\widetilde{V}_t^r(\pi) &= \langle x_\pi,\widehat{\theta}_t \rangle + \alpha_r \beta_t(\delta,d) \| x_\pi\|_{\Sigma_t^{-1}}, \\ 
\label{eq:pessimistic-cost-closed-form}
\widetilde{V}_t^c(\pi) &= \frac{\langle x_\pi^o, e_0\rangle c_0 }{\|x_0\|} + \langle x_\pi^{o, \perp}, \widehat{\mu}^{o, \perp}_t \rangle  + \alpha_c \beta_t(\delta,d-1) \| x_{\pi}^{o, \perp}\|_{(\Sigma_t^{o, \perp})^{-1}}. 
\end{align}
\end{restatable}

\begin{proof}
See Appendix~\ref{sec:proofs-algo-section}.
\end{proof}

Following the same logic as in the high probability formulation, we adapt the pessimistic estimation of the feasible action set $\widetilde{\mathcal{A}}^f_t$ in~\eqref{equation::safe_set} to the policy setting. After observing the action set $\mathcal A_t$, OPLB constructs its feasible (safe) policy set as
\begin{equation}
\label{eq:safe-policy-set}
\widetilde{\Pi}^f_t = \big\{\pi \in \Delta_{\mathcal A_t} : \widetilde{V}_t^c(\pi) \leq \tau\big\}. 
\end{equation}
Note that $\widetilde{\Pi}^f_t$ is an approximation of $\Pi_t^f$, i.e.,~the set of feasible policies over the action set $\mathcal A_t$, and is not an empty set because $\pi_0$ is always in $\widetilde{\Pi}^f_t$. We can think of the safe action $x_0$ as a policy $\pi_0$ whose probability mass is all on $x_0$, and thus, we have $x_{\pi_0}^o = x_0$ and $x_{\pi_0}^{o,\perp} =0$. Plugging $\pi_0$ into~\eqref{eq:pessimistic-cost-closed-form} yields $\widetilde V_t^c(\pi_0) =\frac{\langle x_{\pi_0}^o, e_0\rangle c_0}{\| x_0 \|}  = c_0 \leq \tau$. We prove in the following proposition that all policies in $\widetilde{\Pi}^f_t$ are feasible with high probability.

\begin{restatable}{proposition}{propositionsafesetexpectation}\label{prop:safe-set}
With probability at least $1-\delta$, for all rounds $t\in[T]$, all policies in $\widetilde{\Pi}^f_t $ are feasible. %
\end{restatable}

\begin{proof}
See Appendix~\ref{sec:proofs-algo-section}.
\end{proof}

In Line~6 of Algorithm~\ref{alg:optimistic-pessimistic-LB}, the agent computes its policy $\pi_t$ as the one that is safe (belongs to $\widetilde{\Pi}^f_t $) and attains the maximum optimistic reward value, i.e.,~$\widetilde V_t^r(\pi) = \max_{\theta \in \mathcal{C}_t^r(\alpha_r)} \langle x_\pi, \theta \rangle = \langle x_{\pi_t},\widetilde{\theta}_t \rangle$. We refer to $\widetilde{\theta}_t$ as the {\em optimistic reward parameter}. 

\paragraph{Computational Complexity of OPLB.} {As shown in Line~6 of Algorithm~\ref{alg:optimistic-pessimistic-LB}, in each round $t$, OPLB solves the following optimization problem:
\begin{align}
\label{equation::opt_problem_linear}
&\max_{\pi \in \Delta_{\mathcal{A}_t}} \; \langle x_\pi, \widehat{\theta}_t \rangle + \alpha_r \beta_t(\delta,d) \| x_\pi\|_{\Sigma_t^{-1}}, \\ %
&\;\;\text{s.t. } \quad\! \frac{\langle x_\pi^o, e_0\rangle c_0 }{\|x_0\|} + \langle x_\pi^{o, \perp}, \widehat{\mu}^{o, \perp}_t \rangle + \alpha_c \beta_t(\delta,d-1) \| x_{\pi}^{o, \perp}\|_{(\Sigma_t^{o, \perp})^{-1}} \leq \tau. \nonumber
\end{align}
However, solving~\eqref{equation::opt_problem_linear} could be challenging. The bottleneck is in computing the safe policy set $\widetilde{\Pi}^f_t$, which is the intersection of $\Delta_{\mathcal{A}_t}$ and the ellipsoidal constraint. This is a consequence of the intractability of the optimization problem that needs to be solved in each round of OFU-style algorithms (e.g.,~\citealt{abbasi2011improved}). In contrast to LC-LUCB (see Section~\ref{subsubsection:LC-LUCB-Comp-Complexity}), solving~\eqref{equation::opt_problem_linear} could be intractable even when the action set $\mathcal{A}_t$ is finite star-convex around the safe action $x_0$.  }

\paragraph{Unknown $r_0$ and $c_0$.} In case that the reward $r_0$ and cost $c_0$ of the safe action are unknown, OPLB may use the same warm-starting sub-routine as in LC-LUCB to estimate them with sufficient accuracy. The regret incurred during these first $T_0$ rounds can be upper bounded by $\mathcal{O}\left( \log(T/\delta)\max\left(\frac{1-r_0}{(\tau - c_0)^2},\frac{1}{1 - r_0} \right)\right)$. We discuss this in more details in Appendix~\ref{section::unknown_c0}.

\subsubsection{Regret Analysis}
\label{section::lin_opt_pess_analysis}

In this section, we prove a regret bound for the OPLB algorithm. The main challenge in obtaining this regret bound is to ensure that optimism holds in each round $t\in [T]$, i.e.,~the solution $(\pi_t,\widetilde\theta_t)$ of~\eqref{equation::opt_problem_linear} satisfies $\widetilde{V}_t^r(\pi_t) = \langle x_{\pi_t},\widetilde\theta_t \rangle \geq V_t^r(\pi^*_t)$. This is not obvious, since the approximate feasible policy set $\widetilde \Pi^f_t$ might have been  constructed such that it does not contain the optimal policy $\pi_t^*$. Similar to the analysis of LC-LUCB in Section~\ref{sec:analysis}, our main algorithmic innovation is the use of asymmetric confidence intervals $\mathcal{C}_t^r(\alpha_r)$ and $\mathcal{C}_t^c(\alpha_c)$ for $\theta_*$ and $\mu^{o, \perp}_*$, respectively, that allows us to guarantee optimism by appropriately selecting the ratio $\frac{\alpha_r}{\alpha_c}$. We will show in our analysis that similar to the case of LC-LUCB, $\frac{\alpha_r}{\alpha_c}$ depends on the ratio $\frac{1-r_0}{\tau-c_0}$.

\begin{theorem}[Regret Bound for OPLB]
\label{theorem::main_theorem_linear}
Let $\alpha_c = 1$ and $\alpha_r =1+ \frac{ 2(1-r_0)}{\tau-c_0}$. Then, with probability at least $1-2\delta$, the regret of OPLB satisfies %
\begin{align}
\label{eq:regret-OPLB}
\mathcal{R}_{\Pi}(T) &\leq \frac{2L(\alpha_r+1)\beta_T(\delta, d)}{\sqrt{\lambda}}\sqrt{2T\log(1/\delta)}+ (\alpha_r+1)\beta_T(\delta, d)\sqrt{2Td\log\left(1+\frac{TL^2}{\lambda}\right)}\;. 
\end{align}
\end{theorem}

We provide a proof sketch for Theorem~\ref{theorem::main_theorem_linear} here. Most of the results are obtained in a similar fashion as those in the analysis of the LC-LUCB algorithm in Section~\ref{sec:analysis}. We use the high probability event $\mathcal{E}$ defined by~\eqref{eq:high-prob-event}. We first decompose the regret $\mathcal{R}_{\Pi}(T)$ in~\eqref{eq:regret} as 
\begin{equation}
\label{eq:regret-decomp}
\mathcal{R}_{\Pi}(T) = \underbrace{\sum_{t=1}^T V_t^r(\pi_t^*) - \widetilde{V}_t^r(\pi_t)}_{(\mathrm{I})} \; + \; \underbrace{\sum_{t=1}^T \widetilde{V}_t^r(\pi_t) -V_t^r(\pi_t)}_{(\mathrm{II})}\;,
\end{equation}
where for any policy $\pi \in \Pi^f_t$, we denote by $V_t^r(\pi) = \langle x_\pi, \theta_* \rangle$, its {\bf\em true reward value} and by $\widetilde{V}_t^r(\pi_t)$ its optimistic reward value defined by~\eqref{eq:max-reward} and~\eqref{eq:optimistic-reward-closed-form}. We first bound term $(\mathrm{II})$ in~\eqref{eq:regret-decomp} by further decomposing it as
\begin{equation*}
\label{eq:II-decomp}
(\mathrm{II}) \; = \; \underbrace{\sum_{t=1}^T  \langle x_{\pi_t}, \widetilde{\theta}_t \rangle - \langle X_t, \widetilde{\theta}_t \rangle}_{(\mathrm{III})} \; + \; \underbrace{\sum_{t=1}^T \langle X_t, \widetilde{\theta}_t \rangle - \langle X_t, \theta_* \rangle}_{(\mathrm{IV})} \; + \; \underbrace{\sum_{t=1}^T \langle X_t, \theta_* \rangle - \langle x_{\pi_t}, \theta_* \rangle }_{(\mathrm{V})}\;. 
\end{equation*}
When the event $\mathcal{E}$ defined by~\eqref{eq:high-prob-event} holds, $(\mathrm{IV})$ can be bounded by Lemma~\ref{lemma::bounding_B} as

\begin{equation*}
(\mathrm{IV}) \leq \alpha_r \beta_T(\delta,d) \sqrt{2Td\log\left( 1+\frac{TL^2}{\lambda}\right)}\;.
\end{equation*}

We bound the sum of $(\mathrm{III})$ and $(\mathrm{V})$ in the following lemma.

\begin{lemma}
\label{lemma:bounding-3+5}
On the event $\mathcal{E}$ in~\eqref{eq:high-prob-event}, for any $\delta' \in (0, 1)$, with probability at least $1-\delta'$, we have
\begin{equation*}
(\mathrm{III}) + (\mathrm{V}) \leq \frac{2L(\alpha_r + 1)\beta_T(\delta, d)}{\sqrt{\lambda}}\cdot\sqrt{2 T \log(1/\delta')}\;.
\end{equation*}
\end{lemma}

\begin{proof}
Applying Cauchy-Schwartz to $(\mathrm{III}) + (\mathrm{V}) = \sum_{t=1}^T \; \langle x_{\pi_t} - x_t, \widetilde{\theta}_t - \theta_* \rangle$, we may write $|\langle x_{\pi_t} - X_t, \widetilde{\theta}_t - \theta_* \rangle | \leq \| x_{\pi_t} - X_t \|_{\Sigma^{-1}_t} \| \widetilde{\theta}_t - \theta_*\|_{\Sigma_t}$. Since $\widetilde{\theta}_t\in \mathcal{C}_t^r(\alpha_r)$ on event $\mathcal E$, we have $\|\widetilde{\theta}_t - \theta_*\|_{\Sigma_t} \leq \alpha_r \beta_t(\delta, d)$. From the definition of $\Sigma_t$, we have $\Sigma_t \succeq \lambda I$, and thus, $\| x_{\pi_t} - X_t \|_{\Sigma^{-1}_t} \leq \| x_{\pi_t} - X_t\|/\sqrt{\lambda} \leq 2L/\sqrt{\lambda}$. Hence, $Y_t =\sum_{s=1}^t \langle x_{\pi_s} -X_s, \widetilde{\theta}_s - \theta_* \rangle$ is a martingale sequence with $|Y_t - Y_{t-1}| \leq 2L(\alpha_r + 1)\beta_t(\delta, d)/\sqrt{\lambda}$, for all $t\in [T]$. Using the Azuma–Hoeffding inequality and since $\beta_t$ is an increasing function of $t$, i.e.,~$\beta_t(\delta, d) \leq \beta_T(\delta, d),\;\forall t\in[T]$, with probability at least $1-\delta'$, we may write $\mathbb{P}\big(Y_T \geq 2L\alpha_r \beta_T(\delta, d)\sqrt{2T \log(1/\delta')/\lambda}\big) \leq \delta'$, which concludes the proof. 
\end{proof}

We now bound term $(\mathrm{I})$ in~\eqref{eq:regret-decomp}. Similar to the regret proof for LC-LUCB, setting the values of $\alpha_r$ and $\alpha_c$ to $1$ and then solving for $\pi_t$ is not enough to ensure $\widetilde{V}_t^r(\pi_t) \geq V_t^r(\pi_t^\star)$. However, an appropriate choice of radii $\alpha_r$ and $\alpha_c$ for the confidence intervals can help us to get around this issue. Lemma~\ref{lemma:linear_bandits_optimism} contains the main result in which we prove that by appropriately setting $\alpha_r$ and $\alpha_c$, we can guarantee that in each round $t\in[T]$, OPLB selects an optimistic policy, i.e.,~a policy $\pi_t$ whose optimistic reward, $\widetilde{V}_t^r(\pi_t)$, is larger than the reward of the optimal policy, $V_t^r(\pi_t^*)$, on event $\mathcal E$. This means that with this choice of $\alpha_r$ and $\alpha_c$, $(\mathrm{I})$ is always non-positive. This result is the in-expectation version of the one proved in Lemma~\ref{LEMMA::LINEAR_BANDITS_OPTIMISM}.

\begin{lemma}
\label{lemma:linear_bandits_optimism}
On the event $\mathcal{E}$ defined by~\eqref{eq:high-prob-event}, if we set $\alpha_r$ and $\alpha_c$ such that $\alpha_r,\alpha_c\geq 1$ and $(1+\alpha_c)(1-r_0) \leq (\tau-c_0) (\alpha_r-1)$, then for any $t\in[T]$, we have $\widetilde{V}_t^r(\pi_t) \geq V_t^r(\pi_t^*)$.
\end{lemma}

Lemma~\ref{lemma:linear_bandits_optimism} is a corollary of Lemma~\ref{LEMMA::LINEAR_BANDITS_OPTIMISM}. The exact same proof argument holds with a few notational substitutions. The optimal action $x_t^* \in \widetilde{\mathcal{A}}_t^f$ is substituted with $x_{\pi_t^*}$, which is the average action vector of $\pi_t^* \in \widetilde \Pi^f_t$. Consequently, the mixture action $\widetilde{X}_t$ is substituted with $x_{\widetilde{\pi}_t}$, where $\widetilde{\pi}_t=\eta_t\pi_t^*+(1-\eta_t)\pi_0$, $\pi_0$ is the safe policy that always selects the safe action $x_0$, and $\eta_t\in[0,1]$ is the maximum value of $\eta$ for which the mixture policy belongs to the set of feasible policies, i.e.,~$\widetilde{\pi}_t\in \widetilde\Pi^f_t$. The rest of the proof remains unchanged.

\subsection{Specializing to Multi-Armed Bandits}
\label{sec:constrained-MAB}

We now specialize the results of this section to multi-armed bandits (MABs) and show that the structure of the MAB problem allows a computationally efficient implementation of our OPLB algorithm as well as an improvement in its regret bound.  

In the MAB setting, the action set consists of $K$ arms $\mathcal{A}=\{1,\ldots,K\}$, where each arm $a\in [K]$ has a reward and a cost distribution with means $\bar{r}_a,\bar{c}_a\in [0,1]$. In each round $t\in[T]$, the agent constructs a policy $\pi_t$ over $\mathcal A$, pulls an arm $A_t\sim\pi_t$, and observes a reward-cost pair $(R_{A_t},C_{A_t})$ sampled i.i.d.~from the reward and cost distributions of arm $A_t$. Similar to the constrained contextual linear bandit case studied above, the goal of the agent is to produce a sequence of policies $\{\pi_t\}_{t=1}^T$ with maximum expected cumulative reward over $T$ rounds, i.e.,~$\sum_{t=1}^T\mathbb E_{A_t\sim\pi_t}[\bar{r}_{A_t}]$, while satisfying the {\em stage-wise linear constraint} $\mathbb E_{A_t\sim\pi_t}[\bar{c}_{A_t}] \leq \tau,\;\forall t\in[T]$. Moreover, Arm~$1$ is assumed to be the known safe arm, i.e.,~its mean reward $\bar{r}_1$ and cost $\bar{c}_1$ are known, and $\bar{c}_1 \leq \tau$. 

\paragraph{Optimistic Pessimistic Bandit (OPB) Algorithm.} Let $\{T_a(t)\}_{a=1}^K$ and $\{\widehat{r}_a(t),\widehat{c}_a(t)\}_{a=1}^K$ be the total number of times that arm $a$ has been pulled and the estimated mean reward and cost of arm $a$ up until round $t$. In each round $t\in[T]$, OPB relies on the high-probability upper-bounds on the mean reward and cost of the arms, i.e.,~$\{u_a^r(t),u_a^c(t)\}_{a=1}^K$, where $u_a^r(t) = \widehat{r}_a(t) + \alpha_r\beta_a(t)$, $u_a^c(t) = \widehat{c}_a(t) + \alpha_c\beta_a(t)$, $\beta_a(t) = \sqrt{2\log(1/\delta')/T_a(t)}$, and constants $\alpha_r,\alpha_c\geq 1$. In order to produce a feasible policy, OPB solves the following linear program (LP) in each round $t\in[T]$: 
\begin{equation}
\label{eq::noisy_LP}
\max_{\pi\in\Delta_K} \; \sum_{a\in\mathcal{A}}\pi_a \; u^r_a(t), \qquad \text{s.t.} \;\; \sum_{a\in\mathcal{A}}\pi_a \; u_a^c(t) \leq \tau.
\end{equation}
As~\eqref{eq::noisy_LP} indicates, OPB selects its policy by being optimistic about reward and pessimistic about cost (using an upper-bound for both $r$ and $c$). Algorithm~\ref{alg::optimism_pessimism} contains the pseudo-code of OPB. Note that similar to OPLB, OPB constructs an (approximate) feasible (safe) policy set of the form $\widetilde\Pi_t^f = \big\{ \pi \in \Delta_K : \sum_{a \in \mathcal{A}} \pi_a u_a^c(t) \leq \tau \big\}$ (see Eq.~\ref{eq::noisy_LP}) and sets $\beta_a(0) = 0,\;\forall a\in \mathcal{A}$. 

\begin{algorithm}[t]
\caption{Optimism-Pessimism Bandit (OPB)}
\label{alg::optimism_pessimism}
\begin{algorithmic}[1]
\STATE \textbf{Input:} Number of arms $K$; $\;$ Mean reward $\bar{r}_1$ and cost $\bar{c}_1$ of the safe arm; $\;$ Constraint threshold $\tau\geq 0$; $\;$ Scaling parameters $\alpha_r, \alpha_c \geq 1$
\FOR{$t=1,\ldots,T$}
\STATE Compute estimates $\{u_a^r(t) \}_{a\in\mathcal A}$ and $\{u_a^c(t)  \}_{a\in\mathcal A}$ \label{alg-MAB-est}
\STATE Form the approximate LP (\ref{eq::noisy_LP}) using the estimates in Line~\ref{alg-MAB-est}
\STATE Compute policy $\pi_t$ by solving~\eqref{eq::noisy_LP}
\STATE Play arm $a \sim \pi_t$
\ENDFOR
\end{algorithmic}
\end{algorithm}

\paragraph{Computational Complexity of OPB.} {Unlike OPLB whose optimization problem might be complex to solve, OPB can be implemented extremely efficiently. The following lemma shows that~\eqref{eq::noisy_LP} always has a solution (policy) with support of at most $2$. This property allows us to solve~\eqref{eq::noisy_LP} in closed-form without a LP solver and to implement OPB very efficiently.} 

\begin{lemma}
\label{lemma::LP_support} 
There exists a policy that solves~\eqref{eq::noisy_LP} and has at most $2$ non-zero entries.
\end{lemma}

\begin{proof}
See Appendix~\ref{section::LP_structure_appendix}.
\end{proof}

\paragraph{Regret Analysis of OPB.} {We also prove the following regret-bound for OPB.

\begin{restatable}{theorem}{theoremoptregret}\label{theorem::contrained_MAB}
Let $\delta = 4KT\delta'$, $\alpha_c=1$, and $\alpha_r = 1+\frac{2(1-\bar{r}_1)}{\tau-\bar{c}_1}$. Then, with probability at least $1-\delta$, the regret of OPB satisfies
\begin{align*}
\mathcal{R}_\Pi(T) &\leq \left(1 + \frac{2(1-\bar{r}_1)}{\tau-\bar{c}_1}\right) \times \left(2\sqrt{2KT\log(4KT/\delta)} + 4\sqrt{T\log(2/\delta)\log(4KT
/\delta)}\right).
\end{align*}
\end{restatable}

\begin{proof}
See Appendix~\ref{section:regret_analysis_appendix}. 
\end{proof}

The main component of this proof is the following lemma, which is the analogous to Lemma~\ref{lemma:linear_bandits_optimism} (and therefore also Lemma~\ref{LEMMA::LINEAR_BANDITS_OPTIMISM}) in the contextual linear bandit case. 

\begin{restatable}{lemma}{optimismMAB}
\label{lemma::optimism}
If we set $\alpha_r$ and $\alpha_c$ such that $\alpha_r,\alpha_c\geq 1$ and $(1+\alpha_c) (1-\bar{r}_1) \leq (\tau-\bar{c}_1) (\alpha_r-1) $, then with high probability, for any $t\in[T]$, we have $\mathbb{E}_{a \sim \pi_t}\left[u_a^r(t)\right] \geq \mathbb{E}_{a \sim \pi^*}\left[\bar{r}_a \right]$.
\end{restatable}
}

\begin{proof}
See Appendix~\ref{section:regret_analysis_appendix}.
\end{proof}

\begin{remark}
The constrained contextual linear bandit formulation of Section~\ref{subsec:expectation-setting} is general enough to include the constrained MAB one described here. As a result the regret bound of OPLB in Theorem~\ref{theorem::main_theorem_linear} can be instantiated for the constrained MAB setting, in which case it yields a regret bound of order $\widetilde{\mathcal O}\big((1+\frac{1-\bar{r}_1}{\tau-\bar{c}_1} ) K\sqrt{T}\big)$. However, our OPB regret bound in Theorem~\ref{theorem::contrained_MAB} is of order $\widetilde{\mathcal O}((1+\frac{1-\bar{r}_1}{\tau-\bar{c}_1} ) \sqrt{KT})$, which shows $\sqrt{K}$ improvement over simply casting MAB as an instance of contextual linear bandit and using the regret bound of OPLB. %
\end{remark}

\paragraph{Lower-bound.} {We also prove a min-max lower-bound for our constrained MAB problem. Our lower-bound shows that no algorithm can attain a regret better than $\mathcal{O}\big(\max(\sqrt{KT},\frac{1-\bar r_1}{(\tau - \bar{c}_1)^2})\big)$ for this problem. The formal statement of the lower-bound can be found below.} %

\begin{restatable}{theorem}{theoremlowerboundlclucbmab}\label{theorem::lower_bound_mab}
Let $\tau,\bar c_1, \bar r_1\in (0,1)$, $B = \max\big(\frac{1}{27}\sqrt{(K-1)T},\frac{1-\bar r_1}{21(\tau - \bar c_1)^2}\big)$, and assume $T \geq \max(K-1,\frac{168eB}{1-\bar{r}_1})$. Then, for any algorithm there is a pair of reward and cost parameters $(\theta_*, \mu_*)$, such that $\mathcal{R}_\mathcal{C}(T) \geq B$. %
\end{restatable}

\begin{proof}
See Appendix~\ref{appendix::High_probability_Lower_Bound}.
\end{proof}

\paragraph{Extension to Multiple Constraints.} {In the case of $m$ constraints, the learner receives $m$ cost signals after pulling each arm. The cost vector of the safe arm $\boldsymbol{c}_1$ satisfies $\boldsymbol{c}_1(i)<\tau_i,\forall i\in[m]$, where $\{\tau_i\}_{i=1}^m$ are the constraint thresholds. Similar to single-constraint OPB, multi-constraint OPB is computationally efficient. The main reason is that the LP of $m$-constraint OPB has a solution with at most $(m+1)$ non-zero entries. We also obtain a regret-bound of $\widetilde{\mathcal{O}}(\frac{\sqrt{KT }}{\min_{i\in [K]} (\tau_i - \boldsymbol{c}_1(i))})$ for $m$-constraint OPB. The proofs and details of this case are reported in Appendix~\ref{section::multiple_constraints_appendix}.}

\paragraph{Unknown $\bar{c}_1$ and $\bar{r}_1$.} The same warm starting sub-routine as in LC-LUCB and OPLB can be used for computing sufficiently accurate estimators of $\bar{r}_1$ and $\bar{c}_1$ for OPB. A detailed explanation can be found in Appendix~\ref{section::unknown_c0}.

\subsection{Extension to Nonlinear Rewards and Costs}\label{section::nonlinear_rewards_costs}

Building on a rich line of research that extends bandit problems to function approximation—such as \cite{russo2013eluder}, \cite{zhou2020neural}, \cite{foster2020beyond},  and \cite{liu2023global}—we investigate how to adapt constrained bandit problems to this setting. In this work, we leverage the eluder dimension framework to develop our bounds and algorithms.

In this section, we explore constrained learning beyond linearity and extend the OPLB algorithm to the setting where the reward and cost functions are possibly nonlinear functions of the actions. We call the resulting algorithm Optimistic-Pessimistic Nonlinear Bandit algorithm (OPNLB). In each round $t\in [T]$, the agent is given an action set $\mathcal{A}_t \subseteq \mathcal{A}$, where $\mathcal{A}$ is a formal action set. Upon taking action $X_t \in \mathcal{A}_t$ the learner observes a reward-cost pair $(R_t, C_t)$ such that $R_t = \theta_*(X_t) + \xi_t^r$ and $C_t = \mu_*(X_t) + \xi_t^c$, where $\theta_*(\cdot) \in \mathcal{G}_r$ and $\mu_*(\cdot) \in \mathcal{G}_c$ are the mean reward and mean cost functions that belong to known function classes $\mathcal{G}_r$ and $\mathcal{G}_c$, and $\xi_t^r$ and $\xi_t^c$ are conditionally zero-mean sub-Gaussian random variables. %
The rewards and costs satisfy the following assumption.

\begin{assumption}[Bounded Responses]\label{ass:bounded_reward_costs_function_approx}
For all $t \in [T]$ and $x \in \mathcal{A}_t$, the mean rewards and costs are bounded, i.e.,~$\theta_*(x), \mu_*(x) \in [0,1]$. Moreover, the rewards and costs observed in all rounds of the algorithm are also bounded, i.e.,~$R_t, C_t \in [0,1],\;\forall t\in [T]$.
\end{assumption}

Moreover, the action sets $\mathcal{A}_t$ satisfy the safe action Assumption~\ref{ass:safe-action}, i.e.,~there is an action $x_0 \in \mathcal{A}_t$, $\forall t \in [T]$, with known average reward $r_0 = \theta_*(x_0)$ and known average cost $c_0=\mu_*(x_0) $, such that $c_0 < \tau$. The policy $\pi_t$ according to which an action is taken in round $t$ is an element of $\Delta_{\mathcal{A}_t}$. We denote by $\Pi_t^f=\{\pi\in\Delta_{\mathcal A_t} : \mathbb E_{X\sim\pi}[\mu_*(X)] \leq \tau\}$ and $\pi^*_t=\argmax_{\pi\in\Pi^f_t}\mathbb E_{X\sim\pi}[\theta_*(X)]$ the set of feasible polices and the optimal policy in round $t\in [T]$. Finally, we define the $T$-round regret as 
\begin{equation}
\label{eq:regret-non-linear}
\mathcal R_\Pi(T) = \sum_{t=1}^T \mathbb E_{X\sim\pi_t^*}\big[\theta_*(X)\big] - \mathbb E_{X\sim\pi_t}\big[\theta_*(X)\big]. 
\end{equation}

The nonlinear reward and cost model that we discuss in this section does not allow for a high probability constraint satisfaction scenario without making strong assumptions on the action set. The star-convexity requirement of Definition~\ref{definition::star_convex} does not extend to non-linear action spaces. This is the reason why we study an expected constraint scenario instead. Before introducing our algorithm for this setting, we formally define the eluder dimension, i.e.,~a notion of complexity relevant to adaptive selection procedures introduced in \citet{russo2013eluder}.

\begin{definition}[Action Independence and Eluder Dimension]
Let $\epsilon > 0$ and $\{x_i \}_{i=1}^{n}$ be a set of actions. Then, we have the following definitions: 

\begin{itemize}
    \item An action $x$ is $\epsilon$-dependent on $\{x_i \}_{i=1}^{n}$ w.r.t.~the function space $\mathcal{G}$, if any $f, f' \in \mathcal{G}$ satisfying $\sqrt{\sum_{i=1}^n (f(x_i)-f'(x_i))^2} \leq \epsilon$ also satisfy $|f(x)-f'(x)|\leq \epsilon$. An action $x$ is $\epsilon$-independent of $\{x_i \}_{i=1}^{n}$ w.r.t.~$\mathcal{G}$ if it is not $\epsilon$-dependent on $\{x_i \}_{i=1}^{n}$.
    \item The $\epsilon$-eluder dimension $d_{\mathrm{eluder}}(\mathcal{G}, \epsilon)$ is the length of the longest sequence of elements in $\{ x_i \}_{i=1}^{n}$ such that for some $\epsilon' \geq \epsilon$, every element is $\epsilon'$-independent of its predecessors.
\end{itemize}
\end{definition}

Throughout this section, we will use the notation $d_{\mathrm{eluder}}^r = d_{\mathrm{eluder}}(\mathcal{G}_r, 1/T)$ and $d_{\mathrm{eluder}}^c = d_{\mathrm{eluder}}(\mathcal{G}_c, 1/T)$ to denote the eluder dimensions of the function spaces $\mathcal{G}_r$ and $\mathcal{G}_c$, respectively.

\paragraph{Optimistic-Pessimistic Non-Linear Bandit (OPNLB).} In each round $t\in [T]$, we define two confidence sets
\begin{align*}
C_t^r(\delta) &= \big\{\theta \in \mathcal{G}_r : \|\theta - \widehat{\theta}_{t} \|_{\mathcal{D}_t} \leq  \gamma_r(t,\delta/2)\big\}, \\
C_t^c(\delta) &= \big\{\mu \in \mathcal{G}_c : \| \mu - \widehat{\mu}_{t} \|_{\mathcal{D}_t} \leq  \gamma_c(t,\delta/2) \;\text{ and }\; \mu(x_0) = c_0\big\},
\end{align*}
where $\mathcal D_t = \{ (X_s, R_s, C_s) \}_{s=1}^{t-1}$ is the dataset of actions, rewards, and costs observed up until the beginning of round $t$, $\| f \|_{\mathcal{D}_t} = \sqrt{ \sum_{x \in \mathcal{D}_t} f^2(x)}$ is the norm defined by the dataset $\mathcal{D}_t$ for any function $f:\mathcal A_t\rightarrow \mathbb R$, and $\gamma_r(t, \delta)$ and $\gamma_c(t, \delta)$ are the reward and cost confidence set radii defined as
\begin{align*}
\gamma_r(t, \delta) = 512 \log\left( \frac{24 |\mathcal{G}_r|\log(2t)}{\delta} \right), \qquad
\gamma_c(t, \delta) = 512 \log\left( \frac{24 |\mathcal{G}_c|\log(2t)}{\delta} \right).
\end{align*}
Let $\mathcal{E}$ be the event defined as
\begin{align}
\label{eq:high-prob-event-nonlinear}
\mathcal{E} := \left\{  \theta_*  \in C_t^r(\delta) \; \wedge \; \mu_* \in C_t^r(\delta) , \;\forall t\in[T] \right\}.
\end{align}
This is the same event as in~\eqref{eq:high-prob-event} for the linear case where the confidence sets $C_t^r(\delta)$ and $C_t^c(\delta)$ satisfy $\theta_* \in C_t^r(\delta)$ and $\mu_* \in C_t^c(\delta)$ for all $t \in [T]$. Corollary~\ref{corollary:confidence_sets_function_approx} in Appendix~\ref{section::LS_estimators_properties} implies that $\mathbb{P}( \mathcal{E}) \geq 1-\delta$.

\begin{algorithm}[t]
\caption{Optimistic-Pessimistic Nonlinear Bandit (OPNLB)}
\label{alg::nonlinear_optimism_pessimism}
\begin{algorithmic}[1]
\STATE {\bfseries Input:} Safe action $x_0$ with reward $r_0$ and cost $c_0$; $\;$ Constraint threshold $\tau\geq 0$; $\;$ Scaling parameters $\alpha_r, \alpha_c \geq 1$ 
\FOR{$t=1, \ldots , T$}
\STATE Compute $\widehat{\theta}_{t}, \widehat{\mu}_{t}$ using least squares
\STATE Observe action set $\mathcal{A}_t$ and construct the estimated feasible policy set $\widetilde{\Pi}_t^f$ using~\eqref{equation::feasible_policy_set_eluder}.
\STATE Compute policy $\pi_t = \argmax_{\pi \in \widetilde{\Pi}_t^f} \widetilde{V}_t^r(\pi)$.
\STATE Take action $X_t \sim \pi_t $ and observe reward and cost signals $(R_t,C_t)$
\ENDFOR
\end{algorithmic}
\end{algorithm}

We define pessimistic cost $\widetilde{V}_t^c(\pi)$ and optimistic reward $\widetilde{V}_t^r(\pi)$ values for a policy $\pi$ in each round $t$ as
\begin{align*}
\widetilde{V}_t^c(\pi) &= \max_{\mu \in C^c_t(\delta)} \mathbb E_{x\sim\pi}\big[\mu(x)\big], \\ 
\widetilde{V}_t^r(\pi) &= \max_{\theta \in C^r_t(\delta)} \mathbb E_{x\sim\pi}\big[\theta(x)\big] + \alpha_r \max_{\mu' , \mu''\in C^c_t(\delta)} \mathbb E_{x\sim\pi}\big[\mu'(x)\big] - \mathbb E_{x\sim\pi}\big[\mu''(x)\big].
\end{align*}
Note that $\widetilde V_t^r$ is the combination of an optimistic reward estimator plus an artificially inflated confidence interval that depends on the cost function class. We define the set of feasible policies in round $t$ as
\begin{equation}\label{equation::feasible_policy_set_eluder}
\widetilde{\Pi}_t^f = \big\{ \pi \in \Delta_{\mathcal A_t} : \widetilde{V}_t^c(\pi) \leq \tau \big\}.
\end{equation}
We now show that by appropriately setting the scaling parameters $\alpha_r$ and $\alpha_c$, the policy $\pi_t$ selected by Algorithm~\ref{alg::nonlinear_optimism_pessimism} is feasible and a basic optimistic relationship holds between $\pi_t$ and $\pi_t^*$, i.e.,~$\widetilde{V}_t^r(\pi_t) \geq V_t^r( \pi_t^*)$. The following lemma is the equivalent of Lemma~\ref{lemma:linear_bandits_optimism} from the linear case. 

\begin{restatable}{lemma}{lemmaoptimismeluder}\label{lemma::optimism_eluder}

If the event $\mathcal{E}$ defined by~\eqref{eq:high-prob-event-nonlinear} holds and the scaling parameter satisfies $\alpha_r = \frac{1-r_0}{\tau- c_0}$, then for all $t \in [T]$, we have $\widetilde{V}_t^r(\pi_t) \geq V_t^r( \pi_t^*) = \theta_*(\pi_t^*, \mathcal{A}_t)$.
\end{restatable}

\begin{proof}
See Appendix~\ref{section::proof_lemma_optimism_eluder}. 
\end{proof}

The proof of Lemma~\ref{lemma::optimism_eluder} follows a similar logic as that of Lemmas~\ref{LEMMA::LINEAR_BANDITS_OPTIMISM} and~\ref{lemma:linear_bandits_optimism}. Given this result, we now prove a regret bound for OPNLB in terms of the eluder dimensions $d_{\mathrm{eluder}}^r$ and $d_{\mathrm{eluder}}^c$. 
When $\mathcal{E}$ holds for all $t \in [T]$, the following inequalities are satisfied
\begin{align*}
    &\mathcal{R}_\Pi(T) = \sum_{t=1}^T \mathbb E_{X\sim\pi_t^*}\big[\theta_*(X)\big] - \mathbb E_{X\sim\pi_t}\big[\theta_*(X)\big] \\ 
    &\;\;\stackrel{\text{(a)}}{\leq} \sum_{t=1}^T \widetilde{V}^r_t(\pi_t) - \mathbb E_{X\sim\pi_t}\big[\theta_*(X)\big] \\
    &\;\;= \sum_{t=1}^T \max_{\theta \in C^r_t(\delta)} \mathbb E_{X\sim\pi_t}\big[\theta(X)\big] - \mathbb E_{X\sim\pi_t}\big[\theta_*(X)\big] + \alpha_r \!\!\max_{\mu , \mu'\in C^c_t(\delta)} \mathbb E_{X\sim\pi_t}\big[\mu(X)\big] - \mathbb E_{X\sim\pi_t}\big[\mu'(X)\big] \\
    &\;\;\stackrel{\text{(b)}}{\leq} \sum_{t=1}^T \max_{\theta, \theta' \in C^r_t(\delta)} \!\mathbb E_{X\sim\pi_t}\big[\theta(X)\big] - \mathbb E_{X\sim\pi_t}\big[\theta'(X)\big] + \alpha_r \!\!\!\max_{\mu , \mu'\in C^c_t(\delta)} \!\mathbb E_{X\sim\pi_t}\big[\mu(X)\big] - \mathbb E_{X\sim\pi_t}\big[\mu'(X)\big].
\end{align*}
{\bf (a)} holds because of Lemma~\ref{lemma::optimism_eluder}. {\bf (b)} holds because  $\theta_* \in C_t^r(\delta)$ for all $t \in [T]$ with probability at least $1-\delta$. %
The inequality sequence above implies that the regret can be upper-bounded by a weighted sum of uncertainty widths. We bound the sum of the uncertainty widths using Lemma~3 in~\cite{chan2021parallelizing} (by setting the parallelism parameter $P=1$) as
\begin{equation*}
\sum_{t=1}^T \max_{\theta, \theta' \in C^r_t(\delta)} \mathbb E_{X\sim\pi_t}\big[\theta(X)\big] - \mathbb E_{X\sim\pi_t}\big[\theta'(X)\big] = \mathcal{O}\left( d^r_{\mathrm{eluder} } + \sqrt{Td^r_{\mathrm{eluder} } \gamma_r(T, \delta/2)} \right) 
\end{equation*}
and
\begin{equation*}
\sum_{t=1}^T \max_{\mu, \mu' \in C^c_t(\delta)} \mathbb E_{X\sim\pi_t}\big[\mu(X)\big] - \mathbb E_{X\sim\pi_t}\big[\mu'(X)\big] = \mathcal{O}\left(d^c_{\mathrm{eluder} }+ \sqrt{Td^c_{\mathrm{eluder} } \gamma_c(T, \delta/2)} \right).
\end{equation*}
Combining these results and using $\mathbb{P}( \mathcal{E}) \geq 1-\delta$, we obtain the main result of this section, which is a regret bound for the OPNLB algorithm (Algorithm~\ref{alg::nonlinear_optimism_pessimism}).

\begin{theorem}[OPNLB regret-bound]
\label{theorem::main_nonlinear_result}
With probability at least $1-\delta$, the regret of Algorithm~\ref{alg::nonlinear_optimism_pessimism} satisfies
\begin{equation*}
     \mathcal{R}_\Pi(T) =\mathcal{O}\left( \sqrt{Td^r_{\mathrm{eluder} } \gamma_r(T, \delta/2)} + \frac{1-r_0}{\tau - c_0} \sqrt{Td^c_{\mathrm{eluder} } \gamma_c(T, \delta/2)} + d^r_{\mathrm{eluder} } + \frac{1-r_0}{\tau - c_0}  d^c_{\mathrm{eluder} }\right). 
\end{equation*}

\end{theorem}

\begin{remark}
\label{remark:extension-high-prob-nonlinear}
It is not possible to extend the results of this section (the non-linear case) to the high probability setting studied in Section~\ref{section::high_probability_constraints}. In the linear high probability scenario, star-convexity around the safe action $x_0$ allows the learner to form a model of $\mu_\star(x)$ by playing a convex combination of the safe action and any other action $x$. In the non-linear setting, since $\mathcal{A}_t \subset \mathcal{A}$ is a formal action set, closure of $\mathcal{A}_t$ under convexity is not defined. Thus, it is possible to have actions $x$ that are safe, i.e.,~$\mu_\star(x) < \tau$, but can never b explored safely. 
\end{remark}

\paragraph{Computational Tractability of ONPLB.} Step ~$5$ of Algorithm~\ref{alg::nonlinear_optimism_pessimism}  involves solving a constrained optimization problem that, in general, can be intractable. It remains an open question how to design tractable algorithms for constrained non-linear bandit problems.   

\subsection{Experimental Results}
\label{sec:experiments}

\begin{figure*}
\centering\subfigure{\includegraphics[width=0.25\linewidth]{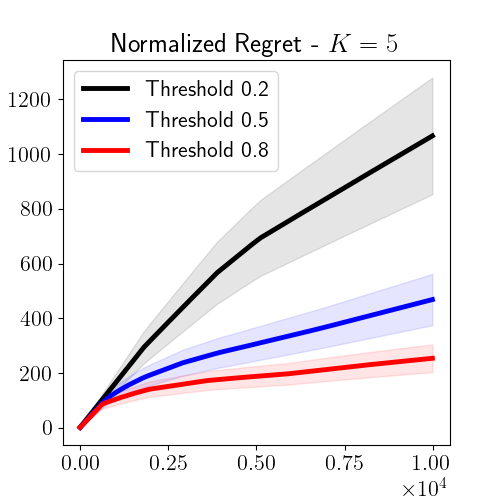}} 
\centering\subfigure{\includegraphics[width=0.32\linewidth]{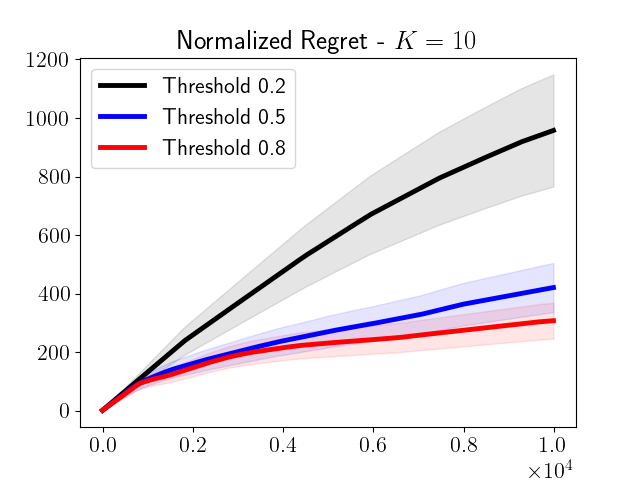}} 
\centering\subfigure{\includegraphics[width=0.32\linewidth]{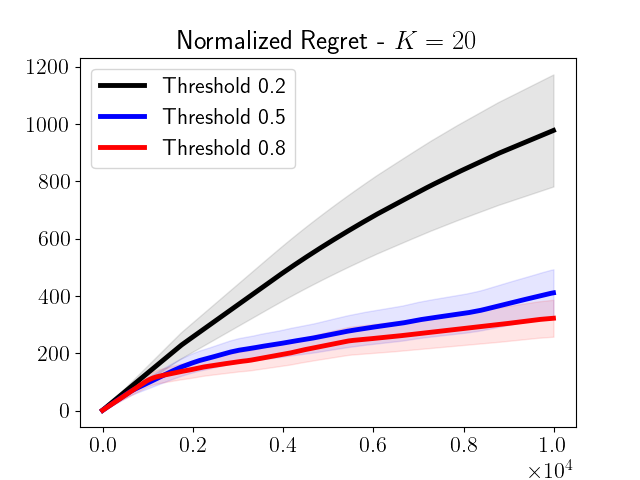}}        
\caption{Regret of OPB for three instances of the randomly generated constrained multi-armed bandit problem with the number of arms equal to $5$ {\em (a)}, $10$ {\em (b)}, and $20$ {\em (c)}. The cost and reward of the safe arm are set to $\bar{c}_1 = \bar{r}_1 = 0$.} 
\label{fig:constrained_bandits5}
\vspace{0.1in}
        
\centering\subfigure{\includegraphics[width=0.32\linewidth]{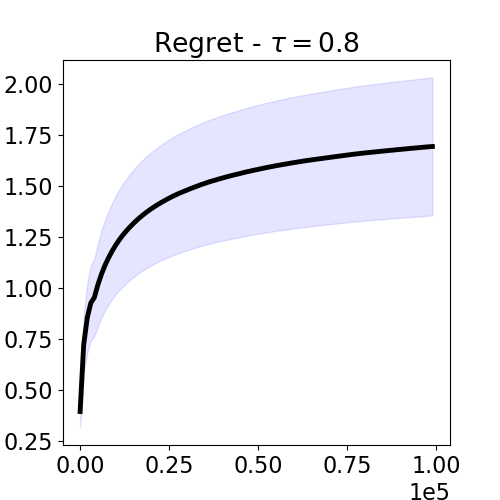}} 
\centering\subfigure{\includegraphics[width=0.32\linewidth]{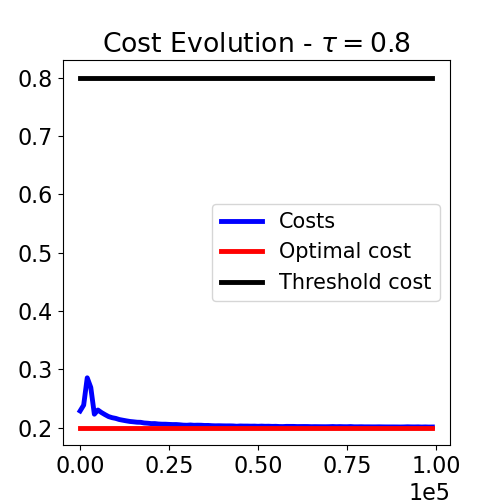}} 
\centering\subfigure{\includegraphics[width=0.32\linewidth]{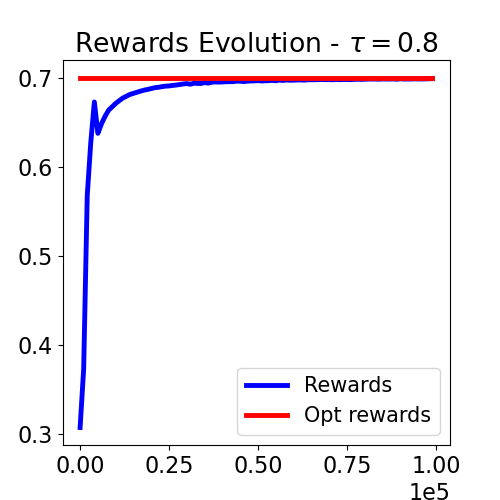}} 
\vspace{0.1in}
       
\centering\subfigure{\includegraphics[width=0.32\linewidth]{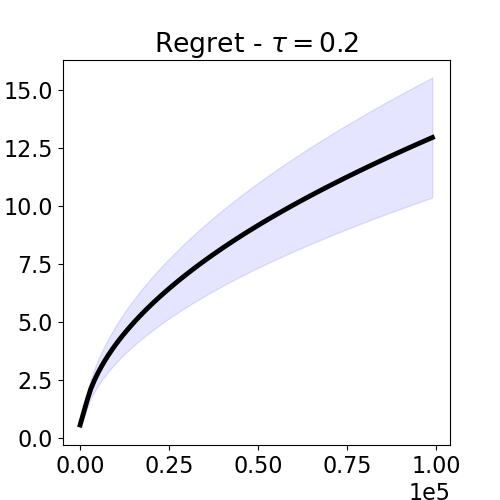}} 
\centering\subfigure{\includegraphics[width=0.32\linewidth]{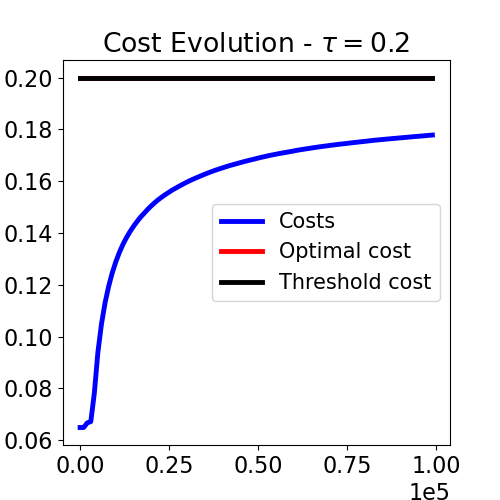}} 
\centering\subfigure{\includegraphics[width=0.32\linewidth]{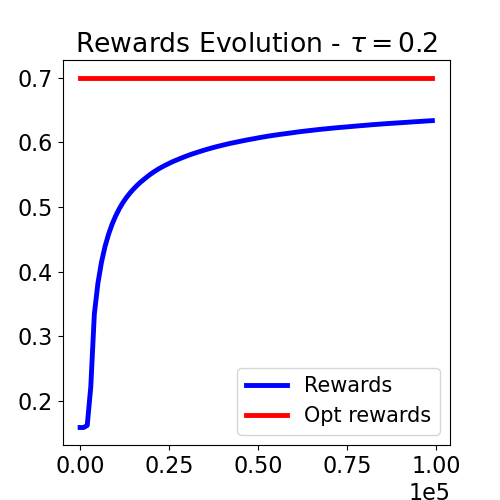}} 
\caption{{Regret {\em (left)}, cost {\em (middle)}, and reward {\em (right)} evolution of OPB in a $4$-armed bandit problem with Bernoulli reward and cost distributions with means $\bar{r} = (0.1, 0.2, 0.4, 0.7)$ and $\bar{c}=(0, 0.4, 0.5, 0.2)$. The cost of the safe arm (Arm~1) is $\bar{c}_1 = 0$. The constraint threshold is set to $\;\tau = 0.8$ {\em (top)} and $\;\tau = 0.2$ {\em (bottom)}.}}
 \label{fig:constrained_bandits4}
\begin{center}
\begin{minipage}{0.9\textwidth}
\end{minipage}
\end{center}
\end{figure*}

We run a set of experiments to show the behavior of the OPB algorithm and validate our theoretical results. In our first experiment, presented in Figure~\ref{fig:constrained_bandits5}, we produce random instances of our constrained multi-armed bandit problem. In all the instances, we set the safe arm to have reward and cost $0$. We generate different problem instances by sampling the Bernoulli mean rewards and costs of the rest of the arms uniformly at random from the interval $[0,1]$. Each sample run in this experiment corresponds to a sample problem instance. In Figure~\ref{fig:constrained_bandits5}, we report the regret of OPB for each of the number of arms $K$ equal to $5$ {\em (left)}, $10$ {\em (middle)}, and $20$ {\em (right)}, and for three constraint threshold $\tau$ values, $0.8$ {\em (red)}, $0.5$ {\em (blue)}, and $0.2$ {\em (black)}. For each parameter setting we sample $10$ random problem instances and report the average regret curves with a shaded region corresponding to the $\pm 0.5$ standard deviation around the regret. Figure~\ref{fig:constrained_bandits5} also shows that the regret of OPB grows inversely with the safety gap.  

In the next experiment, presented in Figure~\ref{fig:constrained_bandits4}, we consider a $K=4$-armed bandit problem in which the reward and cost distributions of the arms are Bernoulli with means $\bar r=(0.1,0.2,0.4,0.7)$ and $\bar c =(0,0.4,0.5,0.2)$. Arm~1 is the safe arm with the expected cost $\bar c_1=0$. We gradually reduce the constraint threshold $\tau$, and as a result, the {\em safety gap} $\tau - \bar{c}_1$, and show the regret {\em (left)}, cost {\em (middle)}, and reward {\em (right)} evolution of OPB. The cost and reward of OPB are in blue and the optimal cost and reward are in red. All results are averaged over $10$ runs and the shade is the $\pm 0.5$ standard deviation around the regret. 

Figure \ref{fig:constrained_bandits4} shows that the regret of OPB grows as we reduce $\tau$, and as a result the safety gap {\em (left)}. This is in support of our theories that identified the safety gap as the complexity of this constrained bandit problem. The results also indicate that the algorithm is successful in satisfying the constraint {\em (middle)} and in reaching the optimal reward/performance {\em (right)}. In the bottom three plots of Figure~\ref{fig:constrained_bandits4}, %
the cost of the best arm (Arm~4) is equal to the constraint threshold $\tau=0.2$. Thus, the cost of the optimal policy {\em (red)} and the constraint threshold {\em (black)} overlap in the cost evolution {\em (middle)} sub-figure. In Figure~\ref{fig:constrained_bandits3} in Appendix~\ref{section::expectation_constraints}, we report more experiments with the same $4$-armed bandit problem instance with constraint threshold values $\tau = 0.5$ and $0.6$. Using these intermediate threshold values we provide further support to our results showing the the safety gap governs the complexity in this constrained bandit problem.

\section{Conclusions}

In this work, we expand the frontier of the study of constrained bandit problems with anytime cost constraints. We extend the results of~\cite{pacchiano2020stochastic} in a variety of ways. First, we introduce the high probability constraint satisfaction regime for linear bandit problems with stage-wise constraints along with the LC-LUCB algorithm (Section~\ref{section::high_probability_constraints}). This formulation captures problems where an in-expectation constraint is not sufficient to ensure safety. We show that in contrast with OPLB, when the action set is finite and star-convex, the LC-LUCB algorithm is computationally tractable (Section~\ref{subsubsection:LC-LUCB-Comp-Complexity}). This stands in marked contrasts with the case of OPLB, that only has a tractable form in the multi-armed bandit setting (see the OPB algorithm in Section~\ref{sec:constrained-MAB}). Second, we improve the regret-bound of OPLB reported in~\cite{pacchiano2020stochastic} to better identify the quantity representing the hardness of the constrained problem $\frac{1-r_0}{\tau-c_0}$. 

Finally, we go beyond the scenario of linear rewards and cost functions and explore the nonlinear regime where the reward and cost functions come from arbitrary function classes of bounded eluder dimension (Section~\ref{section::nonlinear_rewards_costs}). When the reward and cost function classes are arbitrary and the requirement is to satisfy an anytime expected cost constraint, we introduce the OPNLB algorithm and prove it satisfies a regret-bound equivalent to the regret-bound for OPLB where the eluder dimension of the reward and cost function classes plays the role of the linear dimension in the linear case. Since the eluder dimension of linear classes equals the dimension of the ambient space, these results subsume the regret-bounds for OPLB in~\cite{pacchiano2020stochastic}.  

The design of all of our algorithms (LC-LUCB, OPLB, OPB and OPNLB) relies on the principle of optimism-pessimism and the technique of asymmetric confidence intervals that enables the provable analysis of optimistic-pessimistic algorithms. We hope the results of this work can serve as inspiration to extend the study of stage-wise constrained problems to richer scenarios such as reinforcement learning and beyond.

\acks{A.P.\ would like to thank the support of the Eric and Wendy Schmidt Center at the Broad Institute of MIT and Harvard. This work was supported in part by funding from the Eric and Wendy Schmidt Center at the Broad Institute of MIT and Harvard.}

\newpage

\appendix

\tableofcontents
\addtocontents{toc}{\protect\setcounter{tocdepth}{2}}
\clearpage

\section{Proofs of Section~\ref{section::high_probability_constraints}} 
\label{appendix:section-high-prob}

\begin{proof}
{\bf of Proposition~\ref{PROP:OPTIMISTIC-REWARD-PESSIMISTIC-COST}: } We only prove the statement~\eqref{equation::v_t_r} for the optimistic reward $\widetilde{V}_t^r(x)$. The proof of statement~\eqref{equation::v_t_c} for the pessimistic cost $\widetilde{V}_t^c(x)$ is analogous. From the definition of the confidence set $\mathcal{C}_t^r(\alpha_r)$, any vector $\theta \in \mathcal{C}_t^r(\alpha_r)$ can be written as $\widehat{\theta}_t + v$, where $v$ satisfying $\| v\|_{\Sigma_t} \leq \alpha_r \beta_t(\delta, d)$. Thus, we may write 
\begin{align}
\label{equation:supporting_optimistic_reward_pessimistic_cost}
\widetilde{V}_t^r(x) &=  \max_{\beta \in \mathcal{C}_t^r(\alpha_r)} \langle x, \theta \rangle = \langle x, \widehat{\theta}_t \rangle + \max_{v:\|v\|_{\Sigma_t} \leq \alpha_r \beta_t(\delta, d)} \langle x, v \rangle \notag \\ 
&\stackrel{\text{(a)}}{\leq} \langle x,\widehat{\theta}_t \rangle + \alpha_r \beta_t(\delta, d) \|  x \|_{\Sigma_t^{-1}}.
\end{align}
{\bf (a)} By Cauchy-Schwartz, for all $v$, we have $\langle x, v \rangle \leq \| x\|_{\Sigma_t^{-1}} \| v\|_{\Sigma_t}$. The result follows from the condition on $v$ in the maximum, i.e.,~$\| v \|_{\Sigma_t} \leq \alpha_r \beta_t(\delta,d)$.

Let us define $v^* := \frac{ \alpha_r \beta_t(\delta, d) \Sigma^{-1}_t x}{\| x \|_{\Sigma^{-1}_t}}$. This value of $v^*$ is feasible because 
\begin{equation*}
\|v^*\|_{\Sigma_t} = \frac{\alpha_r \beta_t(\delta,d)}{\| x \|_{\Sigma^{-1}_t}} \sqrt{x^\top \Sigma_t^{-1} \Sigma_t \Sigma_t^{-1} x} = \frac{\alpha_r\beta_t(\delta,d)}{\|x\|_{\Sigma^{-1}_t}} \sqrt{x^\top\Sigma_t^{-1}x} = \alpha_r\beta_t(\delta, d).
\end{equation*}
We now show that $v^*$ also achieves the upper-bound in the above inequality resulted from Cauchy-Schwartz 
\begin{equation*}
\langle x, v^* \rangle =  \frac{\alpha_r \beta_t(\delta, d)  x^\top \Sigma_t^{-1} x}{\| x \|_{\Sigma_t^{-1}}}   = \alpha_r \beta_t(\delta, d) \|  x \|_{\Sigma_t^{-1}}.
\end{equation*}
Thus, $v^*$ is the maximizer and we can write 
\begin{align*}
\widetilde{V}_t^r(x) = \langle x,\widehat{\theta}_t \rangle + \langle x,v^* \rangle = \langle x,\widehat{\theta}_t \rangle + \alpha_r \beta_t(\delta, d) \|  x \|_{\Sigma_t^{-1}},
\end{align*}
which concludes the proof. 
\end{proof}

\begin{proof}
{\bf of Lemma~\ref{LEMMA::LINEAR_BANDITS_OPTIMISM}: } In order to prove the desired result, it is enough to show that
\begin{equation*}
(x^{o, \perp})^\top (\Sigma_t^{o,\perp})^{\dagger} x^{o,\perp} \leq x^\top \Sigma_t^{-1}x.
\end{equation*}
Without loss of generality, we can assume $x_o = e_1$, where $e_1$ is the first basis vector. Note that in this case $\Sigma_t^{o, \perp}$ can be thought of as a sub-matrix of $\Sigma_t$ such that $\Sigma_t[2:, 2:] = \Sigma_t^{o, \perp}$, where $\Sigma_t[2:, 2:]$ denotes the sub-matrix with row and column indices from $2$ onward. Using the following formula for the inverse of a positive semi-definite (PSD) symmetric matrix
\begin{equation*}
\begin{bmatrix}
Z & \delta \\
\delta^\top & A 
\end{bmatrix} = \begin{bmatrix}
\frac{1}{D} &  -\frac{ A^{-1}\delta}{D} \\
- \frac{\delta^\top A^{-1} }{D}&A^{-1} + \frac{ A^{1} \delta \delta^\top A^{-1}}{D} 
\end{bmatrix},
\end{equation*}
where $D = z- \delta^\top A^{-1} \delta$, we may write $\Sigma_t^{-1}$ as 
\begin{equation*}
\Sigma_t^{-1} = \begin{bmatrix}
1/D & -\frac{(\Sigma_t^{o, \perp})^{\dagger} \Sigma_t[2, :d] }{D} \\
-\frac{\Sigma_t^\top [2:d](\Sigma_t^{o,\perp})^{\dagger} }{D} & (\Sigma_t^{o, \perp})^{\dagger} + \frac{(\Sigma_t^{o,\perp})^{\dagger} \Sigma_t[2:d]\Sigma_t[2:d](\Sigma_t^{o,\perp})^{\dagger} }{D}
\end{bmatrix},
\end{equation*}
where $D = \Sigma_t[1,1] - \Sigma_t[2:d]^\top(\Sigma_t^{o,\perp})^{\dagger} \Sigma_t[2:d] \in \mathbb{R}$. %
This allows us to write
\begin{align*}
x^\top(\Sigma_t^{-1}) x &= \frac{x(1)^2 - 2x(1)\Sigma_t[2:d]^\top(   \Sigma_t^{o,\perp})^{\dagger} x[2:d]}{D} \\ 
&+ \frac{x[2:d]^\top(\Sigma_t^{o,\perp})^{\dagger} \Sigma_t[2:d]\Sigma_t[2:d]^\top(\Sigma_t^{o,\perp})^{\dagger} x[2:d]}{D} + x[2:d]^\top(\Sigma_t^{o,\perp})^{\dagger} x[2:d] \\
&\geq x[2:d]^\top (\Sigma_t^{o,\perp})^{\dagger} x[2:d].
\end{align*}
The result follows by noting that $x[2:d] = x^{o, \perp}$.
\end{proof}

\section{Additional Experiments for Section~\ref{section::high_probability_constraints}} 
\label{appendix:section-high-prob-exp}

In this section, we present a comprehensive set of results extending the experiments presented in Figure~\ref{fig::figure_linear_constrained_1}. We consider a linear bandit problem in which the safe action equals the zero vector $x_0 = 0$ and the arm sets $\mathcal{A}_t$ are $d$ dimensional star convex sets generated by the $d$ cyclic shifted versions of the vector $v/\| v\|$ where $v = (0, 1, \cdots, d-1)$. Just like in Figure~\ref{fig::figure_linear_constrained_1}, the action set $\mathcal{A}_t$ is the star convex set defined by this set of actions and the lines emanating from the zero vector.  We let We let $\theta_* = v/\|v\|$ and $\mu_\star = (d-1, d-2, \cdots, 0)/\|v\|$, where $(d-1, d-2, \cdots, 0)$ is the flipped version of $(0, 1, \cdots, d-1)$.

In Figures~\ref{fig:constrained_bandits_TS3},~\ref{fig:constrained_bandits_TS4}, and~\ref{fig:constrained_bandits_TS5}, we plot the regret and cost evolution of LC-LUCB for dimensions $d=3, 5, 10$, and threshold values $\tau=0.2, 0.5, 0.8$, and compare them with those for the Safe-LTS algorithm of~\citet{Moradipari19SL}. The results for dimensions $d=3,5$ and $10$ are presented in Figures~\ref{fig:constrained_bandits_TS3},~\ref{fig:constrained_bandits_TS4}, and~\ref{fig:constrained_bandits_TS5} respectively. We show that as the threshold $\tau$ is driven to $0$, the problem gets progressively harder. The results show that LC-LUCB has a better regret profile than Safe-LTS, while satisfying the constraint, for all threshold values and dimensions.

\begin{figure}[ht]%
\begin{center}
\begin{minipage}{1\textwidth}
\centering
Dimension $d=3$.
\end{minipage}
\end{center}
\centering
\vspace{-2mm}
\begin{minipage}{0.875\textwidth}%
\centering\subfigure{\includegraphics[width=0.325\linewidth]{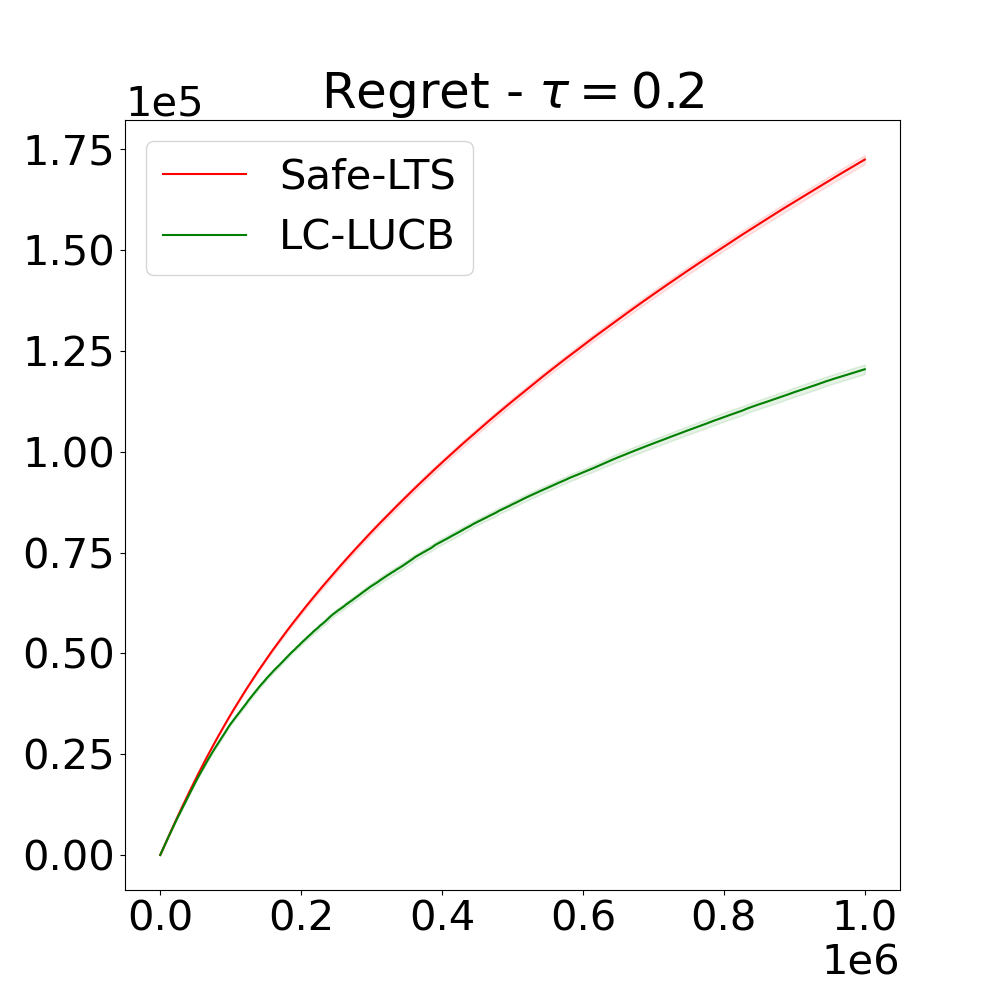}} 
\centering\subfigure{\includegraphics[width=0.325\linewidth]{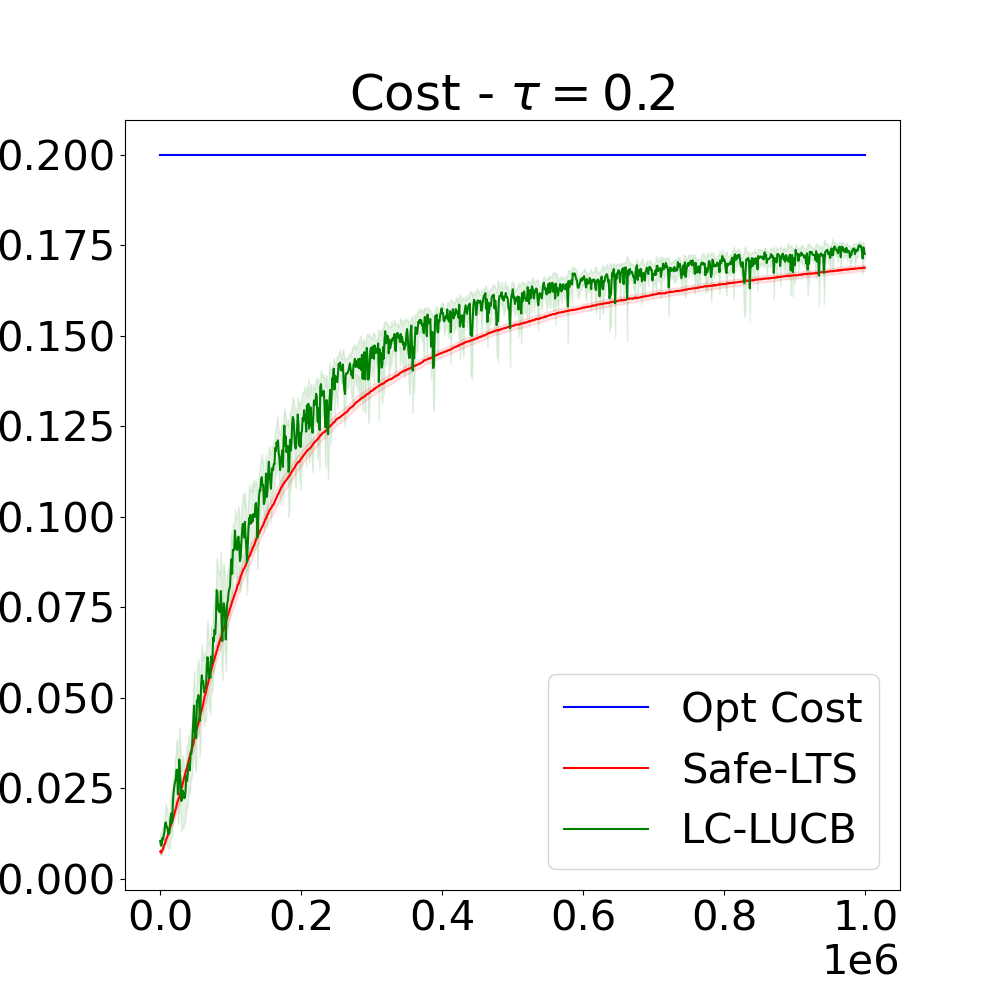}} 
\centering\subfigure{\includegraphics[width=0.325\linewidth]{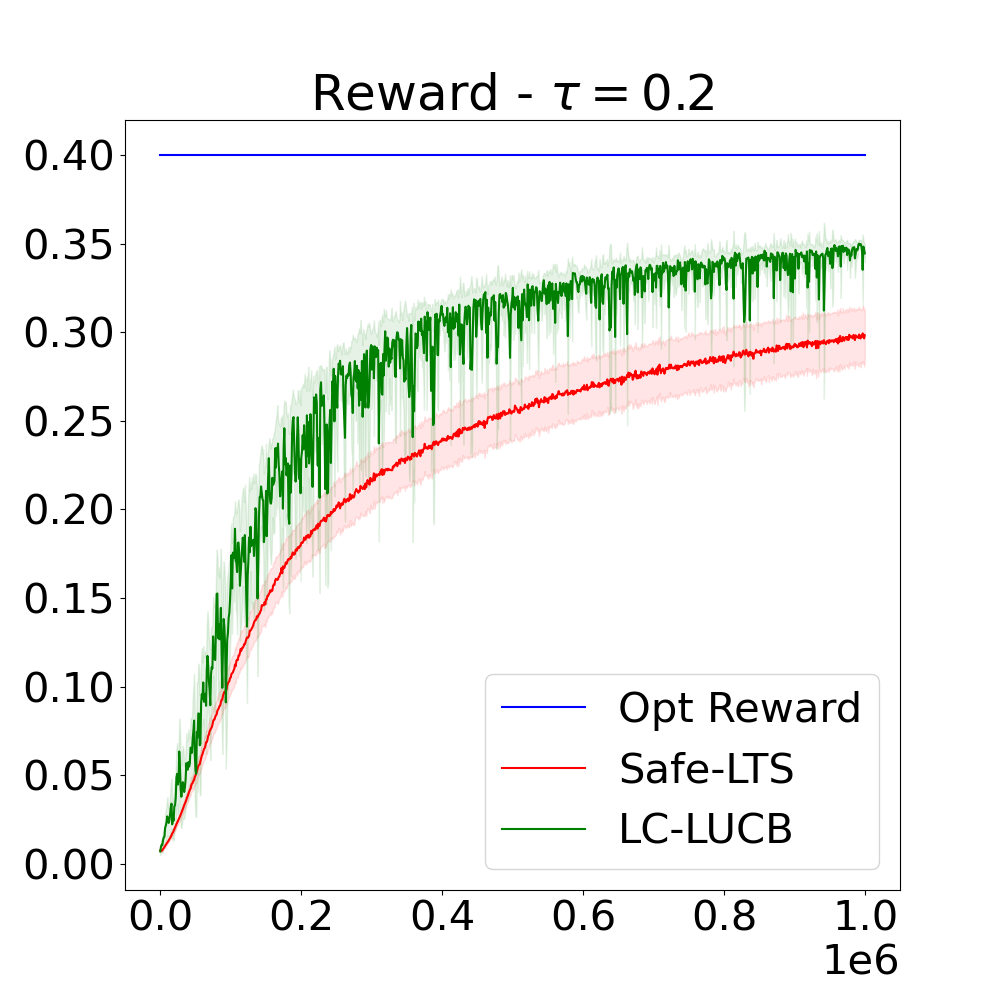}} 
\end{minipage}
\vspace{-2mm}
\vspace{2mm}
\begin{minipage}{0.875\textwidth}%
\centering\subfigure{\includegraphics[width=0.325\linewidth]{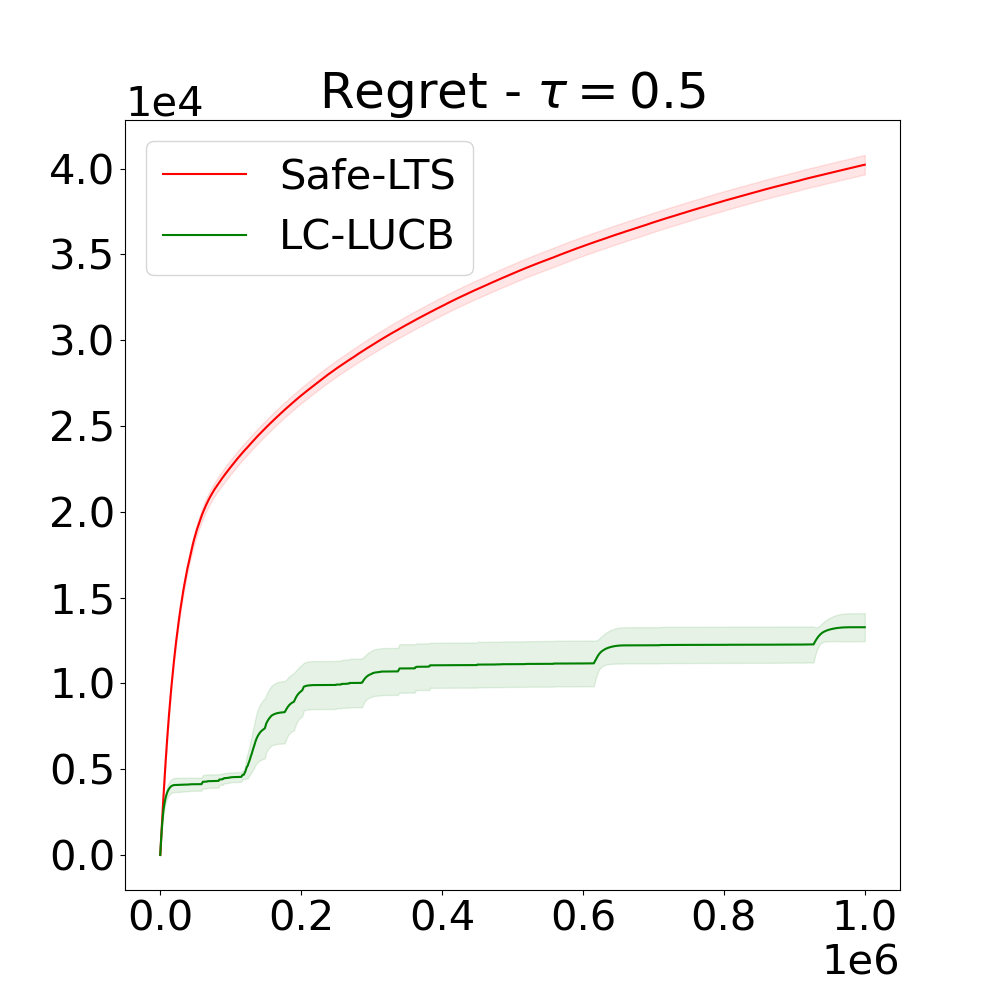}} 
\centering\subfigure{\includegraphics[width=0.325\linewidth]{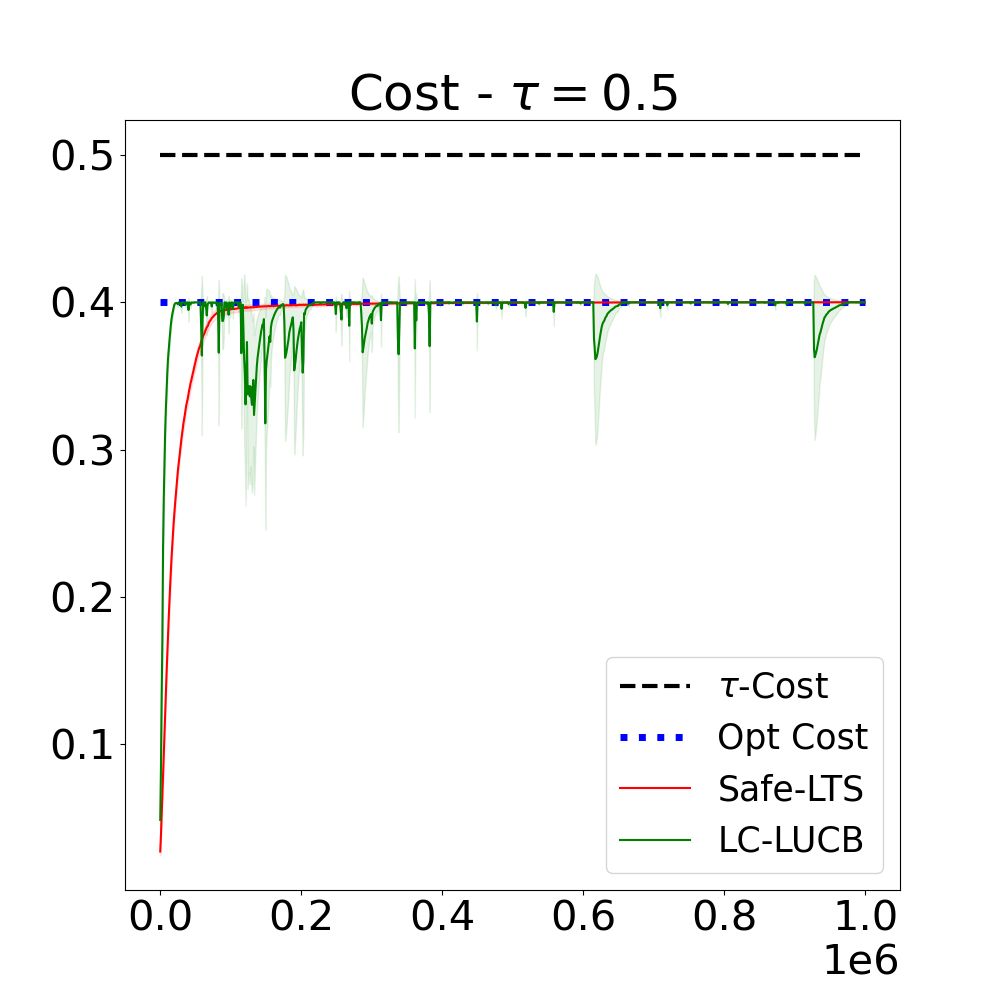}} 
\centering\subfigure{\includegraphics[width=0.325\linewidth]{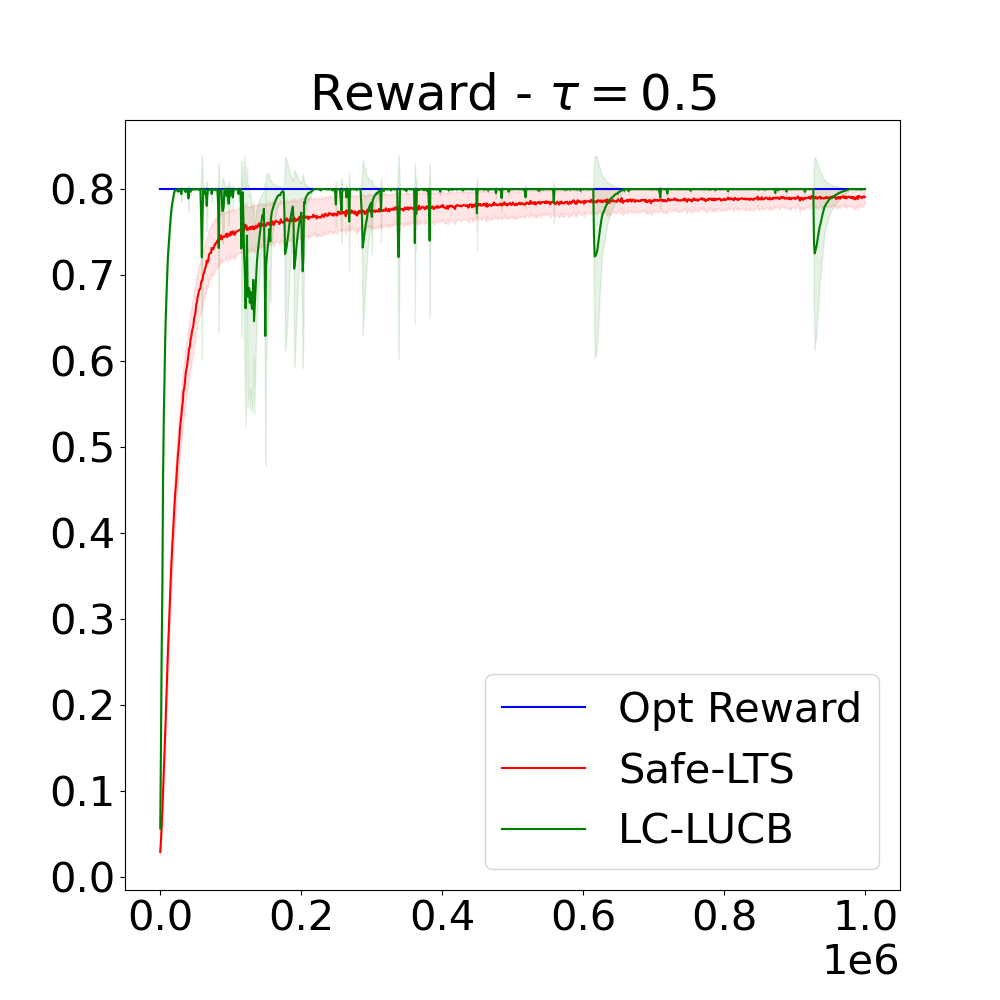}} 
\end{minipage}
\vspace{-2mm}
\vspace{2mm}
\begin{minipage}{0.875\textwidth}%
\centering\subfigure{\includegraphics[width=0.325\linewidth]{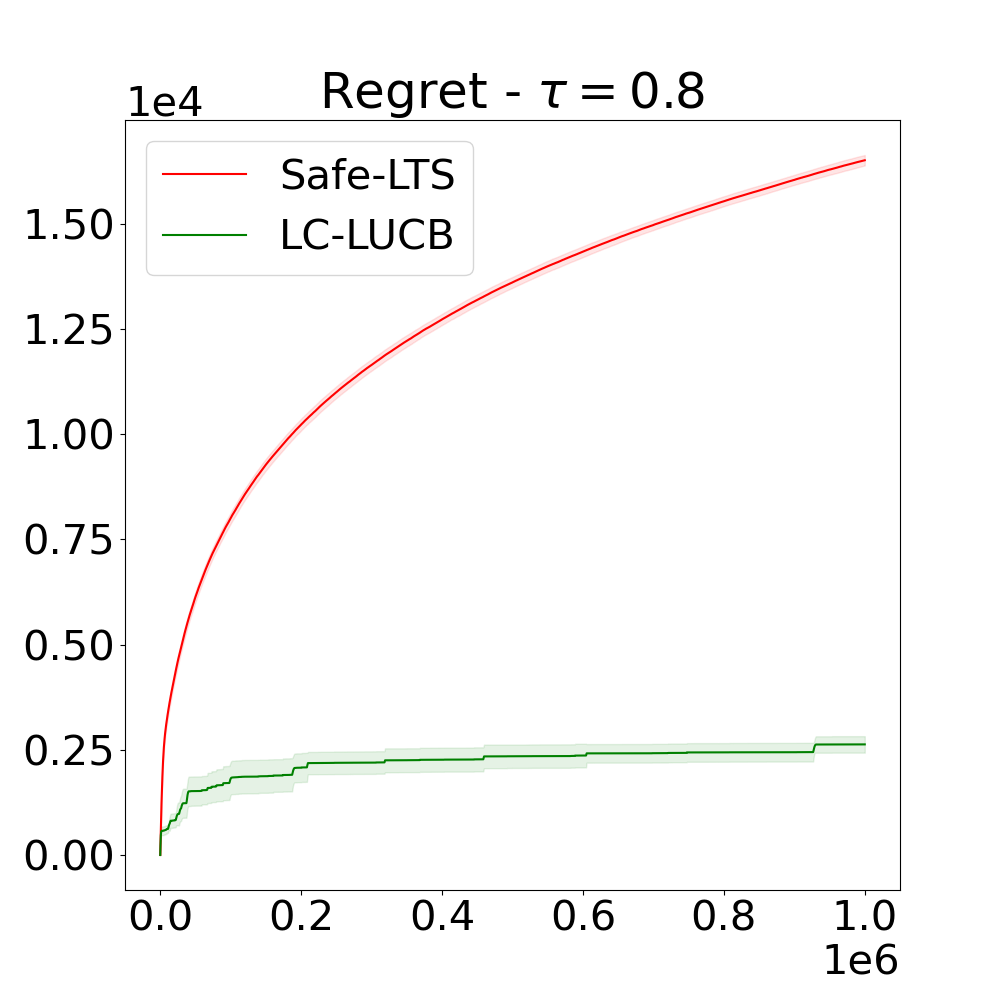}} 
\centering\subfigure{\includegraphics[width=0.325\linewidth]{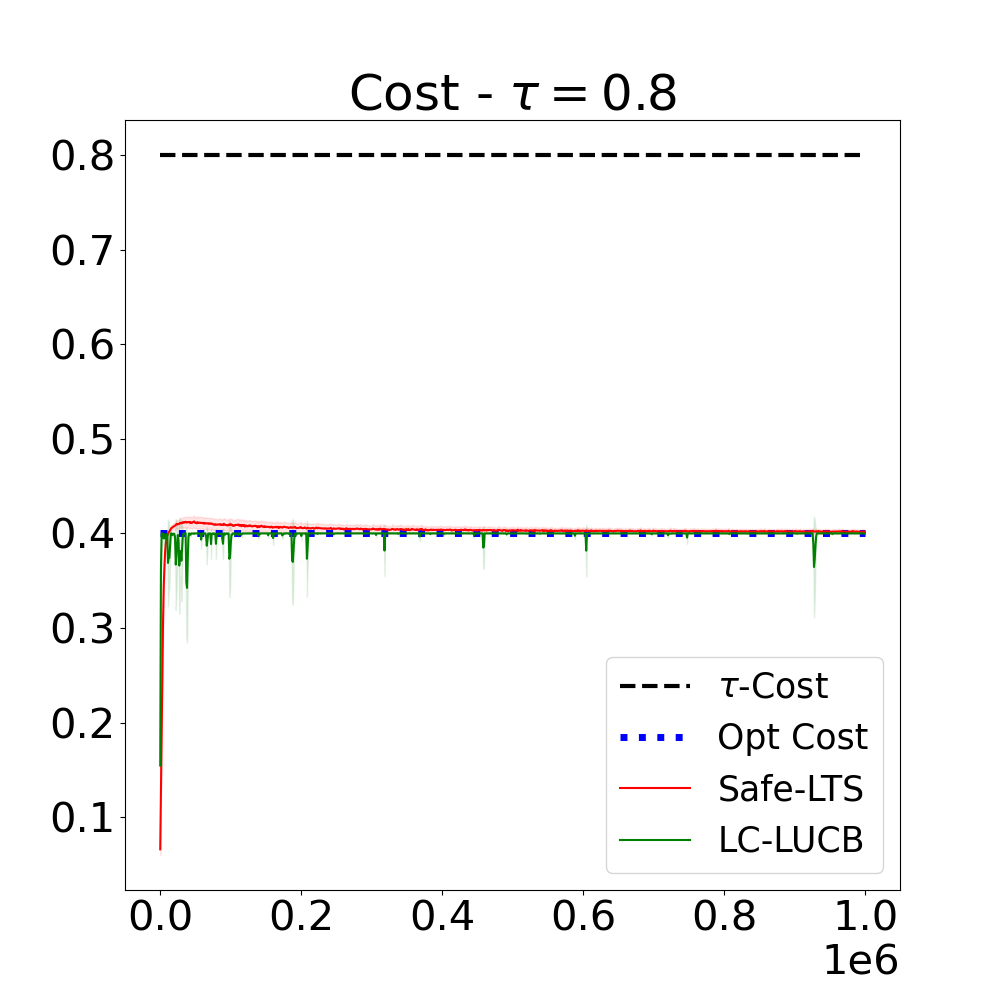}} 
\centering\subfigure{\includegraphics[width=0.325\linewidth]{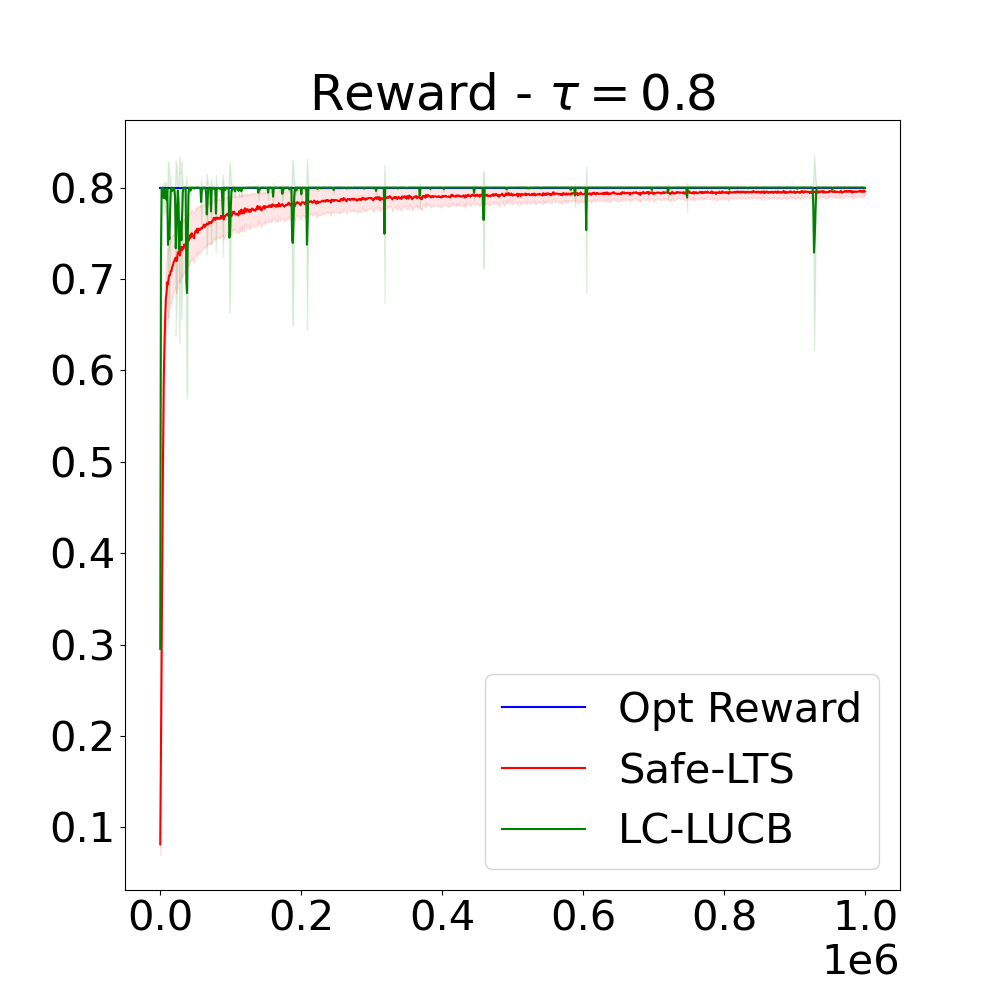}} 
\end{minipage}
\vspace{-2mm}
\caption{\textbf{LC-LUCB:} {Dimension $d = 3$. \textbf{Top:} Constraint Threshold $\tau=0.2$. \textbf{Center:} Constraint Threshold $\tau = 0.5$. \textbf{Bottom:} Constraint Threshold $\tau = 0.8$.  The shaded regions around the curves correspond to one standard deviation.}}
\vspace{2mm}
\label{fig:constrained_bandits_TS3}
\end{figure}

\begin{figure}[ht]%
\begin{center}
\begin{minipage}{1\textwidth}
\centering
Dimension $d=5$.
\end{minipage}
\end{center}
\centering
\vspace{-2mm}
\begin{minipage}{0.875\textwidth}%
\centering\subfigure{\includegraphics[width=0.325\linewidth]{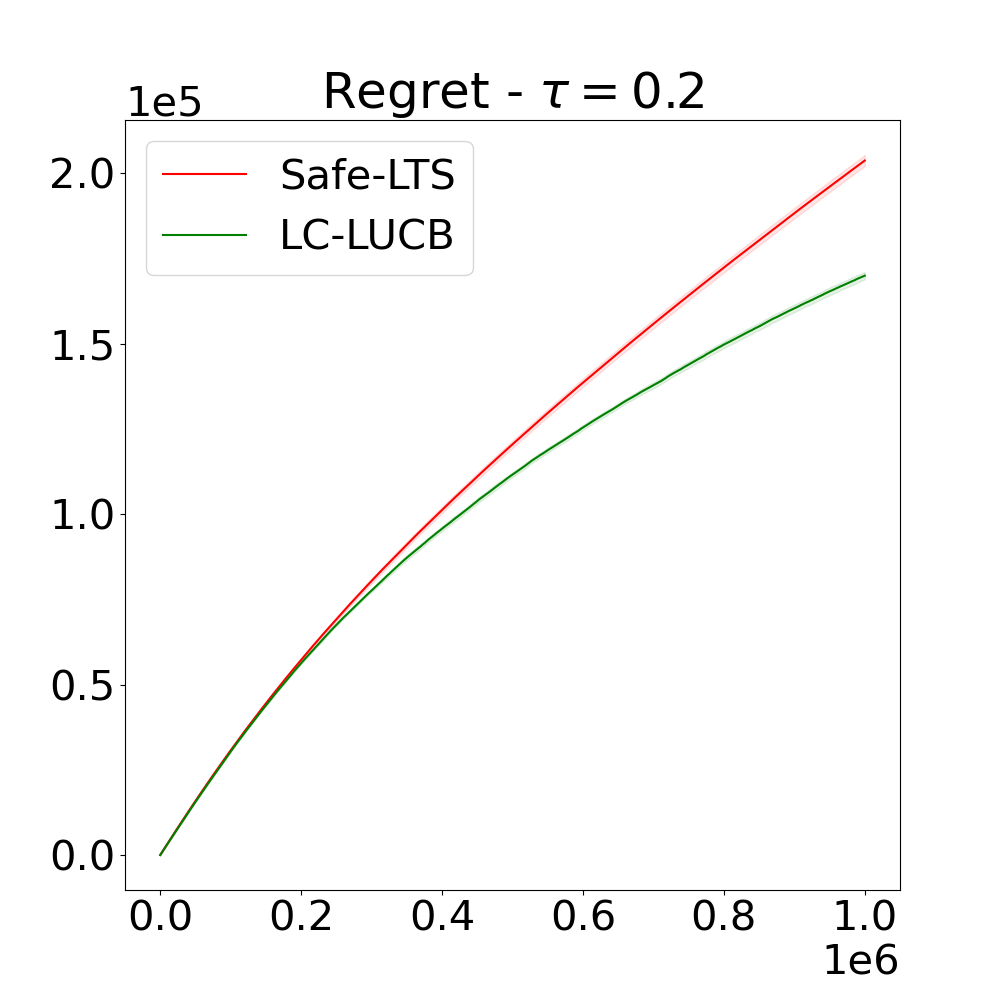}} 
\centering\subfigure{\includegraphics[width=0.325\linewidth]{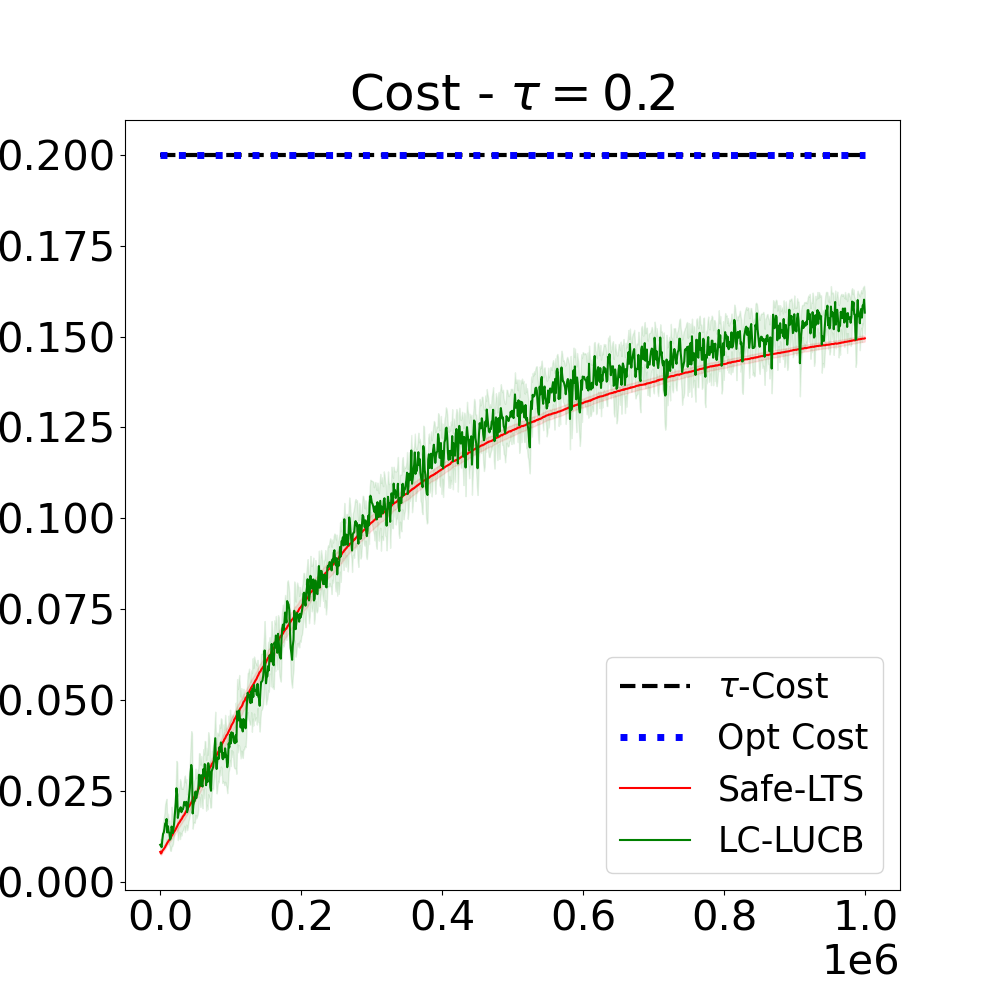}} 
\centering\subfigure{\includegraphics[width=0.325\linewidth]{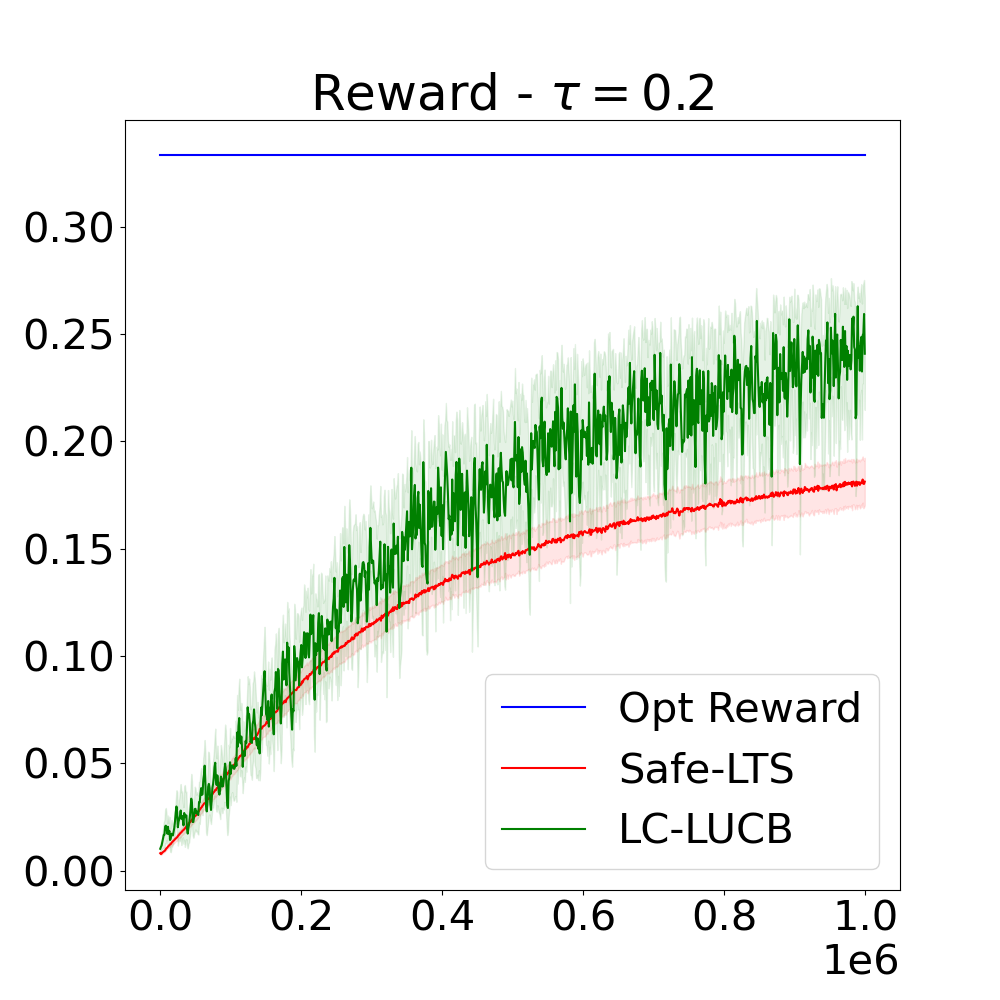}} 
\end{minipage}
\begin{minipage}{0.875\textwidth}%
\centering\subfigure{\includegraphics[width=0.325\linewidth]{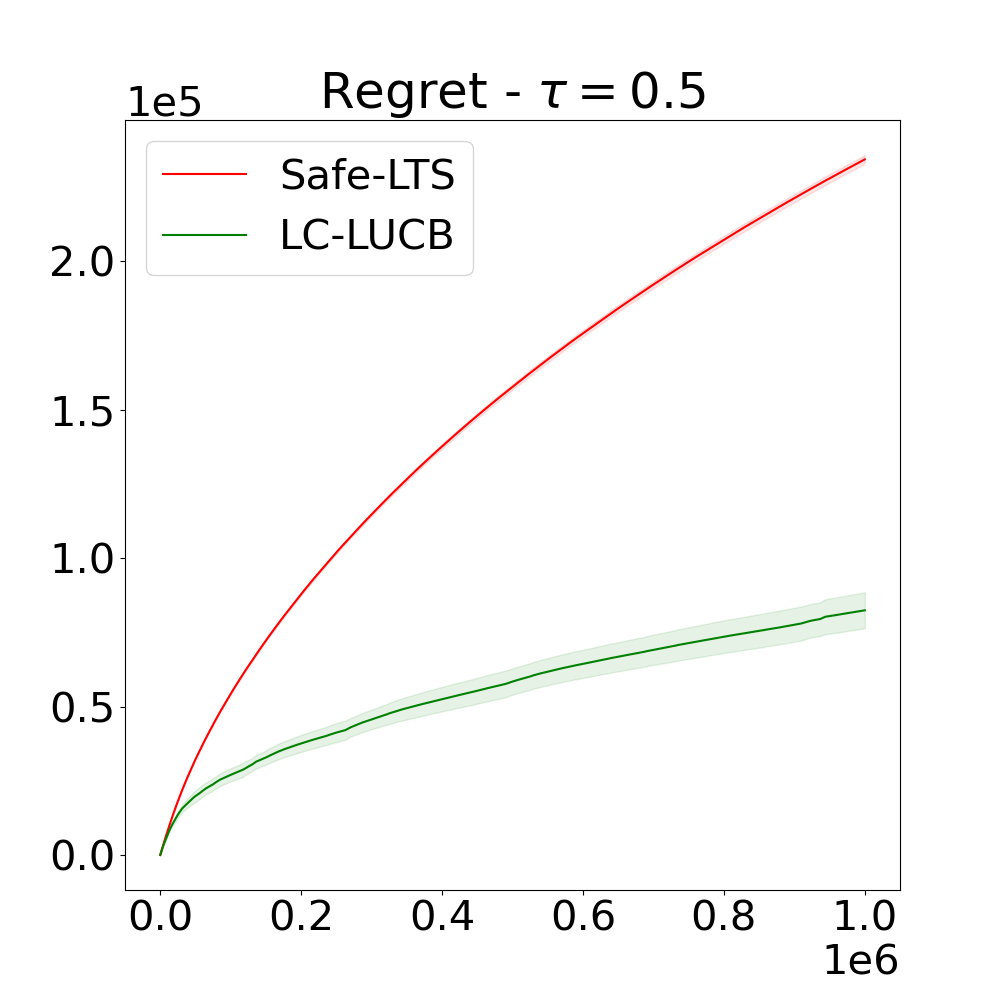}} 
\centering\subfigure{\includegraphics[width=0.325\linewidth]{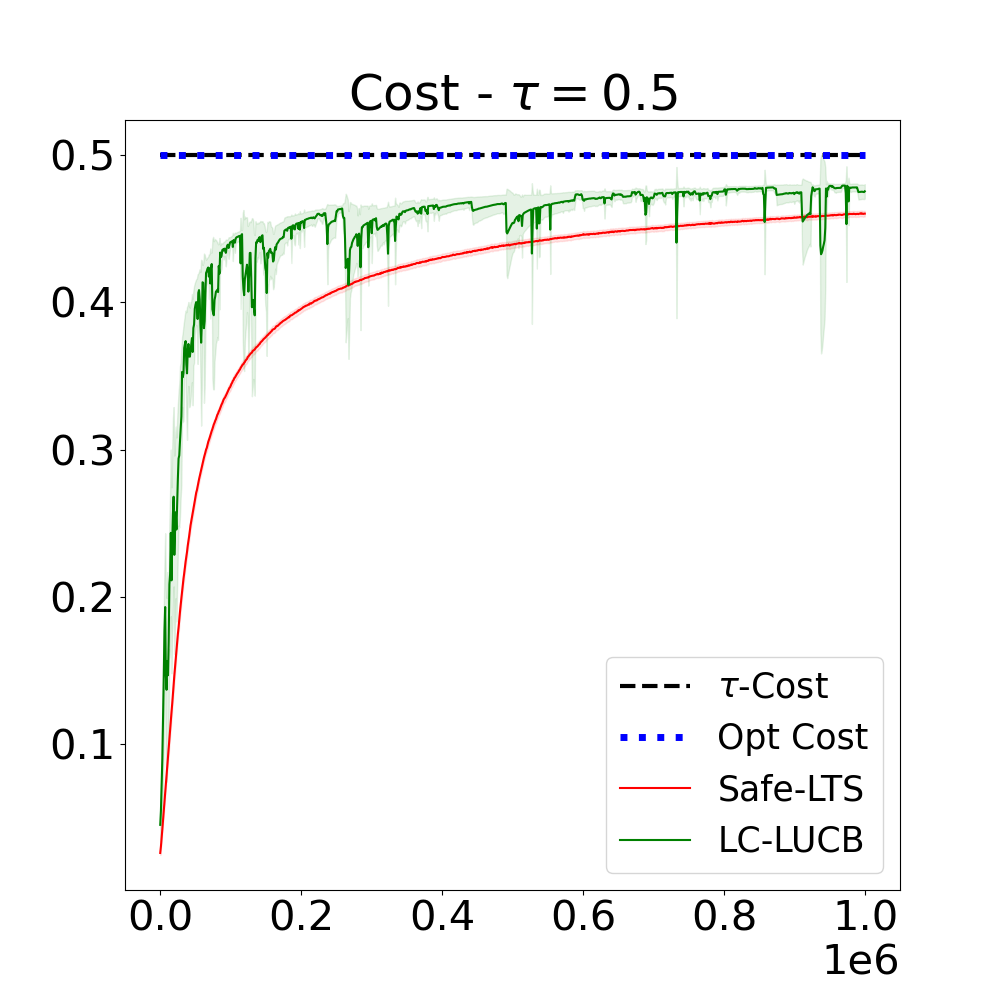}} 
\centering\subfigure{\includegraphics[width=0.325\linewidth]{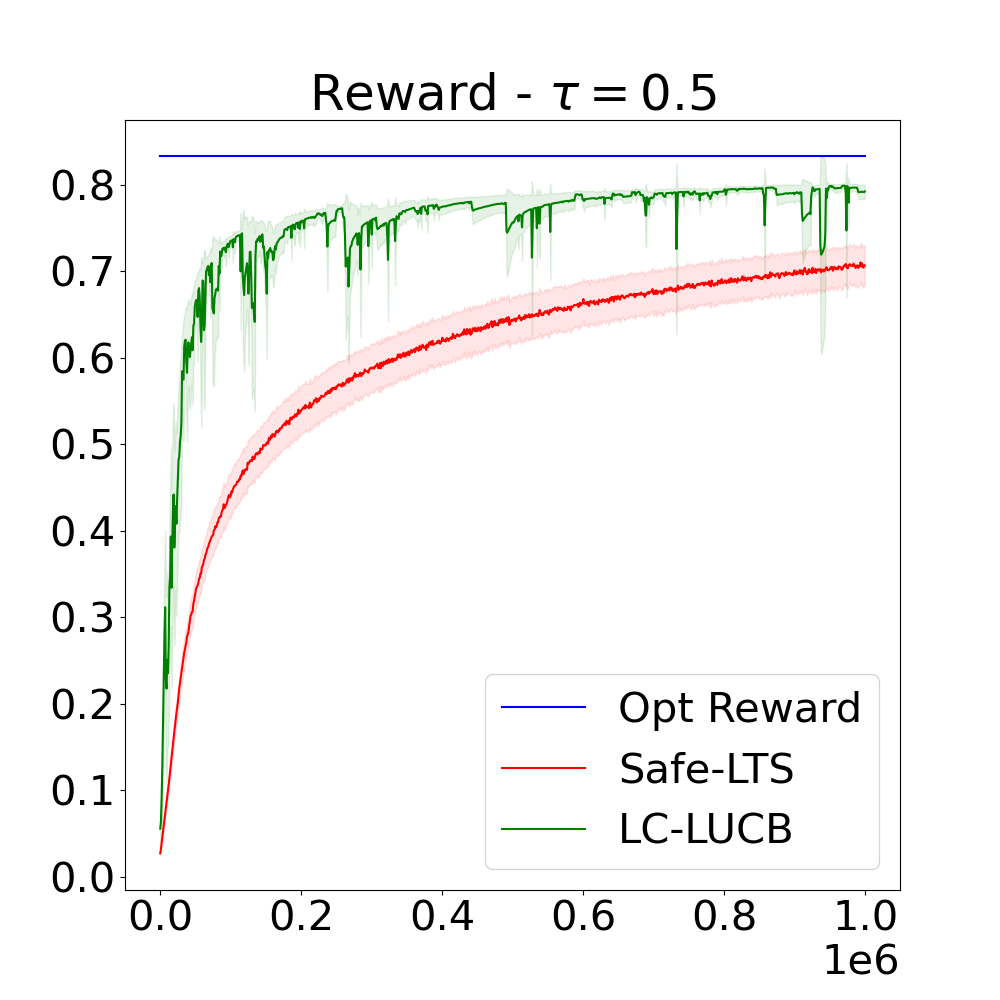}} 
\end{minipage}
\vspace{-2mm}
\begin{minipage}{0.875\textwidth}%
\centering\subfigure{\includegraphics[width=0.325\linewidth]{figs/Linear_Regret_0.8_1000000_3_uniform.png}} 
\centering\subfigure{\includegraphics[width=0.325\linewidth]{figs/Linear_Cost_0.8_1000000_3_uniform.png}} 
\centering\subfigure{\includegraphics[width=0.325\linewidth]{figs/Linear_Reward_0.8_1000000_3_uniform.png}} 
\end{minipage}
\vspace{-2mm}
\caption{\textbf{LC-LUCB:} {Dimension $d = 5$. \textbf{Top:} Constraint Threshold $\tau=0.2$. \textbf{Center:} Constraint Threshold $\tau = 0.5$. \textbf{Bottom:} Constraint Threshold $\tau = 0.8$. The shaded regions around the curves correspond to one standard deviation.}}
\vspace{2mm}
\label{fig:constrained_bandits_TS4}
\end{figure}

\begin{figure}[ht]%
\begin{center}
\begin{minipage}{1\textwidth}
\centering
Dimension $d=10$.
\end{minipage}
\end{center}
\centering
\vspace{-2mm}
\begin{minipage}{0.875\textwidth}%
\centering\subfigure{\includegraphics[width=0.325\linewidth]{figs/Linear_Regret_0.2_1000000_10_uniform.png}} 
\centering\subfigure{\includegraphics[width=0.325\linewidth]{figs/Linear_Cost_0.2_1000000_10_uniform.png}} 
\centering\subfigure{\includegraphics[width=0.325\linewidth]{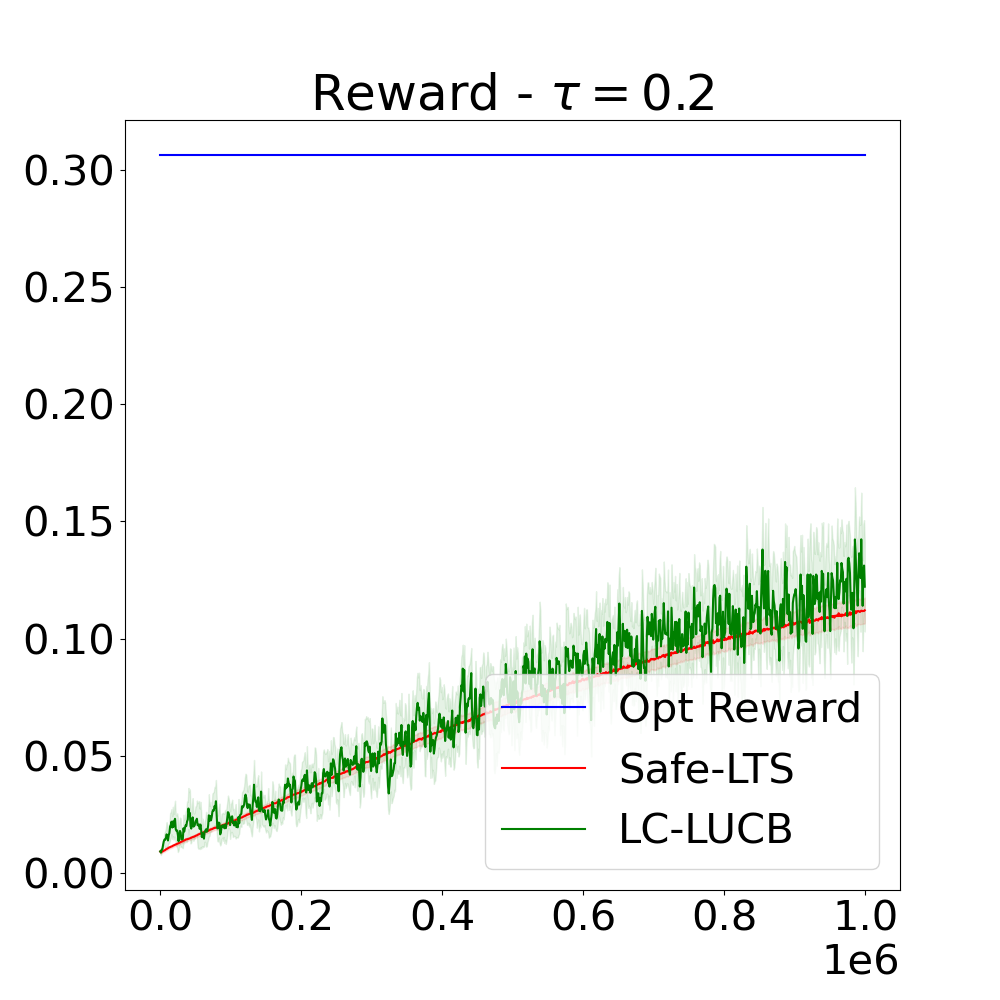}} 
\end{minipage}
\begin{minipage}{0.875\textwidth}%
\centering\subfigure{\includegraphics[width=0.325\linewidth]{figs/Linear_Regret_0.5_1000000_10_uniform.png}} 
\centering\subfigure{\includegraphics[width=0.325\linewidth]{figs/Linear_Cost_0.5_1000000_10_uniform.png}} 
\centering\subfigure{\includegraphics[width=0.325\linewidth]{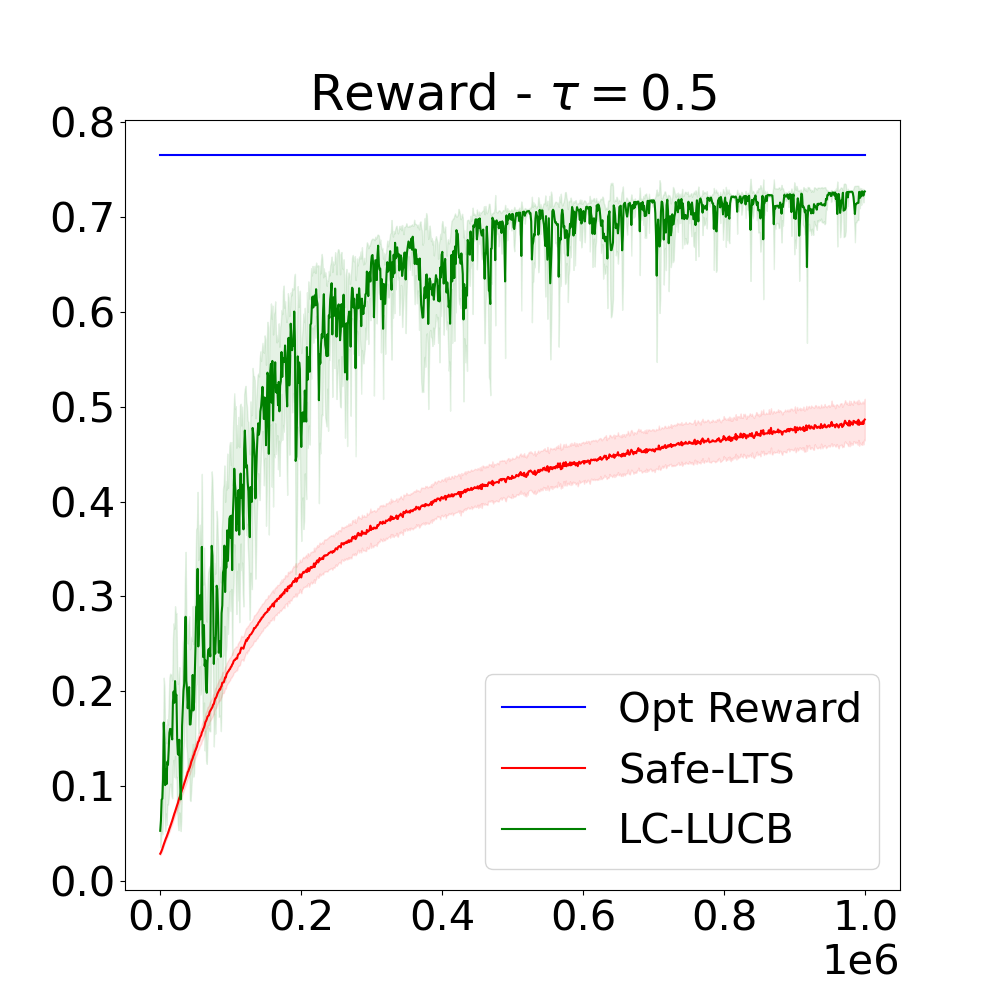}} 
\end{minipage}
\vspace{-2mm}
\begin{minipage}{0.875\textwidth}%
\centering\subfigure{\includegraphics[width=0.325\linewidth]{figs/Linear_Regret_0.8_1000000_10_uniform.png}} 
\centering\subfigure{\includegraphics[width=0.325\linewidth]{figs/Linear_Cost_0.8_1000000_10_uniform.png}} 
\centering\subfigure{\includegraphics[width=0.325\linewidth]{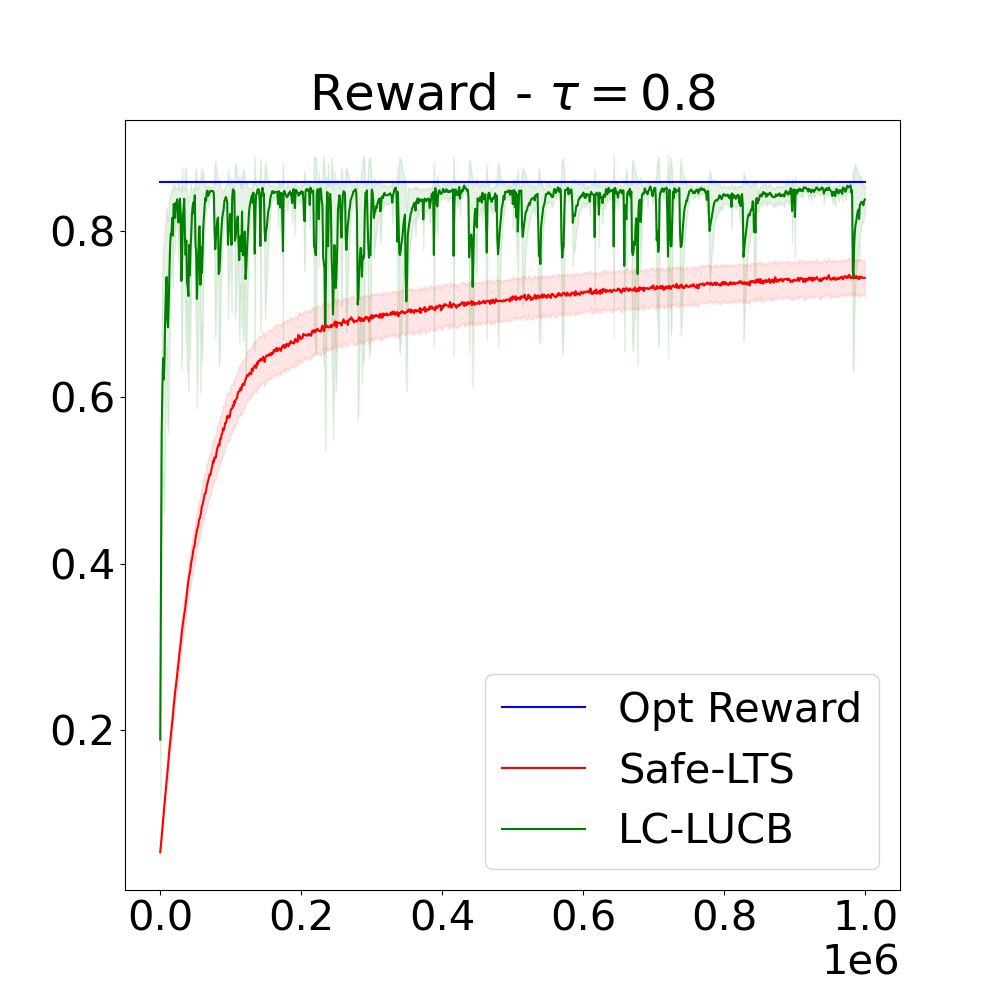}} 
\end{minipage}
\vspace{-2mm}
\caption{\textbf{LC-LUCB:} {Dimension $d = 10$. \textbf{Top:} Constraint Threshold $\tau=0.2$. \textbf{Center:} Constraint Threshold $\tau = 0.5$. \textbf{Bottom:} Constraint Threshold $\tau = 0.8$.  The shaded regions around the curves correspond to one standard deviation.}}
\vspace{2mm}
\label{fig:constrained_bandits_TS5}
\end{figure}

\section{Proofs of Section~\ref{section::expectation_constraints}}
\label{sec:proofs-algo-section}

\begin{proof}
{\bf of Proposition~\ref{prop:optimistic-reward-pessimistic-cost}: } The proof follows the exact same structure as Proposition~\ref{PROP:OPTIMISTIC-REWARD-PESSIMISTIC-COST}. Instead of using Equation~\ref{equation:supporting_optimistic_reward_pessimistic_cost} we utilize the following identity,

\begin{align*}
\widetilde{V}_t^r(\pi) &=  \max_{\theta \in \mathcal{C}_t^r(\alpha_r)} \mathbb{E}_{X \sim \pi}[ \langle X, \theta \rangle ] = \max_{\theta \in \mathcal{C}_t^r(\alpha_r)}  \langle x_\pi, \theta \rangle = \langle x_\pi, \widehat{\theta}_t \rangle + \max_{v:\|v\|_{\Sigma_t} \leq \alpha_r \beta_t(\delta, d)} \langle x_\pi, v \rangle \\
&\stackrel{\text{(a)}}{\leq} \langle x_\pi,\widehat{\theta}_t \rangle + \alpha_r \beta_t(\delta, d) \|  x_\pi \|_{\Sigma_t^{-1}}.
\end{align*}

The rest of the argument remains the same, substituting $x$ by $x_\pi$.

\end{proof}

\begin{proof}
{\bf of Proposition~\ref{prop:safe-set}: } Recall that 
\begin{equation*}
\tilde{c}_{\pi, t} = \frac{ \langle x_{\pi}^o, e_0\rangle c_0}{ \| x_0\| }  + \langle x_{\pi}^{o, \perp}, \widehat{t}_{\pi}^{o, \perp}\rangle  + \alpha_c \beta_t(\delta, d-1)\| x_{\pi}^{o, \perp} \|_{ (\Sigma^{o, \perp}_t)^{-1}}\leq \tau. 
\end{equation*}
Conditioned on the event $\mathcal{E}$ defined in Eq.~\ref{eq:high-prob-event}, it follows that
\begin{align*}
    |\langle x_{\pi}^{o, \perp} ,   \widehat{\mu}_t^{o, \perp} -\mu_*^{o, \perp}\rangle| &\leq \| \mu_*^{o, \perp} - \widehat{\mu}_t^{o, \perp} \|_{\Sigma_t^{o, \perp}}\| x_{\pi}\|_{(\Sigma^{o, \perp}_t)^{-1}}   \\
    &\leq \langle x_{\pi}^{o, \perp} ,   \widehat{\mu}_t^{o, \perp} -\mu_*^{o, \perp}\rangle \beta_t
    (\delta, d-1) \| x_{\pi}\|_{(\Sigma^{o, \perp}_t)^{-1}}.
\end{align*}
Thus, we have
\begin{equation}
    0 \leq \langle x_{\pi}^{o, \perp} ,   \widehat{\mu}_t^{o, \perp} -\mu_*^{o, \perp}\rangle + \beta_t
    (\delta, d-1) \| x_{\pi}\|_{(\Sigma^{o, \perp}_t)^{-1}}. 
    \label{equation::prop_3_eq1}
\end{equation}
Note that 
\begin{align}
\label{equation::prop_3_eq2}
c_{\pi} &= \frac{ \langle x_{\pi}^o, e_0\rangle c_0}{ \| x_0\| }  + \langle x_{\pi}^{o, \perp}, \mu_*^{o, \perp}\rangle \\ 
&\leq \underbrace{\frac{ \langle x_{\pi}^o, e_0\rangle c_0}{ \| x_0\| }  + \langle x_{\pi}^{o, \perp}, \widehat{\mu}_{t}^{o, \perp}\rangle  + \alpha_c \beta_t(\delta, d-1)\| x_{\pi}^{o, \perp} \|_{ (\Sigma^{o, \perp}_t)^{-1}}}_{\mathrm{(V)}}. \notag
\end{align}
The above inequality holds by adding the inequality in Eq.~\ref{equation::prop_3_eq1} to Eq.~\ref{equation::prop_3_eq2}. Since by assumption we have (V) $\leq \tau$ for all $\pi \in \Pi_t$, we obtain that $c_\pi \leq \tau$ which concludes the proof.
\end{proof}

\section{Constrained Multi-Armed Bandits}

\subsection{The LP Structure}
\label{section::LP_structure_appendix}

The main purpose of this section is to prove the optimal solutions of the linear program from (\ref{eq::noisy_LP}) are supported on a set of size at most $2$. This structural result will prove important to develop simple efficient algorithms to solve for solving it. Let's recall the form of the Linear program in~\eqref{eq::noisy_LP}, i.e.,
\begin{equation*}
\max_{\pi \in \Delta_K} \; \sum_{a \in \mathcal{A}} \pi_a u^r_{a}(t), \qquad \text{s.t. } \; \sum_{a \in \mathcal{A}} \pi_a u_a^c(t) \leq \tau. 
\end{equation*}
Let's start by observing that in the case $K= 2$ with $\mathcal{A}  = \{ a_1, a_2\}$ and $u_{a_1}^c(t) < \tau < u_{a_2}^c(t) $, the optimal policy $\pi^*$ is a mixture policy satisfying:
\begin{equation}
\label{equation::optimal_policy_pair}
\pi_{a_1}^* = \frac{ u_{a_2}^c(t) - \tau  }{u^c_{a_2}(t) - u^c_{a_1}(t)}, \qquad \pi_{a_2}^* = \frac{ \tau - u_{a_1}^c(t)}{u^c_{a_2}(t) - u^c_{a_1}(t)}.
\end{equation}
The main result in this section is the following Lemma:

\begin{lemma}[support of $\pi^*$]
\label{lemma::LP_support_appendix} 
If~\eqref{eq::noisy_LP} is feasible, there exists an optimal solution with at most $2$ non-zero entries. 
\end{lemma}

\begin{proof}
We start by inspecting the dual problem of (\ref{eq::noisy_LP}):
\begin{equation}
\label{equation::dual_LP}
\min_{\lambda \geq 0} \max_{a} \lambda( \tau - u_a^c(t) ) + u_a^r(t) \tag{D}
\end{equation}
This formulation is easily interpretable. The quantity $\tau - u_a^c(t)$ measures the feasibility gap of arm $a$, while $u^r_a(t)$ introduces a dependency on the reward signal. Let $\lambda^*$ be the optimal value of the dual variable $\lambda$. Define $\mathcal{A}^*\subseteq \mathcal{A}$ as $\mathcal{A}^* = \arg\max_a \lambda^* (\tau -  u^c_a(t) ) + u^r_a(t)$. By complementary slackness the set of nonzero entries of $\pi^*$ must be a subset of $\mathcal{A}^*$.

If $\left|\mathcal{A}^* \right| = 1$, complementary slackness immediately implies the desired result. If $a_1, a_2$ are two elements of $\mathcal{A}^*$, it is easy to see that:
\begin{equation*}
  u^r_{a_1}(t) - \lambda^* u^c_{a_1}(t) = u^r_{a_2}(t) -\lambda^* u^c_{a_2}(t) , 
\end{equation*}
and thus,
\begin{equation}
\label{eq::dual_lambda_explicit}
    \lambda^* = \frac{ u_{a_2}^r(t) - u^r_{a_1}(t)}{ u_{a_2}^c(t) - u_{a_1}^c(t)}.
\end{equation}
If $\lambda^* = 0$, the optimal primal value is achieved by concentrating all mass on any of the arms in $\mathcal{A}^*$. Otherwise, plugging~\eqref{eq::dual_lambda_explicit} back into the objective of (\ref{equation::dual_LP}) and rearranging the terms, we obtain
\begin{align*}
\text{(D)} = \lambda^*( \tau- u_{a_1}^c(t)) + u^r_{a_1}(t) = u^r_{a_2}(t)\left(\frac{ \tau - u_{a_1}^c(t) }{ u_{a_2}^c(t) - u_{a_1}^c(t)} \right) + u_{a_1}^r(t) \left(\frac{ u_{a_2}^c(t) - \tau  }{ u_{a_2}^c(t) - u_{a_1}^c(t) }  \right).
\end{align*}
If $u_{a_2}^c(t)  \geq \tau \geq u_{a_1}^c(t)$, we obtain a feasible value for the primal variable $\pi_{a_1}^* = \frac{ \tau - u_{a_1}^c(t) }{ u_{a_2}^c(t) - u_{a_1}^c(t)} $, $\pi_{a_2}^* = \frac{  u_{a_2}^c(t) -\tau }{ u_{a_2}^c(t) - u_{a_1}^c(t)}$ and zero for all other $a \in \mathcal{A} \backslash \{ a_1, a_2 \}$. Since we have assumed (\ref{eq::noisy_LP}) to be feasible there must be either one arm $a^* \in \mathcal{A}^*$ satisfying $a^* = \arg\max_{a \in \mathcal{A}^*} u_{a}^r(t)$ and $u_{a^*}^c(t) \leq \tau$ or two such arms $a_1$ and $a_2$ in $\mathcal{A}^*$ that satisfy $u_{a_2}^c(t)  \geq \tau \geq u_{a_1}^c(t)$, since otherwise it would be impossible to produce a feasible primal solution without having any of its supporting arms $a$ satisfying $u_a^c(t) \leq \tau$, there must exist an arm $a \in \mathcal{A}^*$ with $u_a^c(t) < \tau$. This completes the proof. \end{proof}

From the proof of Lemma \ref{lemma::LP_support} we can conclude the optimal policy is either a delta mass centered at the arm with the largest reward - whenever this arm is feasible -  or it is a strict mixture supported on two arms. 

A further consequence of Lemma~\ref{lemma::LP_support_appendix} is that it is possible to find the optimal solution $\pi^*$ to problem~\ref{eq::noisy_LP} by simply enumerating all pairs of arms $(a_i, a_j)$ and all singletons, compute their optimal policies (if feasible) using Equation~\ref{equation::optimal_policy_pair} and their values and selecting the feasible pair (or singleton) achieving the largest value. More sophisticated methods can be developed by taking into account elimination strategies to prune out arms that can be determined in advance not to be optimal nor to belong to an optimal pair. Overall this method is more efficient than running a linear programming solver on~(\ref{eq::noisy_LP}).

If we had instead $m$ constraints, a similar statement to Lemma \ref{lemma::LP_support} holds, namely it is possible to show the optimal policy will have support of size at most $m+1$. The proof is left as an exercise for the reader.

\subsection{Regret Analysis}
\label{section:regret_analysis_appendix}

In order to show a regret bound for Algorithm \ref{alg::optimism_pessimism}, we start with the following regret decomposition:
\begin{align*}
    \mathcal{R}_\Pi(T) &= \sum_{t=1}^T \mathbb{E}_{a\sim \pi^*}[\bar{r}_a] - \mathbb{E}_{a \sim \pi_t}[\bar{r}_{a}] \\
    &=\underbrace{\left( \sum_{t=1}^T \mathbb{E}_{a\sim \pi^*}[\bar{r}_a] - \mathbb{E}_{a \sim \pi_t}[u_a^r(t)] \right)}_{\text{(I)}} +\underbrace{ \left(\sum_{t=1}^T \mathbb{E}_{a\sim \pi_t}[u^r_a(t)] - \mathbb{E}_{a \sim \pi_t}[\bar{r}_{a}] \right) }_{\text{(II)}}.
\end{align*}
In order to bound $\mathcal{R}_\Pi(T)$, we independently bound terms (I) and (II). We start by bounding term (I). We proceed by first proving an Lemma~\ref{lemma::optimism}, the equivalent version of Lemma~\ref{lemma:linear_bandits_optimism} for the multi armed bandit problem. 

\subsection{Proof of Lemma~\ref{lemma::optimism}}

\optimismMAB*

\begin{proof} 
Throughout this proof we denote as $\pi_0$ to the delta function over the safe arm $1$. We'll use the notation $u_a^r(t) = \bar{r}_a + \xi_a^r(t)$ and $u_a^c(t) = \bar{c}_a + \xi_a^c(t)$. We start by noting that under $\mathcal{E}$, and because $\alpha_r, \alpha_c \geq 1$, then:
\begin{equation}
\label{equation::confidence_interval_lower_bounds}
(\alpha_r-1)\beta_a(t) \leq \xi_a^r(t) \leq (\alpha_r + 1)\beta_a(t) \text{ } \forall a \quad \text{ and } \quad (\alpha_c-1)\beta_a(t) \leq \xi_a^c(t) \leq (\alpha_c+1)\beta_a(t) \text{ } \forall a \neq 0.
\end{equation}
If $\pi^*\in \tilde \Pi_t$, it immediately follows that:
\begin{equation}
\label{equation::feasible_lower_bound}
\mathbb{E}_{a\sim \pi^*}\left[\bar{r}_a \right] \leq \mathbb{E}_{a\sim \pi^*}\left[u_a^r(t)\right] \leq \mathbb{E}_{a \sim \pi_t}\left[  u_a^r(t)\right].
\end{equation}

Let's now assume $\pi^* \not\in \tilde\Pi_t$, i.e.,~$\mathbb{E}_{a\sim\pi^*}\left[u_a^c(t)\right] > \tau$. Let $\pi^* = \rho^* \bar{\pi}^* + (1-\rho^*)\pi_0$ with $\bar{\pi}^* \in \Delta_K[2:K]$\footnote{In other words, the support of $\bar{\pi}^*$ does not contain the safe arm $1$.}.

Consider a mixture policy $\widetilde{\pi}_t = \gamma_t \pi^* + (1-\gamma_t)\pi_0 = \gamma_t \rho^* \bar{\pi}^* +  (1-\gamma_t \rho^*)\pi_0$, where $\gamma_t$ is the maximum $\gamma_t\in [0,1]$ such that $\widetilde{\pi}_t\in\tilde \Pi_t$. It can be easily established that 
\begin{align*}
\gamma_t = \frac{\tau-\bar{c}_1}{\rho^*\mathbb{E}_{a \sim \bar{\pi}^*}\left[u^c_a(t)\right] - \rho^*\bar{c}_1} =\frac{\tau - \bar{c}_1}{ \mathbb{E}_{a \sim \bar{\pi}^*}[\rho^*(\bar{c}_a + \xi_a^c(t) )] - \rho^* \bar{c}_1} \stackrel{\text{(i)}}{ \geq}  \frac{\tau - \bar{c}_1}{\tau - \bar{c}_1 + \rho^* (1+\alpha_c)\mathbb{E}_{a \sim \bar{\pi}^*}[\beta_a(t)]}. 
\end{align*}
{\bf (i)} is a consequence of~\eqref{equation::confidence_interval_lower_bounds} and of the observation that since $\pi^*$ is feasible $\rho^*\mathbb{E}_{a \sim \bar{\pi}^*}[\bar{c}_a] + (1-\rho^*)\bar{c}_1 \leq \tau$. Let $\Delta = \mathbb{E}_{a \sim \pi^*}[\bar{r}_a] - \mathbb{E}_{a \sim \pi_0}[\bar{r}_a]$. Since $\widetilde{\pi}_t\in\Pi_t$, we have

\begin{align*}
\mathbb{E}_{a \sim \pi_t}[ u_a^r(t) ] &\geq \mathbb{E}_{a \sim \tilde \pi_t}[ u_a^r(t)] = \mathbb{E}_{a \sim \tilde \pi_t}[ \bar{r}_a] + \mathbb{E}_{a \sim \tilde \pi_t}[\xi_a(t) ] \\
&\stackrel{(a)}{\geq} \mathbb{E}_{a \sim \tilde \pi_t}[ \bar{r}_a] + (\alpha_r-1) \mathbb{E}_{a \sim \tilde \pi_t}[\beta_a(t)] \\
&\stackrel{(b)}{\geq} \gamma_t \mathbb{E}_{a \sim \pi^*}[\bar{r}_a] + (1-\gamma_t) \mathbb{E}_{a \sim \pi_0}[\bar{r}_a] + (\alpha_r-1) \gamma_t \rho^* \mathbb{E}_{a \sim \bar{\pi}^*}[\beta_a(t)]\\
&= \mathbb{E}_{a \sim \pi^*}[\bar{r}_a] + (1-\gamma_t) \mathbb{E}_{a \sim \pi_0}[\bar{r}_a] - (1-\gamma_t) \mathbb{E}_{a \sim \pi^*}[\bar{r}_a]+ (\alpha_r-1) \gamma_t \rho^* \mathbb{E}_{a \sim \bar{\pi}^*}[\beta_a(t)]\\
&= \mathbb{E}_{a \sim \pi^*}[\bar{r}_a] - (1-\gamma_t) \Delta + (\alpha_r-1) \gamma_t \rho^*\mathbb{E}_{a \sim \bar{\pi}^*} [ \beta_a(t) ] \\
&=\mathbb{E}_{a \sim \pi^*}[\bar{r}_a] + \underbrace{\gamma_t \left(    \Delta + (\alpha_r-1) \rho^* \mathbb{E}_{a \sim \bar{\pi}^*}[\beta_a(t)] \right) - \Delta}_{\mathrm{A}}.
\end{align*}

Where $(a)$ holds by Equation~\ref{equation::confidence_interval_lower_bounds}, $(b)$ holds by definition of $\tilde \pi_t$ and because $\mathbb{E}_{a \sim \pi^*}[\beta_a(t)] \geq \mathbb{E}_{a \sim \bar{\pi}^*}[\beta_a(t)]$. Let $C_0 = \rho^* \mathbb{E}_{a \sim \bar{\pi}^*} [\beta_a(t)]$. Let's show conditions under which $\mathbb{I} \geq $. The following chain of inequalities holds,

\begin{equation*}
\mathrm{A} \stackrel{(a)}{\geq} \frac{\tau - \bar{c}_1}{\tau - \bar{c}_1 + (1+\alpha_c) C_0} \left( \Delta + (\alpha_r-1)C_0 - \Delta \right) - \Delta.
\end{equation*}
Where $(a)$ follows by substituting $\gamma \geq \frac{\tau - \bar{c}_1}{\tau - \bar{c}_1 + (1+\alpha_c) C_0}$. Following the same logic as in the analysis of Equation~\ref{equation::condition_optimism_high_prob} in the proof of Lemma~\ref{lemma:linear_bandits_optimism} we conclude that $\mathrm{I}$ is non-negative whenever,
\begin{equation*}
(\tau-\bar{c}_1) (\alpha_r-1) \geq (1+\alpha_c) \Delta.
\end{equation*}

Since $\Delta \leq 1-\bar{r}_1$ this concludes the proof.
\end{proof}

\begin{proposition}\label{proposition::bounding_term_I}
If $\delta = \frac{\delta'}{4KT}$ for $\delta' \in (0,1)$, $\alpha_r, \alpha_c \geq 1$ with $ (1+\alpha_c) (1-\bar{r}_1) \leq (\tau-\bar{c}_1) (\alpha_r-1)$, then with probability at least $1-\frac{\delta'}{2}$, we have
\begin{equation*}
  \sum_{t=1}^T \mathbb{E}_{a\sim \pi^*}[\bar{r}_a] - \mathbb{E}_{a \sim \pi_t}[u_a^r(t)] \leq 0
\end{equation*}
\end{proposition}

\begin{proof}
A simple union bound implies that $\mathbb{P}(\mathcal{E}) \geq 1-\frac{\delta'}{2}$. Combining this observation with Lemma \ref{lemma::optimism} yields the result.
\end{proof}

Term (II) can be bounded using the confidence intervals radii:

\begin{proposition}
\label{proposition::bounding_term_II}
If $\delta = \frac{\delta'}{4KT}$ for an $\delta' \in (0,1)$, then with probability at least $1-\frac{\delta'}{2}$, we have
\begin{equation*}
\sum_{t=1}^T \mathbb{E}_{a\sim \pi_t}[u_a^r(t)] - \mathbb{E}_{a \sim \pi_t}[\bar{r}_{a}] \leq (\alpha_r+1) \left(2\sqrt{2TK\log(1/\delta)} + 4\sqrt{T\log(2/\delta')\log(1/\delta)} \right).
\end{equation*}
\end{proposition}

\begin{proof}
Under these conditions $\mathbb{P}(\mathcal{E}) \geq 1-
\frac{\epsilon}{2}$. Recall $u_a^r(t) = \widehat{r}_a(t) + \alpha_r\beta_a(t)$ and that conditional on $\mathcal{E}$, $\bar{r}_a \in [\widehat{r}_a(t) - \beta_a(t),\widehat{r}_a(t) + \beta_a(t)]$ for all $t \in [T]$ and $a \in \mathcal{A}$. Thus, for all $t$, we have
\begin{equation*}
    \mathbb{E}_{a \sim \pi_t}[u_a^r(t)] - \mathbb{E}_{a \sim \pi_t}[\bar{r}_a] \leq (\alpha_r +1)\mathbb{E}_{a \sim \pi_t}[\beta_a(t)].
\end{equation*}
Let $\mathcal{F}_{t-1}$ be the sigma algebra defined up to the choice of $\pi_t$ and $a_t'$ be a random variable distributed as $\pi_t \mid \mathcal{F}_{t-1}$ and conditionally independent from $a_t$, i.e.,~$a'_t \perp a_t \mid \mathcal{F}_{t-1}$. Note that by definition the following equality holds: %
\begin{equation*}
    \mathbb{E}_{a \sim \pi_t}[\beta_a(t)] = \mathbb{E}_{a'_t \sim \pi_t}[\beta_a(t) \mid \mathcal{F}_{t-1}].
\end{equation*}
Consider the following random variables $A_t =  \mathbb{E}_{a'_t \sim \pi_t} [\beta_{a'_t}(t) \mid \mathcal{F}_{t-1}]- \beta_{a_t}(t)$. Note that $M_t = \sum_{i=1}^t A_i$ is a martingale. Since $|A_t| \leq 2\sqrt{2 \log(1/\delta)}$, a simple application of Azuma-Hoeffding\footnote{We use the following version of Azuma-Hoeffding: if $X_n$, $n\geq 1$ is a martingale such that $|X_i - X_{i-1}| \leq d_i$, for $1 \leq i \leq n$, then for every $n \geq 1$, we have $\mathbb{P}(X_n > r) \leq \exp\left(-\frac{r^2 }{2\sum_{i=1}^n d_i^2}\right)$.} implies:
\begin{equation*}
\mathbb{P}\left(\underbrace{\sum_{t=1}^T \mathbb{E}_{a \sim \pi_t} [\beta_a(t)] \geq \sum_{t=1}^T \beta_{a_t}(t) + 4\sqrt{T\log(2/\delta')\log(1/\delta)}}_{\mathcal{E}_A^c}\right ) \leq \epsilon/2.
\end{equation*}
We can now upper-bound $\sum_{t=1}^T \beta_{a_t}(t)$. Note that $\sum_{t=1}^T \beta_{a_t}(t) = \sum_{a \in \mathcal{A}}\sum_{t=1}^T \mathbf{1}\{a_t=a\}\beta_a(t)$. We start by bounding for an action $a\in\mathcal A$:
\begin{align*}
    \sum_{t=1}^T \mathbf{1}\{a_t=a\}\beta_a(t) = \sqrt{2\log(1/\delta)} \sum_{t=1}^{T_a(T)} \frac{1}{\sqrt{t}} \leq 2\sqrt{2T_a(T)\log(1/\delta)}.
\end{align*}
Since $\sum_{a\in\mathcal A } T_a(T) = T$ and by concavity of $\sqrt{ \cdot}$, we have
\begin{equation*}
    \sum_{a\in\mathcal A} 2\sqrt{2T_a(T)\log(1/\delta)} \leq 2\sqrt{2TK\log(1/\delta)}.
\end{equation*}
Conditioning on the event $\mathcal{E} \cap \mathcal{E}_A$ whose probability satisfies $\mathbb{P}( \mathcal{E}\cap \mathcal{E}_A) \geq 1-\delta' $ yields the result.
\end{proof}

We can combine these two results into our main theorem:
\theoremoptregret*

\begin{proof}
This result is a direct consequence of Propositions \ref{proposition::bounding_term_I} and \ref{proposition::bounding_term_II} by setting  $ \delta = 4KT \delta'$.
\end{proof}

\subsection{Multiple Constraints}
\label{section::multiple_constraints_appendix}

We consider the problem where the learner must satisfy $M$ constraints with threshold values $\tau_1, \ldots, \tau_M$. Borrowing from the notation in the previous sections, we denote by as $\{\bar{r}_a\}_{a\in \mathcal{A}}$ the mean reward signals and $\{ \bar{c}_a^{(i)} \}$ the mean cost signals for $i = 1,\ldots, M$. The full information optimal policy can be obtained by solving the following linear program:
\begin{align*}
\label{eq::no_noise_LP_multiple_constraints}\tag{P-M}
\max_{\pi \in \mathrm{\Delta}_K} \; \sum_{a \in \mathcal{A}} \pi_a \bar{r}_a, \qquad \text{s.t. } \; \sum_{a \in \mathcal{A}} \pi_a \bar{c}^{(i)}_a \leq \tau_i, \quad \text{ for } \; i=1, \ldots, M. 
\end{align*}
In order to ensure the learner's ability to produce a feasible policy at all times, we make the following assumption:
\begin{assumption}
    The learner has knowledge of $\bar{c}_1^{(i)} < \tau_i$ for all $i = 1, \ldots, M$.  %
\end{assumption}
We denote by $\{ \widehat{r}_a \}_{a \in \mathcal{A}}$ and $\{ \widehat{c}_a^{(i)} \}_{a\in\mathcal{A}}$, for $i = 1, \ldots, M$ the empirical means of the reward and cost signals. We call $\{ u_a^r(t)\}_{a \in \mathcal{A}}$ to the upper confidence bounds for our reward signal and $\{ u_a^{c}(t, i)\}_{a \in \mathcal{A}}$, for $i = 1, \ldots, M$ the costs' upper confidence bounds:
\begin{equation*}
u_a^r(t) = \widehat{r}_a(t) + \alpha_r  \beta_a(t), \qquad u_a^c(t, i) = \widehat{c}^{(i)}_a(t) + \alpha_c  \beta_a(t),
\end{equation*}
where $\beta_a(t) = \sqrt{2\log(1/\delta)/T_a(t)}$, $\delta \in (0,1)$ as before. A straightforward extension of Algorithm \ref{alg::optimism_pessimism} considers instead the following $M-$constraints LP:
\begin{align*}
\label{eq::noisy_LP_multiple}
\tag{$\widehat{P-M}$} 
\max_{\pi\in\mathrm{\Delta}_K} \;\; \sum_{a \in \mathcal{A}} \pi_a \; u^r_a(t), \qquad \text{s.t.} \quad \sum_{a \in \mathcal{A}} \pi_a \; u_a^c(t, i)\leq \tau_i, \quad \text{ for } \; i= 1, \ldots, M.
\end{align*}

We now generalize Lemma \ref{lemma::optimism}:
\begin{lemma}\label{lemma::optimisim_multiple}

Let $\alpha_r, \alpha_c\geq 1$ satisfying $\alpha_c \leq \min_i(\tau_i- \bar{c}^{(i)}_1) (\alpha_r-1)$. Conditioning on $\mathcal{E}_a(t)$ ensures that with probability $1-\delta$:
\begin{equation*}
    \mathbb{E}_{a \sim \pi_t}\left[u_a^r(t)\right] \geq \mathbb{E}_{a \sim \pi^*}\left[\bar{r}_a \right].
\end{equation*}

\end{lemma}

\begin{proof}
The same argument as in the proof of Lemma \ref{lemma::optimism} follows through, the main ingredient is to realize that $\gamma_t$ satisfies the sequence of inequalities in the lemma with $\tau - \bar{c}_1$ substituted by $\min \tau_i - \bar{c}_1^{(i)}$.
\end{proof}

The following result follows:

\begin{theorem}[Multiple Constraints Main Theorem]
If $\epsilon \in(0,1)$, $ \alpha_c=1$ and $\alpha_r = \frac{2}{\min_i \tau_i-\bar{c}^{(i)}_1} + 1$, then with probability at least $1-\epsilon$, Algorithm \ref{alg::optimism_pessimism} satisfies the following regret guarantee:
\begin{equation*}
    \mathcal{R}_\Pi(T)  \leq \left(\frac{2}{\min_i\tau_i-\bar{c}^{(i)}_1} +1\right)\left(2\sqrt{2TK\log(4KT/\epsilon)} + 4\sqrt{T\log(2/\epsilon)\log(4KT/\epsilon)} \right)
\end{equation*}
\end{theorem}

\begin{proof}
The proof follows the exact same argument we used for the proof of Theorem~\ref{theorem::contrained_MAB} substituting $\tau - \bar{c}_1$ by $\min_i \tau_i -\bar{c}_1^{(i)}$.
\end{proof}

\section{Lower Bounds}
\label{appendix::High_probability_Lower_Bound}

In this section we prove the two lower bounds from the main text. We will do so by exhibiting a lower bound for the

We start by stating a generalized version of the divergence decomposition lemma for bandits. %
The proof is a direct application of Lemma~15.1~in~\citet{lattimore2018bandit} to this case.

\begin{lemma}[Divergence decomposition for constrained multi-armed bandits] 
\label{lemma::divergence_decomposition}

Let $\nu = ((P_1, Q_1), \cdots, (P_K, Q_K))$ be the reward and constraint distributions associated with one instance of the single constraint multi-armed bandit, and let $\nu' = ((P_1', Q_1'), \cdots, (P_K', Q_K'))$ be the reward and constraint distributions associated with another constrained bandit instance. Fix some algorithm $\mathcal{A}$ and let $\mathbb{P}_\nu = \mathbb{P}_{\nu^\mathcal{A}}$ and $\mathbb{P}_{\nu'} = \mathbb{P}_{\nu'^\mathcal{A}}$ be the probability measures on the cannonical bandit model (see Section~4.6~of~\citealt{lattimore2018bandit}) induced by the $T$ round interconnection of $\mathcal{A}$ and $\nu$ (respectively $\mathcal{A}$ and $\nu'$). Then,
\begin{equation*}
\mathrm{KL}(\mathbb{P}_\nu,\mathbb{P}_{\nu'}) = \sum_{a=1}^K \mathbb{E}_{\nu}\big[T_a(T)\big] \mathrm{KL}\big((P_a,Q_a), (P_a', Q_a')\big),
\end{equation*}
where $T_a(T)$ denotes the number of times that arm "$a$" has been pulled by $\mathcal{A}$ up to time $T$. 
\end{lemma}

The following two lemmas will also be useful in our lower-bound proof, so we state them here. 

\begin{lemma}[Gaussian Divergence] 
\label{lemma::gaussian_divergence}
The divergence between two multi-variate normal distributions with means $\mu_1, \mu_2 \in \mathbb{R}^d$ and spherical identity covariance $\mathbb{I}_d$ is equal to
\begin{equation*}
\mathrm{KL}\big(\mathcal{N}(\mu_1, \mathbb{I}_d) , \mathcal{N}(\mu_2, \mathbb{I}_d)\big) = \| \mu_1- \mu_2\|^2/2.
\end{equation*}
\end{lemma}

\begin{lemma}\label{lemma::binary_entropy_definition_properties}
The binary relative entropy to be
\begin{equation*}
d(x,y) = x\log\left(\frac{x}{y}\right) + (1-x)\log\left(\frac{1-x}{1-y}\right),
\end{equation*}
and satisfies
\begin{equation}
\label{equation::lower_bound_divergence}
d(x,y) \geq (1/2)\log(1/4y),
\end{equation}
for $x \in [1/2,1]$ and $y \in (0,1)$. 

\end{lemma}

\begin{lemma}[Adapted from~\citealt{kaufmann2016complexity}, Lemma 1.]
\label{lemma::binary_relative_entropy}
Let $\nu, \nu'$ be two constrained bandit models with $K$ arms. Borrow the setup, definitions and notations of Lemma~\ref{lemma::divergence_decomposition}, then for any measurable event $\mathcal{B} \in \mathcal{F}_T$:
\begin{equation}
\label{equation::relative_entropy_bound}
\mathrm{KL}(\mathbb{P}_\nu,\mathbb{P}_{\nu'}) = \sum_{a=1}^K \mathbb{E}_\nu\big[T_a(T)\big] \mathrm{KL}\big((P_a,Q_a),(P_a', Q_a')\big) \geq d\big(\mathbb{P}_\nu(\mathcal{B}),\mathbb{P}_{\nu'}(\mathcal{B})\big).
\end{equation}
\end{lemma}

We start by showing that under an appropriate noise assumption, it is possible to reduce the constrained (in expectation) Multi Armed Bandit (CE-MAB) problem studied in~\citet{pacchiano2020stochastic} to our setting. The argument behind the proof of the main result in this section, the lower bound Theorem~\ref{theorem::lower_bound} relies on the problem structure behind the LC-LUCB version of the CE-MAB problem given by this reduction.

\paragraph{Setup:} Let's first describe the CE-MAB setup. %
In the constrained $K$-armed bandit setting, the action sets satisfy $\mathcal{A}_t = \Delta_{K}$, where $\Delta_K$ is the $K$-dimensional simplex. The reward and cost parameters are reduced to the $K$-dimensional vectors containing the mean reward and cost values of the $K$ arms, i.e.,~$\theta_* =(\bar{r}_0,\ldots,\bar{r}_{K-1})$ and $\mu_* = (\bar{c}_0,\ldots,\bar{c}_{K-1})$. 

In this case $X_t \in \Delta_K$ and we assume that abusing notation $a_t \sim X_t$, an index in $[K]$ is sampled from the distribution $X_t$, after which the reward value and the cost satisfy:
\begin{equation*}
R_t = \bar{r}_{a_t} + \nu^r_t\;, \qquad\qquad C_t = \bar{c}_{a_t} + \nu^c_t\;,
\end{equation*}
where $\nu_t^r$ and $\nu_t^c$ are conditionally zero mean sub-Gaussian random variables. The learner's objective is to play policies $X_t$ such that for all $t$, $\langle X_t, \mu_* \rangle \leq \tau$ while at the same time maximizing $\langle X_t, \theta_*\rangle$. We work under the assumption that $\bar{c}_0$ is known to the learner and satisfies $\bar{c}_0 \leq \tau$.

\paragraph{Reduction:} We now show that it is possible to reduce the CE-MAB problem to the LC-LUCB setup. Using the notation in Assumption~\ref{ass:noise-sub-gaussian} we define $\xi_t^r = R_t - \sum_{a \in [K]} X_t(a) \bar{r}_a $ and $\xi_t^r = C_t - \sum_{a \in [K]} X_t(a) \bar{c}_a$. Where $X_t(a)$ corresponds to the $a-$th coordinate of $X_t$. Notice that:
\begin{equation*}
R_T = \langle X_t, \theta_* \rangle + \xi_t^r\;, \qquad\qquad C_T = \langle X_t, \mu_* \rangle + \xi_t^c\;,
\end{equation*}
with $\xi_t^r$ and $\xi_t^c$ both conditionally zero mean subgaussian random variables:

Indeed since $\{\bar{r}_a, \bar{c}_a\}_{a \in [K]}$ are all assumed to be bounded, the conditional subgaussianity assumption of $\xi_t^r$ and $\xi_t^c$ is satisfied for an appropriate choice of subgaussianity parameter $R$, dependent on the subgaussianity parmeters of $\nu_t^r, \nu_t^c$ and the boundedness of $\{\bar{r}_a, \bar{c}_a\}_{a \in [K]}$. This finalizes the reduction.

We now proceed to prove Theorem~\ref{theorem::lower_bound} the main result of this section:

\theoremlowerboundlclucb*

\begin{proof}
If $\max\left(d\sqrt{T}, \frac{1-r_0}{21(\tau - c_0)^2} \right) = d\sqrt{T}$, then the argument in Theorem 24.1 of \cite{lattimore2018bandit} yields the desired result by noting that the framework of constrained bandits subsumes linear bandits. In this case we conclude there is a constrained linear bandit instance with $\theta_* \in \{ -\frac{1}{\sqrt{T}}, \frac{1}{\sqrt{T}}\}^d$ and $\mathcal{A}_t = [-1,1]^d$ satisfying:
\begin{equation*}
\mathcal{R}_\mathcal{C}(T) \geq \frac{d\sqrt{T}}{8e^2}.
\end{equation*}

Let's instead focus on the case where $B = \max\big(\frac{d\sqrt{T}}{8e^2},\frac{1-r_0}{21(\tau - c_0)^2}\big) =\frac{1-r_0}{21(\tau - c_0)^2} $. %

Pick any algorithm $\mathfrak{A}$. We want to show that the algorithm's regret on some environment is as large as $B$. For the remainder of the proof we restrict ourselves to instances where $\mathcal{A}_t = \Delta_d$ and $\theta_* = \bar{r}$, $\mu_* = \bar{c}$ with $\bar{r}, \bar{c} \in [0,1]^d$ parametrize a constrained Multi Armed Bandit problem such that arm $0$, the (known) safe arm satisfies $\bar{r}_0 = r_0$ and $\bar{c}_0 = c_0$. 

If there was any such instance $\bar{r}, \bar{c}$ such that $\mathcal{R}_\mathcal{C}(T) > B$ there would be nothing to prove. Hence without loss of generality, and by the way of contradiction we assume algorithm $\mathfrak{A}$ satisfies $\mathcal{R}_\mathcal{C}(T) \leq B $ for all $\bar{r}, \bar{c} \in [0,1]^d$ with $\bar{r}
_0 = r_0$ and $\bar{c}_0 = c_0$ and where all arms have unit variance Gaussian rewards and costs. 

In what follows we denote $\mathcal{R}_{\mathcal{C}}(T, \bar{r}, \bar{c})$ as the regret incurred by algorithm $\mathfrak{A}$ on instance $\theta_* = \bar{r}, \mu_* = \bar{c}$.

For simplicity we will introduce a couple of shorthand notational choices. Let $c = \tau - c_0$, $\Delta = \frac{1-r_0}{7}$ and $D = \frac{8r_0-1}{7}$ when $r_0 \geq \frac{1}{8}$. This will make it easier to explain the logic of the proof argument. Let's consider the following constrained bandit instance inducing measure $\nu$:
\begin{align*}
    \bar{c}^1  &= (\tau - c, \qquad \tau+2c,\qquad \tau-c,\qquad \tau+2c, \qquad\ldots,\qquad \tau+2c )\\
    \bar{r}^1 &= (D+\Delta,\qquad \qquad D+8\Delta,\qquad\qquad D, \qquad\quad D+4\Delta,\qquad\quad \ldots, \quad\qquad D+4\Delta\text{ }).
\end{align*}
Notice that by definition $\bar{c}^1_0 = c_0$ and $\bar{r}^1_0 = r_0$. For the $(\bar{r}^1, \bar{c}^1)$ problem instance, the optimal policy is a mixture between arms $0$ and $1$, where arm $0$ is chosen with probability $2/3$ and arm $1$ with probability $1/3$. The value of this optimal policy equals $D + \frac{10}{3}\Delta $.

Let $\bar{T}_j(T) \in [0,T]$ be the total amount of probability mass that $\mathfrak{A}$ allocated to arm $j$ up to time $T$. %

Let's lower bound the regret in the event that $\bar{T}_0(T) < \frac{T}{2}$. By the feasibility assumption it follows the average visitation policy $\pi_T$ defined as $\pi_T(i) = \frac{\bar{T}_i(T)}{T}$ is feasible. 

When the event $\{ \bar{T}_0(T) < \frac{T}{2}\}$ holds policy $\pi_T$ satisfies $\pi_T(0) \leq \frac{1}{2}$. A simple computation shows that to maximize its cumulative reward while maintaining feasibility constrained to $\pi_T(0) \leq\frac{1}{2}$, $\pi_T$'s optimal mass allocation is 

\begin{equation*}
\pi_T(i) = \begin{cases}
        \frac{1}{2} &\text{if }i = 0\\
        \frac{1}{3} &\text{if }i=1\\
        \frac{1}{6} &\text{if } i= 2
        \end{cases}
\end{equation*}
This policy has a reward of $D + \frac{19\Delta}{6} $ and therefore the regret of $\pi_T$ is lower bounded by $\frac{\Delta}{6}$. Thus, the regret $\mathcal{R}_\mathcal{C}(T, \bar{r}^1,\bar{c}^1)$ can be lower bounded bounded as
\begin{equation*}
\mathcal{R}_\mathcal{C}(T, \bar{r}^1,\bar{c}^1)  \geq \frac{\Delta T}{6} \mathbb{P}_\nu\left( \bar{T}_0(T) < \frac{T}{2}  \right)
\end{equation*}

Since by assumption, $\mathfrak{A}$ satisfies $\mathcal{R}_\mathcal{C}(T, \bar{r}^1, \bar{c}^1) \leq B $:
\begin{align*}
   B \geq \mathcal{R}_\mathcal{C}(T, \bar{r}^1,\bar{c}^1) \geq \frac{\Delta T}{6} \mathbb{P}_\nu\left( \bar{T}_0(T) < \frac{T}{2}  \right)
\end{align*}
And therefore:
\begin{align}\label{equation::lower_bound_p_thalf_nu}
    \mathbb{P}_\nu\left( \bar{T}_0(T) \geq \frac{T}{2} \right) = 1-\mathbb{P}_\nu\left(  \bar{T}_0(T) < \frac{T}{2}  \right) \geq 1-\frac{6B}{\Delta T} \geq 1/2
\end{align}
The last inequality follows from the assumption $T \geq \max(d-1,\frac{168eB}{1-r_0})$ (recall that $\Delta= \frac{1-r_0}{7}$.

Let's now consider the following constrained bandit instance inducing measure $\nu'$:
\begin{align*}
    \bar{c}^2  &= (\tau - c, \qquad \tau+2c,\qquad \tau-c,\qquad \tau-c, \qquad\ldots,\qquad \tau+2c )\\
    \bar{r}^2 &= (D+\Delta,\qquad \qquad D+8\Delta,\qquad\qquad D, \qquad\quad D+4\Delta,\text{ }\qquad \ldots, \quad\qquad D+4\Delta\text{ }).
\end{align*}
In this instance the optimal policy is to play arm 3 deterministically. This policy achieves a reward of $D + 4\Delta$.

We will now lower bound the regret $\mathcal{R}_\mathcal{C}(T, \bar{r}^2, \bar{c}^2)$ under measure $\nu'$. We'll do so by lower bounding the regret in the event that $\{ \bar{T}_0 \geq \frac{T}{2}\} $ holds. 

Similar to the argument we expanded on above, when $\{ \bar{T}_0 \geq \frac{T}{2} \}$ feasibility  guarantees the average policy $\pi_T'(i) = \frac{\bar{T}_i(T)}{T}$ is feasible. When  $\{ \bar{T}_0 \geq \frac{T}{2} \}$  holds policy $\pi_T'$ satisfies $\pi_T'(0) \geq \frac{1}{2}$. Simple computations show that to maximize the cumulative reward of policy $\pi_T'$ while maintaining feasibility constrained to $\pi_T'(0) \geq \frac{1}{2}$ the optimal allocation satisfies,

\begin{equation*}
\pi_T'(i) = \begin{cases}
\frac{1}{2} &\text{if } i=0\\
\frac{1}{2} &\text{if }i=3
\end{cases}
\end{equation*}
This policy achieves an expected  reward of $D + \frac{5\Delta}{2}$ and therefore the regret of $\pi_T'$ is lower bounded by $D +  4 \Delta - D - \frac{5}{2}\Delta = \frac{3}{2}\Delta$. Therefore,
\begin{equation*}
\mathcal{R}_\mathcal{C}(T, \bar{r}^2, \bar{c}^2) \geq \frac{3}{2}\Delta \mathbb{P}_{\nu'}\left(  \bar{T}_0(T) \geq \frac{T}{2}  \right)
\end{equation*}
Since by assumption $\mathfrak{A}$ satisfies $\mathcal{R}_\mathcal{C}(T, \bar{r}^2, \bar{c}^2) \leq B $, we have
\begin{equation*}
    \mathbb{P}_{\nu'}\left( \bar{T}_0(T) \geq \frac{T}{2} \right) \leq \frac{2B}{3\Delta T} \leq \frac{1}{4e}.
\end{equation*}
The last inequality follows from the assumption that $T \geq \max(  d-1,\frac{168eB}{1-r_0})$. We now apply the results of Lemma~\ref{lemma::binary_entropy_definition_properties} to this upper bound and the lower bound from Equation~\ref{equation::lower_bound_p_thalf_nu} and obtain,

\begin{equation*}
d\left(\mathbb{P}_{\nu}\left( \bar{T}_0(T) \geq \frac{T}{2} \right), \mathbb{P}_{\nu'}\left( \bar{T}_0(T) \geq \frac{T}{2} \right)\right) \geq  1/2.
\end{equation*}

As a consequence of~\eqref{equation::lower_bound_divergence}, Lemma~\ref{lemma::gaussian_divergence} and~\ref{lemma::binary_relative_entropy}, we have 
\begin{align*}
\mathrm{KL}(\mathbb{P}_\nu , \mathbb{P}_{\nu'} ) &= \mathbb{E}_\nu\big[T_3(T)\big]\times\mathrm{KL}\left(
    \mathcal{N}\left(\binom{\tau+2c}{4\Delta}, \mathbb{I}_d\right) , \mathcal{N}\left(\binom{\tau-c}{4\Delta}, \mathbb{I}_d\right) \right) \\
    &= 2c^2\mathbb{E}_\nu\big[T_3(T)\big]\\
    &\geq d\left(\mathbb{P}_{\nu}\left( \bar{T}_0(T) \geq \frac{T}{2} \right), \mathbb{P}_{\nu'}\left( \bar{T}_0(T) \geq \frac{T}{2} \right)\right) \\
    &\geq \frac{1}{2}.
\end{align*}
Therefore, we can conclude
\begin{equation}
 \mathbb{E}_\nu[\bar{T}_3(T)]= \mathbb{E}_\nu[T_3(T)]  \geq \frac{1}{4c^2}.
\end{equation}
Since in $\nu$, any feasible policy with support in arm $4$ and no support in arm 2 has a suboptimality gap of $\frac{4}{3}\Delta$, we conclude the regret $\mathcal{R}_\mathcal{C}(T, \bar{r}^2, \bar{c}^2)$ must satisfy:
\begin{equation*}
    \mathcal{R}_\mathcal{C}(T, \bar{r}^2, \bar{c}^2) \geq \frac{\Delta}{3c^2}.
\end{equation*}
Since $\Delta= \frac{1-r_0}{7}$ and $D = \frac{8r_0-1}{7}$ and noting that in this case $\frac{\Delta}{3c^2} = B$. The result follows. 
\end{proof}

\theoremlowerboundlclucbmab*

\begin{proof}

If $\max\left(\frac{1}{27}\sqrt{(K-1)T}, \frac{1-r_0}{21(\tau - c_0)^2} \right) = \frac{1}{27}\sqrt{(K-1)T}$, then the argument in Theorem 15.2 of \cite{lattimore2018bandit} yields the desired result by noting that the framework of constrained bandits subsumes multi armed bandits. In this case we conclude there is a constrained multi armed bandit instance satisfying:
\begin{equation*}
\mathcal{R}_\mathcal{C}(T) \geq \frac{1}{27}\sqrt{(K-1)T}.
\end{equation*}

When $B = \max\big(\frac{d\sqrt{T}}{8e^2},\frac{1-r_0}{21(\tau - c_0)^2}\big) =\frac{1-r_0}{21(\tau - c_0)^2} $, the same argument as in the proof of Theorem~\ref{theorem::lower_bound} finalizes the result. 
\end{proof}

\section{Extensions}

\subsection{Unknown $c_0$ and $r_0$}
\label{section::unknown_c0}

In this section, we relax Assumption~\ref{ass:safe-action}, and instead assume that we only have the knowledge of a safe action $x_0$, and no knowledge about its cost $c_0$ and reward $r_0$. The same discussion applies to $\bar{c}_1$ and $\bar{r}_1$ in OPB. The objective will be to design an algorithm capable of estimating $c_0$ and $r_0$ up to an accuracy of $\tau-c_0$ and $1-r_0$ for $c_0$ and $r_0$ respectively. We summarize the algorithm in the box below,

  \begin{algorithm}[H]
    \textbf{Input:} Safe arm $x_0$. \\
    \For{$t=1, \ldots , T$}{
     1. Pull arm $x_0$.\\
     2. Compute average cost estimator $\widehat{c}_0(t)$.\\
     3. Compute average reward estimator $\widehat{r}_0(t)$. \\
     4. \If{ $\widehat{c}_0(t) + 3\sqrt{2 \log(1/\delta)/t } \leq \tau$:}
     {
        Stop estimating $c_0 $.
    }
     5. \If{ $\widehat{r}_0(t) + 3\sqrt{2 \log(1/\delta)/t } \leq 1$:}
     {
        Stop estimating $r_0$.
    }   
     6. \If{ $\widehat{c}_0(t) + 3\sqrt{2 \log(1/\delta)/t } \leq \tau$ and $\widehat{r}_0(t) + 3\sqrt{2 \log(1/\delta)/t } \leq 1$:}
     {
        Return $\widehat{c}_0(T_0^c) $ and $\widehat{r}_0(T_0^r)$.
    }
     }    
     \caption{Unknown $c_0$, $r_0$ estimation.}
    \label{alg::unknown_c0_r0}
    \end{algorithm}

For all $t$ rounds to produce empirical mean estimators $\widehat{c}_0$ and $\widehat{r}_0$. Note that for all $\delta \in (0,1)$ and all $t \leq T$, $\widehat{c}_0(t)$ and $\widehat{r}_0(t)$ satisfy,
\begin{align}
\mathbb{P}\big(|\widehat{c}_0(t) - c_0 |\leq \sqrt{2\log(2T/\delta)/t}\big) \geq 1- \delta/2T \label{equation::helper_unknown_c0}
\\
\mathbb{P}\big(|\widehat{r}_0(t) - r_0| \leq \sqrt{2\log(2T/\delta)/t}\big) \geq 1 -\delta/2T. \label{equation::helper_unknown_r0}
\end{align}
Let's define $\tilde{\mathcal{E}}$ as the event where Equations~\ref{equation::helper_unknown_c0} and~\ref{equation::helper_unknown_r0} hold for all $t \leq T$. The reasoning above implies $\mathbb{P}(\tilde{\mathcal{E}}) \geq 1-\delta $. 

Denote $T_0^c, T_0^r$ as the times when conditions 4. and 5. of Algorithm~\ref{alg::unknown_c0_r0} trigger. Let's analyze $T_0^c$. Since $T_0^c$ is the first time when conditions 4. triggers thus,

\begin{equation*}
c_0 + 2\sqrt{2\log(2T/\delta)/T^c_0}\stackrel{(i)}{\leq} \widehat{c}_0(T_0^c) + 3\sqrt{2\log(2T/\delta)/T^c_0}  \leq \tau. 
\end{equation*}

Where $(i)$ holds because of equation~\ref{equation::helper_unknown_c0}. Thus 
\begin{equation}\label{equation::upper_bounding_confidence_interfal_c0_r0_est}
\sqrt{2\log(2T/\delta)/T^c_0} \leq \frac{\tau - c_0}{2}. 
\end{equation}
Since $\widetilde{\mathcal{E}}$ implies $\widehat{c}_{0}(T_0^c) \in \left[c_0- \sqrt{2\log(2T/\delta)/t},c_0 + \sqrt{2\log(2T/\delta)/t} \right]$ we have,
\begin{equation*}
\tau - \widehat{c}_0(T_0^c)  \in \left[ \tau- c_0 - \sqrt{2\log(2T/\delta)/T^c_0}, \tau - c_0 + \sqrt{2\log(2T/\delta)/T^c_0} \right] \subseteq \left[ \frac{\tau - c_0}{2},\frac{3(\tau - c_0 )}{2}  \right]
\end{equation*}
Similarly we conclude that whenever $\widetilde{\mathcal{E}}$ holds, 
\begin{equation*}
1 - \widehat{r}_0(T_0^r)  \in \left[ 1- r_0 - \sqrt{2\log(2T/\delta)/T^r_0}, 1 - r_0 + \sqrt{2\log(2T/\delta)/T^r_0} \right] \subseteq \left[ \frac{1 - r_0}{2},\frac{3(1- r_0 )}{2}  \right]
\end{equation*}
We define $\widehat{\Delta}_c = \tau - \widehat{c}_0(T_0^c)$ and $\widehat{\Delta}_r = 1 - \widehat{r}_0(T_0^r)$. The above discussion implies $\Delta_c$ and $\Delta_r$ are upper and lower bounded by constant multiples of $\tau - c_0$ and $1-r_0$ respectively.

When $\widetilde{\mathcal{E}}$ holds, Equation~\ref{equation::upper_bounding_confidence_interfal_c0_r0_est} implies $T_0^c \geq \frac{8  \log(2T/\delta)}{(\tau - c_0)^2}$ and $T_0^r\geq  \frac{8  \log(2T/\delta)}{(1 - r_0)^2}$. We now see that we can also upper bound these quantities, let's work through the argument for $c_0$. For all $t $ such that $\sqrt{2\log(2T/\delta)/t} \leq \frac{\tau - c_0}{4}$, when $\widetilde{\mathcal{E}}$ holds,
\begin{equation*}
\widehat{c}_0(t) + 3 \sqrt{2\log(2T/\delta)/t} \leq c_0 + 4\sqrt{2\log(2T/\delta)/t} \leq \tau.
\end{equation*}
Thus, condition 4. of Algorithm~\ref{alg::unknown_c0_r0} holds. Similarly for all $t$ such that $\sqrt{2\log(2T/\delta)/t} \leq \frac{1 - r_0}{4}$, when $\widetilde{\mathcal{E}}$ holds,
\begin{equation*}
\widehat{r}_0(t) + 3 \sqrt{2\log(2T/\delta)/t} \leq r_0 + 4\sqrt{2\log(2T/\delta)/t} \leq 1.
\end{equation*}
Thus condition 5. of Algorithm~\ref{alg::unknown_c0_r0} holds. This implies $T_0^c \leq \frac{32  \log(2T/\delta)}{(\tau - c_0)^2}$ and $T_0^r\leq  \frac{32  \log(2T/\delta)}{(1 - r_0)^2}$. If we define as $T_0$ to the time-step when condition 6. of Algorithm~\ref{alg::unknown_c0_r0} triggers, it follows that $$T_0 \leq 32  \log(2T/\delta)\max\left(\frac{1}{(\tau - c_0)^2},\frac{1}{(1 - r_0)^2} \right) .$$
 We then set $\frac{\alpha_r}{\alpha_c} = \widehat{\Delta}_r/\widehat\Delta_c$ and run LC-LUCB for rounds $t > T_0$. Since the scaling of $\alpha_r$ w.r.t. $\alpha_c$ is optimal up to constants, the same regret bounds (plus the regret incurred up to $T_0$) would hold. We can upper bound the regret incurred during $T_0$, 
\begin{equation*}
32  \log(2T/\delta)\max\left(\frac{1-r_0}{(\tau - c_0)^2},\frac{1}{1 - r_0} \right)
\end{equation*}

Therefore, in case $c_0$ is unknown, the algorithm proceeds by warm-starting our estimates of $\theta_*$ and $\mu_*$ using the data collected by playing the safe arm $x_0$. However, instead of estimating $\mu_*^{o, \perp}$, we build an estimator for $\mu_*$ over all its directions, including $e_0$, similar to what Algorithm~\ref{alg::linear_optimism_pessimism} (LC-LUCB) and Algorithm~\ref{alg::linear_optimism_pessimism} (OPLB) do for the reward parameter $\theta_*$. For the multi-constrained setting the estimation procedure of Algorithm~\ref{alg::unknown_c0_r0} can be used to estimate each of the cost signals simultaneously. An equivalent stopping condition yields a scheme to estimate the minimal cost gap up to constant accuracy. The same analysis as in the single constraint case holds.

\section{Nonlinear Rewards}\label{section::appendix_nonlinear_rewards}

\subsection{Properties of Least Squares Estimators}\label{section::LS_estimators_properties}

In this section we derive convergence properties of least squares estimators. These results will be crucial to analyze the NLC-LUCB algorithm in the following section.  Let $\{X_t, Y_t\}_{t=1}^\infty$ be a martingale sequence such that $X_t \in \mathcal{X}$ and $Y_t \in \mathbb{R}$ with $Y_t = f_\star(X_t) + \xi_t$ where $\xi_t$  satisfies Assumption~\ref{ass:noise-sub-gaussian}. Throughout this section we will use the notation $\mathcal{F}_{t-1} = \sigma(X_1, Y_1 \cdots, X_{t-1}, Y_{t-1})$ to denote the sigma algebra generated by all previous outcomes. 

Let $\mathcal{G}$ be a finite\footnote{Our results can be easily extended to the case of infinite function classes with bounded metric entropy.} class of functions such that $f_\star \in \mathcal{G}$ and for all $t \in \mathbb{N}$ consider the least squares regression estimator $\widehat{f}_t$ defined as,

\begin{equation*}
\widehat{f}_t = \min_{f \in \mathcal{F}} \sum_{\ell=1}^t (f(X_\ell) - Y_\ell)^2
\end{equation*}
We assume that 

\begin{assumption}[Bounded responses]
There exists a $B > 0$ such that for all $X \in \mathcal{X}$, and all $f\in \mathcal{F}$,
\begin{equation*}
| f(X) | \leq B, \text{ and } |Y_i| \leq B.
\end{equation*}
\end{assumption}

Our results rely on the following Uniform Empirical Bernstein bound from~\cite{howard2021time}.

\begin{lemma}[Uniform empirical Bernstein bound]
\label{lem:uniform_emp_bernstein}
In the terminology of \citet{howard2021time}, let $Z_t = \sum_{i=1}^t Y_i$ be a sub-$\psi_P$ process with parameter $c > 0$ and variance process $W_t$. Then with probability at least $1 - \widetilde{\delta}$ for all $t \in \mathbb{N}$
\begin{align*}
    Z_t &\leq  1.44 \sqrt{\max(W_t , m) \left( 1.4 \ln \ln \left(2 \left(\max\left(\frac{W_t}{m} , 1 \right)\right)\right) + \ln \frac{5.2}{\widetilde{\delta}}\right)}\\
   & \qquad + 0.41 c  \left( 1.4 \ln \ln \left(2 \left(\max\left(\frac{W_t}{m} , 1\right)\right)\right) + \ln \frac{5.2}{\widetilde{\delta}}\right)
\end{align*}
where $m > 0$ is arbitrary but fixed.
\end{lemma}

As a corollary of Lemma~\ref{lem:uniform_emp_bernstein} we can show the following,

\begin{lemma}[Freedman]\label{lemma:super_simplified_freedman}
Suppose $\{ X_t \}_{t=1}^\infty$ is a martingale difference sequence with $| X_t | \leq b$. Let $S_t = \sum_{\ell=1}^t X_\ell^2$ For any $\widetilde{\delta} \in (0,1)$, with probability at least $1-\widetilde{\delta}$,
\begin{equation*}
   \sum_{\ell=1}^t X_\ell \leq    4\sqrt{S_t \ln \frac{12\ln 2 t }{\widetilde{\delta}}  } + 6b \ln \frac{12\ln  2t }{\widetilde{\delta}}  .
\end{equation*}
for all $t \in \mathbb{N}$ simultaneously.
\end{lemma}

\begin{proof}
We are ready to use Lemma~\ref{lem:uniform_emp_bernstein} (with $c = b$). Let $S_t = \sum_{\ell=1}^t X_t$ and $W_t = \sum_{\ell=1}^t \mathrm{Var}_\ell(X_\ell)$. Let's set $m = b^2$. It follows that with probability $1-\widetilde{\delta}$ for all $t \in \mathbb{N}$

\begin{align*}
    S_t &\leq  1.44 \sqrt{\max(W_t , b^2) \left( 1.4 \ln \ln \left(2 \left(\max\left(\frac{W_t}{b^2} , 1 \right)\right)\right) + \ln \frac{5.2}{\widetilde{\delta}}\right)}\\
   & \qquad + 0.41 b  \left( 1.4 \ln \ln \left(2 \left(\max\left(\frac{W_t}{b} , 1\right)\right)\right) + \ln \frac{5.2}{\widetilde{\delta}}\right)\\
   &\leq 2 \sqrt{\max(W_t , b^2) \left( 2 \ln \ln \left(2 \left(\max\left(\frac{W_t}{b^2} , 1 \right)\right)\right) + \ln \frac{6}{\widetilde{\delta}}\right)}\\
   & \qquad + b  \left( 2 \ln \ln \left(2 \left(\max\left(\frac{W_t}{b^2} , 1\right)\right)\right) + \ln \frac{6}{\widetilde{\delta}}\right) \\
   &= 2\max(\sqrt{W_t}, b)A_t + bA_t^2\\
   &\leq 2\sqrt{W_t}A_t + 2bA_t + bA_t^2\\
   &\stackrel{(i)}{\leq} 2\sqrt{W_t}A_t + 3b A_t^2~,
\end{align*}
where $A_t = \sqrt{2 \ln \ln \left(2 \left(\max\left(\frac{W_t}{b^2} , 1\right)\right)\right) + \ln \frac{6}{\widetilde{\delta}}}$. Inequality $(i)$ follows because $A_t \geq 1$. By identifying $V_t = W_t$ we conclude that for any $\widetilde{\delta} \in (0,1)$ and $t \in \mathbb{N}$
\begin{equation*}
    \mathbb{P}\left(  \sum_{\ell=1}^t X_\ell >    2\sqrt{V_t}A_t + 3b A_t^2 \right) \leq \widetilde{\delta}
\end{equation*}
Where $A_t = \sqrt{2 \ln \ln \left(2 \left(\max\left(\frac{V_t}{b^2} , 1\right)\right)\right) + \ln \frac{6}{\widetilde{\delta}}}$. Since $V_t \leq tb^2$ with probability $1$,

\begin{equation*}
    \frac{V_t}{b^2} \leq t,
\end{equation*}

And therefore $2\ln \ln \left(2 \max(\frac{V_t}{b^2}, 1) \right) \leq 2 \ln \ln 2 t $ implying,
\begin{equation*}
    A_t \leq \sqrt{ 2 \ln \frac{12\ln t }{\widetilde{\delta}}   }
\end{equation*}
Thus
\begin{equation*}
    \mathbb{P}\left(  \sum_{\ell=1}^t X_\ell >    4\sqrt{V_t \ln \frac{12\ln 2t }{\widetilde{\delta}}  }A + 6b \ln \frac{12\ln 2t }{\widetilde{\delta}}  \right) \leq \widetilde{\delta}
\end{equation*}
Since $V_t \leq S_t$ the result follows. 
\end{proof}

\begin{lemma} \label{lemma::confidence_intervals_least_squares_regression} Let $\tilde \delta \in (0,1)$. The estimator $\widehat{f}_t$ satisfies,
\begin{equation*}
\sum_{\ell=1}^t \left( \widehat{f}_t(X_\ell) - f_\star(X_\ell) \right)^2 \leq \gamma(t, \tilde\delta)
\end{equation*}
for all $t \in \mathbb{N}$ with probability at least $1-\tilde \delta$, where $\gamma(t, \tilde\delta) = 256B(B+1) \log\left( \frac{12 |\mathcal{G}|\log 2t}{\tilde\delta} \right) $.
\end{lemma}

\begin{proof}
Since $\widehat{f}_t$ is the minimizer of the square loss over the data up to time $t$,

\begin{equation*}
\sum_{\ell=1}^t (\widehat{f}_t(X_\ell) - Y_\ell)^2 \leq \sum_{\ell=1}^t (f_\star(X_\ell) - Y_\ell)^2 
\end{equation*}

Plugging in the definition $Y_\ell = f_\star(X_\ell)  + \xi_\ell$ and expanding both sides of the inequality yields,

\begin{equation*}
    \sum_{\ell=1}^t (\widehat{f}_t(X_\ell) - f_\star(X_\ell) - \xi_\ell)^2 \leq \sum_{\ell=1}^t \xi_\ell^2
\end{equation*}

and therefore,

\begin{equation}\label{equation::fundamental_least_squares_inequality}
    \sum_{\ell=1}^t (\widehat{f}_t(X_\ell) - f_\star(X_\ell) )^2 \leq 2\sum_{\ell=1}^t  \xi_\ell \left( \widehat{f}_t(X_\ell) - f_\star(X_\ell) \right)
\end{equation}

For any \underline{fixed} $f \in \mathcal{F}$ consider the martingale difference process $\{Z_{\ell}\}_{\ell=1}^\infty$,

\begin{equation*}
    Z^f_\ell =  \xi_\ell \left( f(X_\ell) - f_\star(X_\ell) \right).
\end{equation*}
Since $|Z_\ell| \leq B$ it is easy to see that by the boundedness assumption, 
$\mathbb{E}\left[  \left(Z^f_\ell\right)^2 \Big| \mathcal{O}_{\ell-1}\right] \leq B^2  \left( f(X_\ell) - f_\star(X_\ell) \right)^2 $. 
Thus, Freedman's inequality (Lemma~\ref{lemma:super_simplified_freedman}) implies,
\begin{align*}
\sum_{\ell=1}^t Z^f_\ell &\leq 4 \sqrt{ \sum_{\ell=1}^t \mathbb{E}\left[  \left(Z^f_\ell\right)^2 \Big| \mathcal{O}_{\ell-1}\right]  \log\left( \frac{12 |\mathcal{G}|\log 2t}{\tilde\delta} \right)  }    + 6B \log\frac{12 |\mathcal{G}|\log 2t}{\tilde\delta}\\
&\leq 4 B \sqrt{ \left[\sum_{\ell=1}^t \left( f(X_\ell) - f_\star(X_\ell) \right)^2 \right] \log\left( \frac{12 |\mathcal{G}|\log 2t}{\tilde\delta} \right)  }    + 6B \log\frac{12 |\mathcal{G}|\log 2t}{\tilde\delta}\\
&\stackrel{(i)}{\leq} \frac{\sum_{\ell=1}^t \left( f(X_\ell) - f_\star(X_\ell) \right)^2 }{4} + 64B^2 \log\left( \frac{12 |\mathcal{G}|\log 2t}{\tilde\delta} \right) + 6B \log\left( \frac{12 |\mathcal{G}|\log 2t}{\tilde\delta} \right) \\
&\leq \frac{\sum_{\ell=1}^t \left( f(X_\ell) - f_\star(X_\ell) \right)^2 }{4} + 64B(B+1) \log\left( \frac{12 |\mathcal{G}|\log 2t}{\tilde\delta} \right) 
\end{align*}
with probability at least $1-\frac{\tilde\delta}{ |\mathcal{G}|}$ for all $t \in \mathbb{N}$. Where $(i)$ holds because of the inequality $2\sqrt{ab} \leq \alpha a + \frac{b}{\alpha}$ for any $\alpha > 0$. Plugging this back into equation~\ref{equation::fundamental_least_squares_inequality} we obtain,
\begin{equation*}
   \sum_{\ell=1}^t (\widehat{f}_t(X_\ell) - f_\star(X_\ell) )^2 \leq 128B(B+1) \log\left( \frac{12 |\mathcal{G}|\log 2t}{\tilde\delta} \right)  + \frac{1}{2}  \sum_{\ell=1}^t \left( \widehat{f}_t(X_\ell) - f_\star(X_\ell)\right)^2
\end{equation*}
Canceling terms on both sides yields (and multiplying by two) yields,
\begin{equation*}
  \sum_{\ell=1}^t (\widehat{f}_t(X_\ell) - f_\star(X_\ell) )^2 \leq 256B(B+1) \log\left( \frac{12 |\mathcal{G}|\log 2t}{\tilde\delta} \right) 
\end{equation*}
The result follows.
\end{proof}

\begin{corollary}\label{corollary:confidence_sets_function_approx}
If $\gamma_r(t, \delta) = 512 \log\left( \frac{24 |\mathcal{G}_r|\log 2t}{\delta} \right),
\gamma_c(t, \delta) = 512 \log\left( \frac{24 |\mathcal{G}_c|\log 2t}{\delta} \right)$ then $\theta_\star \in C_t^r(\delta)$ and $\mu_\star \in C_t^c(\delta)$ for all $t \in \mathbb{N}$ with probability at least $1-\delta$. 
\end{corollary}

\begin{proof}
This result is an immediate consequence of setting $B=1$ and $\tilde \delta = \delta/2$ in Lemma~\ref{lemma::confidence_intervals_least_squares_regression}. 
\end{proof}

\subsection{Proof of Lemma~\ref{lemma::optimism_eluder}}\label{section::proof_lemma_optimism_eluder}

Notice that for any policy $\pi$ %
\begin{align}
        \widetilde{V}_t^c(\pi) \leq \mu_\star(\pi) +  \max_{\mu,\mu' \in C^c_t(\delta)} \mu(\pi) - \mu'(\pi). \label{equation::inequalities_V_tilde_t_c}
\end{align}
with probability at least $1-\delta$ for all $t \in \mathbb{N}$. This is because $\mu_\star$ belongs to $C^c_t(\delta)$ w.h.p and therefore
\begin{align*}
    \widetilde{V}_t^c(\pi) &= \max_{\mu \in C^c_t(\delta)} \mu(\pi) = \mu_\star(\pi) + \max_{\mu \in C^c_t(\delta) }\mu(\pi) - \mu_\star(\pi)\\
    &\leq \mu_\star(\pi) + \max_{\mu,\mu' \in C^c_t(\delta)} \mu(\pi) - \mu'(\pi).
\end{align*}
Similarly since $\theta_\star \in C^r_t(\delta)$ with high probability, $\max_{\theta \in C_t^r(\delta)} \theta(\pi) \geq \theta_\star(\pi) $ and therefore
\begin{align}
    \widetilde{V}_t^r(\pi) \geq \underbrace{\theta_\star(\pi)}_{V^r_t(\pi)} + \alpha_r \max_{\mu^{'} , \mu^{''}\in C^c_t(\delta)} \mu^{'} (\pi) - \mu^{''}(\pi) \label{equation::lower_bound_vtilder}
\end{align}

\lemmaoptimismeluder*

\begin{proof}
We are going to prove this result by splitting it into two cases determined by whether $\pi_t^\star$ belongs to $\widetilde{\Pi}_t$ or not. 

\textbf{Case 1. $\pi_t^* \in \widetilde{\Pi}_t$.} Recall that $\pi_t = \argmax_{\pi \in \widetilde{\Pi}_t} \widetilde{V}_t^r(\pi)$. It follows that $\widetilde{V}_t^r(\pi_t) \geq \widetilde{V}_t^r(\pi_t^*) \geq V_t^r(\pi_t^*)$ where the last inequality is true because $\widetilde{V}_t^r(\pi)$ is an optimistic estimator of the value of all policies (see Equation~\ref{equation::lower_bound_vtilder}).

\textbf{Case 2. $\pi_t^* \not\in \widetilde{\Pi}_t$.} Let $\pi_0 = \delta(x_0)$. By definition for all $\mu \in C_t^c(\delta)$ it follows that $\mu(x_0) = c_0$. Consider a mixture policy $\widetilde{\pi}_t = \gamma_t \pi_t^* + (1-\gamma_t) \pi_0$ where $\gamma_t$ is the smallest value in $[0,1]$ such that $\widetilde{\pi}_t \in \widetilde{\Pi}_t$. Let's see this value exists: 

Let $\widetilde{\mu}_t = \argmax_{\mu \in C_t^c(\delta) \text{ s.t } \mu(x_0) = c_0 }  \mu(\pi_t^\star, \mathcal{A}_t)$. Observe that $\widetilde{\mu}_t$ by definition also satisfies $\widetilde{\mu}_t = \argmax_{\mu \in C_t^c(\delta) \text{ s.t } \mu(x_0) = c_0} \mu(\gamma \pi_t^* + (1-\gamma) \pi_0, \mathcal{A}_t)$ and that $$\widetilde{V}_t^c(\gamma \pi_t^* + (1-\gamma) \pi_0)= \widetilde{\mu}_t(\gamma \pi_t^* + (1-\gamma) \pi_0, \mathcal{A}_t) = \gamma \widetilde{\mu}_t(\pi_t^*, \mathcal{A}_t) + (1-\gamma) c_0.$$ This shows there exists a value $\gamma_t \in [0,1]$ such that $\widetilde{V}_t^c( \widetilde{\pi}_t ) = \widetilde{\mu}_t(\gamma_t \pi_t^* + (1-\gamma_t) \pi_0, \mathcal{A}_t) = \tau$. Let's start by proving a lower bound for $\gamma_t$. By definition
$$\widetilde{V}_t^c(\widetilde{\pi}_t) = \gamma_t \widetilde{V}_t^c(\pi^*_t) + (1-\gamma_t) c_0 = \tau.$$ 
And therefore,
\begin{align}
\gamma_t = \frac{\tau - c_0}{\widetilde{V}_t^c(\pi^*_t) - c_0} &\stackrel{(i)}{\geq} \frac{\tau - c_0}{\mu_\star(\pi_t^*) - c_0 + \max_{\mu,\mu' \in C^c_t(\delta_c)} \mu(\pi_t^*, \mathcal{A}_t) - \mu'(\pi_t^*, \mathcal{A}_t)  } \notag\\
&\stackrel{(ii)}{\geq}  \frac{\tau - c_0}{\tau - c_0 + \max_{\mu,\mu' \in C^c_t(\delta_c)} \mu(\pi_t^*, \mathcal{A}_t) - \mu'(\pi_t^*, \mathcal{A}_t)  } \label{equation::lower_bound_gamma_t}
\end{align}
Where $(i)$ follows from~\ref{equation::inequalities_V_tilde_t_c} and $(ii)$ holds because it satisfies $\mu_\star(\pi_t^*, \mathcal{A}_t) \leq \tau$. Let $r_0 = \theta_\star(x_0)$. Since $\pi_t$ and $\widetilde \pi_t$ are both feasible, it follows that $\widetilde{V}_t^r(\pi_t) \geq \widetilde{V}_t^r(\widetilde{\pi}_t)$ and therefore,
\begin{align*}
    \widetilde{V}_t^r(\pi_t) &\geq \widetilde{V}_t^r(\widetilde{\pi}_t) = \gamma_t \widetilde{V}_t^r(\pi_t^*)  + (1-\gamma_t) r_0\\
    &\stackrel{(i)}{\geq} \gamma_t \left( \theta_\star(\pi_t^\star, \mathcal{A}_t) + \alpha_r \max_{\mu^{'} , \mu^{''}\in C^c_t(\delta)} \mu^{'} (\pi_t^\star, \mathcal{A}_t) - \mu^{''}(\pi_t^\star, \mathcal{A}_t) \right)+ (1-\gamma_t) r_0.
\end{align*}
Where $(i)$ is a result of inequality~\ref{equation::lower_bound_vtilder}. Let $C_1 = \max_{\mu^{'} , \mu^{''}\in C^c_t(\delta)} \mu^{'} (\pi_t^\star, \mathcal{A}_t) - \mu^{''}(\pi_t^\star, \mathcal{A}_t)$. Substituting the $\gamma_t$ lower bound from Equation~\ref{equation::lower_bound_gamma_t} in the RHS of the equation above,
\begin{equation*}
\gamma_t(\theta_\star(\pi_t^\star, \mathcal{A}_t) + \alpha_r C_1) + (1-\gamma_t) r_0= \frac{\tau-c_0}{\tau- c_0 + C_1} (\theta_\star(\pi_t^\star, \mathcal{A}_t) + \alpha_r C_1)  + \frac{C_1}{\tau- c_0 + C_1} r_0
\end{equation*}
Since $\theta_\star(\pi_t^*, \mathcal{A}_t) \leq 1$ (Assumption~\ref{ass:bounded_reward_costs_function_approx}), setting $\alpha_r = \frac{1-r_0}{\tau- c_0}$ is enough to guarantee the inequality $\widetilde{V}_t^r(\pi_t) \geq \theta_\star(\pi_t^\star, \mathcal{A}_t)$ holds. \end{proof}

\newpage

\section{Additional Experiments of Section~\ref{section::expectation_constraints}}

\begin{figure*}%
\centering
\subfigure{\includegraphics[width=0.32\linewidth]{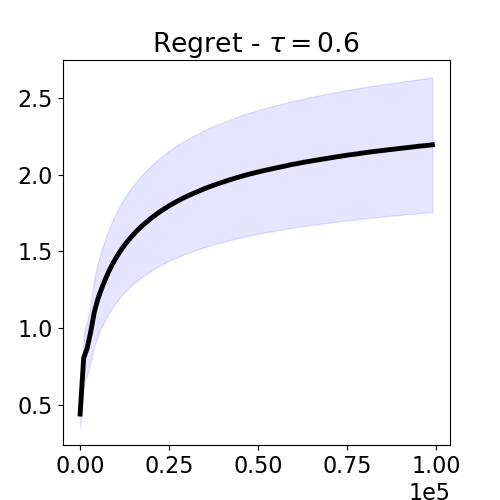}} \centering
\subfigure{\includegraphics[width=0.32\linewidth]{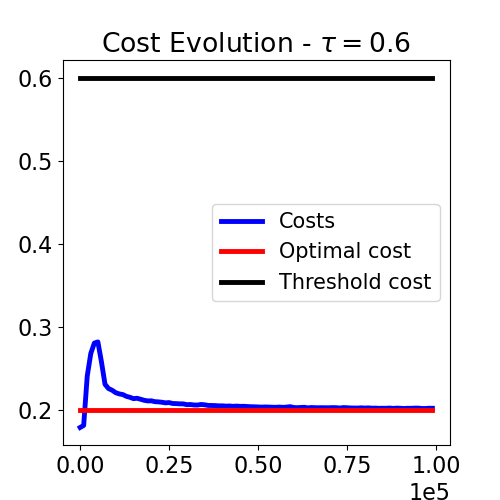}} 
\centering
\subfigure{\includegraphics[width=0.32\linewidth]{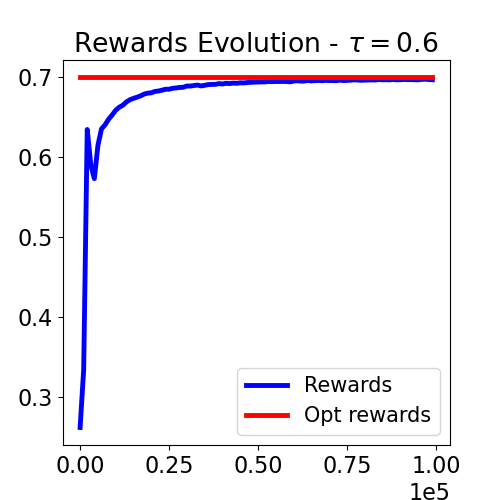}}
\centering
\subfigure{\includegraphics[width=0.32\linewidth]{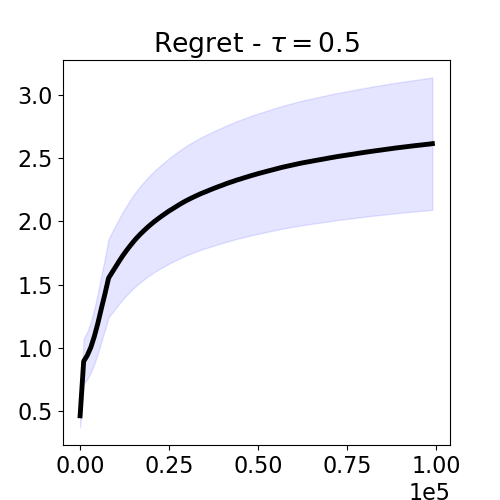}} 
\centering
\subfigure{\includegraphics[width=0.32\linewidth]{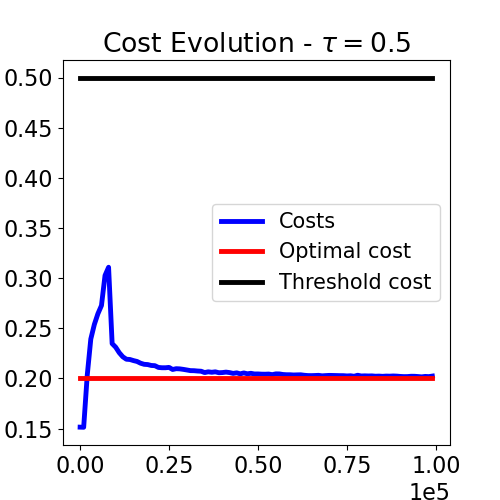}} 
\centering
\subfigure{\includegraphics[width=0.32\linewidth]{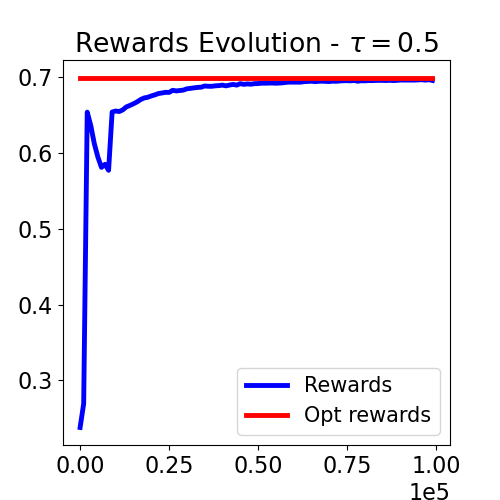}}
        \caption{\small{Regret {\em (left)}, cost {\em (middle)}, and reward {\em (right)} evolution of OPB in a $4$-armed bandit problem with Bernoulli reward and cost distributions with means $\bar{r} = (0.1, 0.2, 0.4, 0.7)$ and $\bar{c}=(0, 0.4, 0.5, 0.2)$. The cost of the safe arm (Arm~1) is $\bar{c}_1 = 0$.}}
        \label{fig:constrained_bandits3}    
\end{figure*}

\vskip 0.2in
\bibliography{ref}

\begin{thebibliography}{48}
\providecommand{\natexlab}[1]{#1}
\providecommand{\url}[1]{\texttt{#1}}
\expandafter\ifx\csname urlstyle\endcsname\relax
  \providecommand{\doi}[1]{doi: #1}\else
  \providecommand{\doi}{doi: \begingroup \urlstyle{rm}\Url}\fi

\bibitem[Abbasi-Yadkori et~al.(2011)Abbasi-Yadkori, P{\'a}l, and
  Szepesv{\'a}ri]{abbasi2011improved}
Y.~Abbasi-Yadkori, D.~P{\'a}l, and C.~Szepesv{\'a}ri.
\newblock Improved algorithms for linear stochastic bandits.
\newblock In \emph{Advances in Neural Information Processing Systems 24}, pages
  2312--2320, 2011.

\bibitem[Abeille and Lazaric(2017)]{abeille2017linear}
M.~Abeille and A.~Lazaric.
\newblock Linear {T}hompson sampling revisited.
\newblock \emph{Electronic Journal of Statistics}, 11\penalty0 (2):\penalty0
  5165--5197, 2017.

\bibitem[Agrawal and Devanur(2014)]{agrawal2014bandits}
S.~Agrawal and N.~Devanur.
\newblock Bandits with concave rewards and convex knapsacks.
\newblock In \emph{Proceedings of the Fifteenth ACM conference on Economics and
  computation}, pages 989--1006, 2014.

\bibitem[Agrawal and Devanur(2016)]{agrawal2016linear}
S.~Agrawal and N.~Devanur.
\newblock Linear contextual bandits with knapsacks.
\newblock In \emph{Advances in Neural Information Processing Systems 29}, pages
  3450--3458, 2016.

\bibitem[Agrawal and Goyal(2013{\natexlab{a}})]{agrawal13further}
S.~Agrawal and N.~Goyal.
\newblock Further optimal regret bounds for {Thompson} sampling.
\newblock In \emph{Proceedings of the 16th International Conference on
  Artificial Intelligence and Statistics}, pages 99--107, 2013{\natexlab{a}}.

\bibitem[Agrawal and Goyal(2013{\natexlab{b}})]{agrawal13thompson}
S.~Agrawal and N.~Goyal.
\newblock Thompson sampling for contextual bandits with linear payoffs.
\newblock In \emph{Proceedings of the 30th International Conference on Machine
  Learning}, pages 127--135, 2013{\natexlab{b}}.

\bibitem[Amani et~al.(2019)Amani, Alizadeh, and Thrampoulidis]{amani2019linear}
S.~Amani, M.~Alizadeh, and C.~Thrampoulidis.
\newblock Linear stochastic bandits under safety constraints.
\newblock In \emph{Advances in Neural Information Processing Systems}, pages
  9252--9262, 2019.

\bibitem[Amani et~al.(2021)Amani, Thrampoulidis, and Yang]{amani2021safe}
S.~Amani, C.~Thrampoulidis, and L.~Yang.
\newblock Safe reinforcement learning with linear function approximation.
\newblock In \emph{International Conference on Machine Learning}, pages
  243--253, 2021.

\bibitem[Auer et~al.(2002)Auer, Cesa-Bianchi, and Fischer]{auer02finitetime}
P.~Auer, N.~Cesa-Bianchi, and P.~Fischer.
\newblock Finite-time analysis of the multiarmed bandit problem.
\newblock \emph{Machine Learning}, 47:\penalty0 235--256, 2002.

\bibitem[Badanidiyuru et~al.(2013)Badanidiyuru, Kleinberg, and
  Slivkins]{badanidiyuru2013bandits}
A.~Badanidiyuru, R.~Kleinberg, and A.~Slivkins.
\newblock Bandits with knapsacks.
\newblock In \emph{{IEEE} 54th Annual Symposium on Foundations of Computer
  Science}, pages 207--216, 2013.

\bibitem[Badanidiyuru et~al.(2014)Badanidiyuru, Langford, and
  Slivkins]{badanidiyuru2014resourceful}
A.~Badanidiyuru, J.~Langford, and A.~Slivkins.
\newblock Resourceful contextual bandits.
\newblock In \emph{Proceedings of The 27th Conference on Learning Theory},
  pages 1109--1134, 2014.

\bibitem[Balakrishnan et~al.(2018)Balakrishnan, Bouneffouf, Mattei, and
  Rossi]{balakrishnan2018using}
A.~Balakrishnan, D.~Bouneffouf, N.~Mattei, and F.~Rossi.
\newblock Using contextual bandits with behavioral constraints for constrained
  online movie recommendation.
\newblock In \emph{IJCAI}, pages 5802--5804, 2018.

\bibitem[Bura et~al.(2022)Bura, Hasanzade~Zonuzy, Kalathil, Shakkottai, and
  Chamberland]{Bura22DD}
A.~Bura, A.~Hasanzade~Zonuzy, D.~Kalathil, S.~Shakkottai, and J.-F.
  Chamberland.
\newblock Dope: Doubly optimistic and pessimistic exploration for safe
  reinforcement learning.
\newblock In \emph{Advances in Neural Information Processing Systems}, 2022.

\bibitem[Chan et~al.(2023)Chan, Pacchiano, Tripuraneni, Song, Bartlett, and
  Jordan]{chan2021parallelizing}
J.~Chan, A.~Pacchiano, N.~Tripuraneni, Y.~Song, P.~Bartlett, and M.~Jordan.
\newblock Parallelizing contextual bandits.
\newblock \emph{arXiv:2105.10590}, 2023.

\bibitem[Chaudhary and Kalathil(2022)]{Chaudhary22SO}
S.~Chaudhary and D.~Kalathil.
\newblock Safe online convex optimization with unknown linear safety
  constraints.
\newblock In \emph{AAAI Conference on Artificial Intelligence}, pages
  6175--6182, 2022.

\bibitem[Chen et~al.(2022)Chen, Gangrade, and Saligrama]{Chen22DO}
T.~Chen, A.~Gangrade, and V.~Saligrama.
\newblock A doubly optimistic strategy for safe linear bandits.
\newblock \emph{arXiv preprint arXiv:2209.13694}, 2022.

\bibitem[Dani et~al.(2008)Dani, Hayes, and Kakade]{dani08stochastic}
V.~Dani, T.~Hayes, and S.~Kakade.
\newblock Stochastic linear optimization under bandit feedback.
\newblock In \emph{Proceedings of the 21st Annual Conference on Learning
  Theory}, pages 355--366, 2008.

\bibitem[Ding et~al.(2021)Ding, Wei, Yang, Wang, and Jovanovic]{Ding21PE}
D.~Ding, X.~Wei, Z.~Yang, Z.~Wang, and M.~Jovanovic.
\newblock Provably efficient safe exploration via primal-dual policy
  optimization.
\newblock In \emph{International Conference on Artificial Intelligence and
  Statistics}, pages 3304--3312, 2021.

\bibitem[Efroni et~al.(2020)Efroni, Mannor, and Pirotta]{efroni2020exploration}
Y.~Efroni, S.~Mannor, and M.~Pirotta.
\newblock Exploration-exploitation in constrained mdps.
\newblock \emph{arXiv:2003.02189}, 2020.

\bibitem[Foster and Rakhlin(2020)]{foster2020beyond}
Dylan Foster and Alexander Rakhlin.
\newblock Beyond ucb: Optimal and efficient contextual bandits with regression
  oracles.
\newblock In \emph{International Conference on Machine Learning}, pages
  3199--3210. PMLR, 2020.

\bibitem[Garcelon et~al.(2020)Garcelon, Ghavamzadeh, Lazaric, and
  Pirotta]{Garcelon20IA}
E.~Garcelon, M.~Ghavamzadeh, A.~Lazaric, and M.~Pirotta.
\newblock Improved algorithms for conservative exploration in bandits.
\newblock In \emph{AAAI}, 2020.

\bibitem[Ghosh et~al.(2022)Ghosh, Zhou, and Shroff]{Ghosh22PE}
A.~Ghosh, X.~Zhou, and N.~Shroff.
\newblock Provably efficient model-free constrained {RL} with linear function
  approximation.
\newblock In \emph{Advances in Neural Information Processing Systems},
  volume~35, pages 13303--13315, 2022.

\bibitem[Howard et~al.(2021)Howard, Ramdas, McAuliffe, and
  Sekhon]{howard2021time}
S.~Howard, A.~Ramdas, J.~McAuliffe, and J.~Sekhon.
\newblock Time-uniform, nonparametric, nonasymptotic confidence sequences.
\newblock \emph{The Annals of Statistics}, 2021.

\bibitem[Kaufmann et~al.(2016)Kaufmann, Capp{\'e}, and
  Garivier]{kaufmann2016complexity}
E.~Kaufmann, O.~Capp{\'e}, and A.~Garivier.
\newblock On the complexity of best-arm identification in multi-armed bandit
  models.
\newblock \emph{The Journal of Machine Learning Research}, 17\penalty0
  (1):\penalty0 1--42, 2016.

\bibitem[Kazerouni et~al.(2017)Kazerouni, Ghavamzadeh, Yadkori, and
  Van~Roy]{kazerouni2017conservative}
A.~Kazerouni, M.~Ghavamzadeh, Y.~Abbasi Yadkori, and B.~Van~Roy.
\newblock Conservative contextual linear bandits.
\newblock In \emph{Advances in Neural Information Processing Systems}, pages
  3910--3919, 2017.

\bibitem[Lai and Robbins(1985)]{lai85asymptotically}
T.~Lai and H.~Robbins.
\newblock Asymptotically efficient adaptive allocation rules.
\newblock \emph{Advances in Applied Mathematics}, 6\penalty0 (1):\penalty0
  4--22, 1985.

\bibitem[Lattimore and Szepesv{\'a}ri(2019)]{lattimore2018bandit}
T.~Lattimore and C.~Szepesv{\'a}ri.
\newblock \emph{Bandit Algorithms}.
\newblock Cambridge University Press, 2019.

\bibitem[Li et~al.(2010)Li, Chu, Langford, and Schapire]{Li10CB}
L.~Li, W.~Chu, J.~Langford, and R.~Schapire.
\newblock A contextual-bandit approach to personalized news article
  recommendation.
\newblock In \emph{WWW}, pages 661--670, 2010.

\bibitem[Liu and Wang(2023)]{liu2023global}
Chong Liu and Yu-Xiang Wang.
\newblock Global optimization with parametric function approximation.
\newblock In \emph{International Conference on Machine Learning}, pages
  22113--22136. PMLR, 2023.

\bibitem[Liu et~al.(2021{\natexlab{a}})Liu, Zhou, Kalathil, Kumar, and
  Tian]{liu2021learning}
T.~Liu, R.~Zhou, D.~Kalathil, P.~Kumar, and C.~Tian.
\newblock Learning policies with zero or bounded constraint violation for
  constrained {MDP}s.
\newblock \emph{Advances in Neural Information Processing Systems},
  34:\penalty0 17183--17193, 2021{\natexlab{a}}.

\bibitem[Liu et~al.(2021{\natexlab{b}})Liu, Li, Shi, and Ying]{Liu21EP}
X.~Liu, B.~Li, P.~Shi, and L.~Ying.
\newblock An efficient pessimistic-optimistic algorithm for stochastic linear
  bandits with general constraints.
\newblock \emph{Advances in Neural Information Processing Systems},
  34:\penalty0 24075--24086, 2021{\natexlab{b}}.

\bibitem[Maghsudi and Hossain(2016)]{maghsudi2016multi}
S.~Maghsudi and E.~Hossain.
\newblock Multi-armed bandits with application to 5{G} small cells.
\newblock \emph{IEEE Wireless Communications}, 23\penalty0 (3):\penalty0
  64--73, 2016.

\bibitem[Moradipari et~al.(2019)Moradipari, Amani, Alizadeh, and
  Thrampoulidis]{Moradipari19SL}
A.~Moradipari, S.~Amani, M.~Alizadeh, and C.~Thrampoulidis.
\newblock Safe linear {T}hompson sampling with side information.
\newblock \emph{arXiv:1911.02156}, 2019.

\bibitem[Ontan{\'o}n(2013)]{ontanon2013combinatorial}
S.~Ontan{\'o}n.
\newblock The combinatorial multi-armed bandit problem and its application to
  real-time strategy games.
\newblock In \emph{Ninth Artificial Intelligence and Interactive Digital
  Entertainment Conference}, 2013.

\bibitem[Pacchiano et~al.(2021)Pacchiano, Ghavamzadeh, Bartlett, and
  Jiang]{pacchiano2020stochastic}
A.~Pacchiano, M.~Ghavamzadeh, P.~Bartlett, and H.~Jiang.
\newblock Stochastic bandits with linear constraints.
\newblock In \emph{Proceedings of the 24th International Conference on
  Artificial Intelligence and Statistics}, 2021.

\bibitem[Rusmevichientong and Tsitsiklis(2010)]{rusmevichientong10linearly}
P.~Rusmevichientong and J.~Tsitsiklis.
\newblock Linearly parameterized bandits.
\newblock \emph{Mathematics of Operations Research}, 35\penalty0 (2):\penalty0
  395--411, 2010.

\bibitem[Russo et~al.(2018)Russo, {Van Roy}, Kazerouni, Osband, and
  Wen]{russo18tutorial}
D.~Russo, B.~{Van Roy}, A.~Kazerouni, I.~Osband, and Z.~Wen.
\newblock A tutorial on {Thompson} sampling.
\newblock \emph{Foundations and Trends in Machine Learning}, 11\penalty0
  (1):\penalty0 1--96, 2018.

\bibitem[Russo and Van~Roy(2013)]{russo2013eluder}
Daniel Russo and Benjamin Van~Roy.
\newblock Eluder dimension and the sample complexity of optimistic exploration.
\newblock \emph{Advances in Neural Information Processing Systems}, 26, 2013.

\bibitem[Shi et~al.(2023)Shi, Liang, and Shroff]{Shi23NO}
M.~Shi, Y.~Liang, and N.~Shroff.
\newblock A near-optimal algorithm for safe reinforcement learning under
  instantaneous hard constraints.
\newblock \emph{arXiv preprint arXiv:2302.04375}, 2023.

\bibitem[Thompson(1933)]{thompson33likelihood}
W.~Thompson.
\newblock On the likelihood that one unknown probability exceeds another in
  view of the evidence of two samples.
\newblock \emph{Biometrika}, 25\penalty0 (3-4):\penalty0 285--294, 1933.

\bibitem[Villar et~al.(2015)Villar, Bowden, and Wason]{Villar15MA}
S.~Villar, J.~Bowden, and J.~Wason.
\newblock Multi-armed bandit models for the optimal design of clinical trials:
  Benefits and challenges.
\newblock \emph{Statistical Science}, 30\penalty0 (2):\penalty0 199--215, 2015.

\bibitem[Wang et~al.(2022)Wang, Wagenmaker, and Jamieson]{Wang22BA}
Z.~Wang, A.~Wagenmaker, and K.~Jamieson.
\newblock Best arm identification with safety constraints.
\newblock In \emph{International Conference on Artificial Intelligence and
  Statistics}, pages 9114--9146, 2022.

\bibitem[Washburn(2008)]{washburn2008application}
R.~Washburn.
\newblock Application of multi-armed bandits to sensor management.
\newblock In \emph{Foundations and Applications of Sensor Management}, pages
  153--175. Springer, 2008.

\bibitem[Wei et~al.(2021)Wei, Liu, and Ying]{Wei21PE}
H.~Wei, X.~Liu, and L.~Ying.
\newblock A provably-efficient model-free algorithm for constrained {M}arkov
  decision processes.
\newblock \emph{arXiv preprint arXiv:2106.01577}, 2021.

\bibitem[Wu et~al.(2015)Wu, Srikant, Liu, and Jiang]{wu2015algorithms}
H.~Wu, R.~Srikant, X.~Liu, and C.~Jiang.
\newblock Algorithms with logarithmic or sub-linear regret for constrained
  contextual bandits.
\newblock In \emph{Advances in Neural Information Processing Systems 28}, pages
  433--441, 2015.

\bibitem[Wu et~al.(2016)Wu, Shariff, Lattimore, and
  Szepesv{\'a}ri]{wu2016conservative}
Y.~Wu, R.~Shariff, T.~Lattimore, and C.~Szepesv{\'a}ri.
\newblock Conservative bandits.
\newblock In \emph{International Conference on Machine Learning}, pages
  1254--1262, 2016.

\bibitem[Zhou et~al.(2020)Zhou, Li, and Gu]{zhou2020neural}
Dongruo Zhou, Lihong Li, and Quanquan Gu.
\newblock Neural contextual bandits with ucb-based exploration.
\newblock In \emph{International Conference on Machine Learning}, pages
  11492--11502. PMLR, 2020.

\bibitem[Zhou and Ji(2022)]{Zhou22KM}
X.~Zhou and B.~Ji.
\newblock On kernelized multi-armed bandits with constraints.
\newblock \emph{Advances in Neural Information Processing Systems},
  35:\penalty0 14--26, 2022.

\end{thebibliography}

\end{document}